\documentclass[twoside,11pt]{article}

\usepackage{amsmath, amssymb, amsfonts, amsthm, mathtools, mathrsfs, bbm} 
\usepackage{graphicx} 
\usepackage{hyperref}
\usepackage{geometry} 
\usepackage{setspace} 
\usepackage{caption} 
\usepackage[numbers, sort]{natbib}
\usepackage{lineno} 

\usepackage{lipsum}
\usepackage{titlesec}
\usepackage{fancyhdr} 
\usepackage{array}
\usepackage{xcolor}
\usepackage{enumitem}
\usepackage{bbm}
\usepackage{tikz} 
\usetikzlibrary{calc, arrows.meta, intersections, patterns, positioning, shapes.misc, fadings, through,decorations.pathreplacing, trees, shapes.geometric,math,arrows}
\usepackage{tikz-cd}
\usepackage{amscd}
\usepackage{pgfplots}
\pgfplotsset{compat=1.18}
\usepackage{url}
\usepackage{aliascnt}

\newcommand{\supp}[1]{\mathrm{supp}\!\left(#1\right)}
\newcommand{\suppinline}[1]{\mathrm{supp}(#1)}
\newcommand{\spanC}[1]{\mathrm{span}_{\mathbb{C}}\!\left(#1\right)}
\newcommand{\spanCinline}[1]{\mathrm{span}_{\mathbb{C}}(#1)}
\newcommand{\spanR}[1]{\mathrm{span}_{\mathbb{R}}\!\left(#1\right)}
\newcommand{\spanRinline}[1]{\mathrm{span}_{\mathbb{R}}(#1)}

\newcommand{\dimC}[1]{\mathrm{dim}_{\mathbb{C}}\!\left(#1\right)}
\newcommand{\dimCinline}[1]{\mathrm{dim}_{\mathbb{C}}(#1)}
\newcommand{\dimR}[1]{\mathrm{dim}_{\mathbb{R}}\!\left(#1\right)}
\newcommand{\dimRinline}[1]{\mathrm{dim}_{\mathbb{R}}(#1)}
\newcommand{\sepcap}[1]{\mathcal{SC}\left(#1\right)}
\newcommand{\innerprod}[2]{\left\langle #1, #2 \right\rangle}
\newcommand{\Mod}[1]{\ (\mathrm{mod}\ #1)}
\newcommand\restr[2]{{
  \left.\kern-\nulldelimiterspace 
  #1 
  \vphantom{\big|} 
  \right|_{#2} 
  }}
\newcommand{\comma}[0]{,}
\newcommand{\mapstoleft}[0]{\reflectbox{$\mapsto$}}

\newcommand{\sig}[1]{\operatorname{sig}(#1)}
\newtheorem{theorem}{Theorem}[section]

\newaliascnt{lemma}{theorem}
\newtheorem{lemma}[lemma]{Lemma}
\aliascntresetthe{lemma}

\newaliascnt{proposition}{theorem}
\newtheorem{proposition}[proposition]{Proposition}
\aliascntresetthe{proposition}

\newaliascnt{definition}{theorem}
\theoremstyle{definition}
\newtheorem{definition}[definition]{Definition}
\aliascntresetthe{definition}

\newaliascnt{remark}{theorem}
\theoremstyle{remark}
\newtheorem{remark}[remark]{Remark}
\aliascntresetthe{remark}

\newtheorem{corollary}{Corollary}[theorem]

\numberwithin{equation}{section}

\makeatletter
\let\c@table\c@theorem

\let\c@figure\c@theorem

\makeatother

\titleformat{\section}[block]{\large\bfseries}{\thesection}{1em}{}
\titleformat{\subsection}[block]{\normalsize\bfseries}{\thesubsection}{1em}{}
\titleformat{\subsubsection}[block]{\normalsize\bfseries}{\thesubsubsection}{1em}{}

 \renewcommand{\paragraph}[1]{
     \textbf{#1.} 
 }

\fancypagestyle{plain}{
    \fancyhf{}

}

\newcommand{\mytitle}{Separation Capacity of Scattering Networks}
\title{\mytitle}
\author{
    Konstantin H\"aberle\\
    ETH Zurich\\
    \texttt{haeberlk@ethz.ch}
    \and
    Helmut B\"olcskei\\
    ETH Zurich\\
    \texttt{hboelcskei@ethz.ch}
}
\date{}

\usepackage[capitalize,nameinlink,compress]{cleveref}
\crefname{equation}{}{}
\crefname{subsection}{Subsection}{Subsections}
\crefname{remark}{Remark}{Remarks}
\begin{document}

\maketitle

\begin{abstract}
In this paper, we attempt to enhance the theoretical understanding of convolutional neural networks (CNNs) as feature extractors in classification tasks by analyzing them through the lens of Cover's function-counting theory. Specifically, our focus lies on the notion of separation capacity, a combinatorial quantity derived from counting the number of realizable dichotomies (i.e., binary label assignments).   
Our contributions are threefold. First, we extend Cover's framework by establishing a conceptually insightful and practically useful formulation for the separation capacity. Second, leveraging this formulation, we identify the factors governing the separation capacity of feature extractors that employ a specific CNN architecture, so-called scattering networks, in terms of their network building blocks. Third, we provide practical insights for scattering network design.  
\\[1em]
\end{abstract}

\section{Introduction}
Pattern classification stands as a central task in the field of machine learning \cite{bishop2006pattern}. Methods for solving classification problems often involve feature extraction as a preprocessing step, succeeded by a trainable classifier, such as a support vector machine (SVM) \cite{cortes1995support}. This classification pipeline has led to remarkable successes in practical applications \cite{huang2006large}, in particular, the use of feature extractors based on convolutional neural networks (CNNs) \cite{mallat2012group,bruna2013invariant,anden2014deep} for image and audio signal classification. These CNN-based feature extractors are multi-layered neural-network-type structures, where in each layer convolutional transforms are computed, followed by nonlinearities and pooling operators. 

The foundation for a mathematical framework for these networks was established by Mallat \cite{mallat2012group}, who introduced so-called scattering networks consisting of wavelet transforms followed by a modulus nonlinearity. 
Mallat's framework was subsequently extended in \cite{wiatowski2017deep} to more general scattering networks. Specifically, the theory developed in \cite{wiatowski2017deep} encompasses general convolutional transforms, nonlinearities, and pooling operators, allowing each of these components to vary across different network layers.
Despite various invariance and deformation stability results for scattering networks \cite{mallat2012group,wiatowski2017deep}, an understanding of the reasons for their success in a wide range of practical applications remains elusive.

In this paper, we report an attempt to characterize the theoretical limits of the classification performance of scattering networks à la Wiatowski \& B\"olcskei \cite{wiatowski2017deep}, including Mallat's original construction \cite{mallat2012group}. Specifically, our goal is to employ Cover's \cite{cover1965geometrical} framework for characterizing the separation capacity of feature extractors to scattering networks. 
Although a direct application of Cover's combinatorial techniques for determining separation capacities is not possible in this case, analyzing the separation capacity of scattering networks yields several key insights. In particular, this analysis helps gain a better understanding of the reasons behind the practical success of scattering networks, e.g., by identifying the driving and limiting factors underlying their classification capabilities. 
Furthermore, knowledge of the separation capacity can guide choices for scattering network design and selection in practice.  

\paragraph{Contributions} Besides studying scattering networks in Cover's framework, we have three main contributions. 
\begin{enumerate}[label=(\roman*)]
    \item First, we extend Cover's theory by introducing a novel and very general measure-theoretic approach to separation capacity computations. Our approach significantly simplifies separation capacity computations, as it avoids dealing with certain technicalities inherent in Cover's framework, such as product spaces and the notion of general position. Notably, we derive a necessary and sufficient condition for almost every tuple being in general position with respect to the feature extractor. 
    \item Our second main contribution is the analytic characterization of the factors controlling the separation capacity of scattering networks, particularly in terms of their depths, widths, filters, nonlinearities, and pooling operators, by using tools from complex analysis. We show that there exists a scattering network of low separation capacity, while it is easy to construct scattering networks of high separation capacity. The separation capacity is mainly governed by interplay between the spectral support of the filters and the nonlinearities. Pooling operators generally reduce the separation capacity. 
    \item As our third main contribution, we report practical insights for scattering network design. We establish the design principle that the network should fill out its codomain within the first few layers to attain a high separation capacity. When considering signals on finite cyclic groups, to realize this design principle, it is crucial to employ filters whose spectral support sets do not exhibit any subgroup-type structure. Moreover, the filters should be paired with non-polynomial nonlinearities.
\end{enumerate}

The remainder of this paper is organized as follows.
In \cref{sec:sep-cap}, we review Cover's framework, discuss the notion of separation capacity, and present our extension of Cover's framework. \cref{sec:feature-extractors} is devoted to scattering networks. In \cref{sec:sep-cap-CNN}, we determine the separation capacity of such feature extractors and discuss the impact of filters, nonlinearities, and pooling operators. Finally, in \cref{sec:design-insights}, we provide insights for the design of scattering networks in practice. The notation used throughout this paper is summarized in \cref{app:notation}.   

\section{Separation Capacity}\label{sec:sep-cap}
\subsection{Basic definitions and Cover's function-counting theory}
We begin by reviewing Cover's framework as presented in \cite{cover1965geometrical}.
Consider a set of $N$ points, $F \coloneqq \{f_1,\ldots,f_N\}$, in the $M$-dimensional Euclidean space $\mathbb R^M$ equipped with the standard inner product $\innerprod{f}{g}=g^\mathsf{T}f$, $f,g \in \mathbb R^M$. This space is called the \emph{pattern space}. Our focus is on the binary classification of the points in $F$; namely, we wish to partition the set $F$ into two classes, denoted as $F_+$ and $F_-$. Any such assignment of the points in $F$ to the classes $F_+$ and $F_-$ will be called a \emph{dichotomy}. A natural and simple way to separate the set $F$ into a dichotomy $\{F_+,F_-\}$ is to use (affine) hyperplanes.
The dichotomy $\{F_+,F_-\}$ is said to be \emph{linearly separable} if there exist $w \in \mathbb R^M$ and $t\in \mathbb R$ such that
\begin{align*}
    \innerprod{f}{w} &> t, \quad \text{if $f\in F_+$}, \\
        \innerprod{f}{w} &< t, \quad \text{if $f\in F_-$}. 
\end{align*}
If $t=0$, we say that the dichotomy $\{F_+,F_-\}$ is \emph{homogeneously linearly separable}, see \cref{fig:lin_class} for an illustration. 
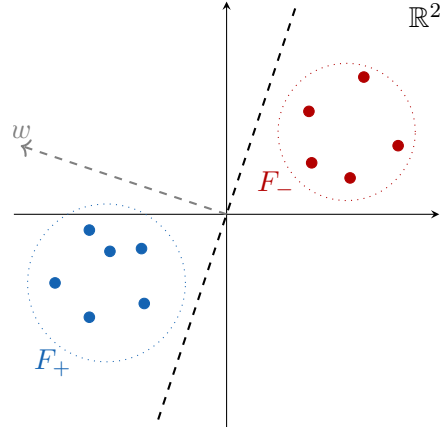
\begin{figure}[t]
    \centering
    \resizebox{0.4\columnwidth}{!}{
        \definecolor{blue2}{RGB}{20,99,178}
\definecolor{green2}{RGB}{0,95,1}
\colorlet{red2}{red!70!black}
\begin{tikzpicture}[scale=1, every node/.style={scale=1}]
    \begin{axis}[
        unit vector ratio*=1 1 1,
        grid=none,
        ticks =none,
        axis lines=middle,
        xmin=-3.1,
        xmax=3.1,
        ymin=-3.1,
        ymax=3.1,
        xticklabels={,,},
        yticklabels={,,}
    ]
    \addplot [only marks, mark=*, color=blue2] coordinates {
        (-2, -1.5)
        (-1.2, -1.3)
        (-2.5, -1)
        (-1.24,-0.5)
        (-2,-0.23)
        (-1.7,-0.54)
    };
    \addplot [only marks, mark=*, color=red2] coordinates {
        (2, 2)
        (1.2, 1.5)
        (2.5, 1)
        (1.24,0.75)
        (1.8,0.53)
    };
    \draw[color=blue2, dotted] (-1.75, -1) circle[radius=1.15];
    \node[color=blue2, anchor=south west] at (-2.95, -2.5) {$F_+$};
    \draw[color=red2, dotted] (1.75, 1.2) circle[radius=1];
    \node[color=red2, anchor=south west] at (0.3, 0.1) {$F_-$};
    \draw[thick,dashed]{} (1,3)--(-1,-3);
    \draw[->,thick,gray,dashed]{} (0,0)--(-3,1);
    \node[gray] at (-3,1.2) {$w$};
    \node[] at (2.9,2.9){$\mathbb R^2$};
    \end{axis}
    \end{tikzpicture}
    }
    \caption{Dichotomy realized by a hyperplane through the origin, i.e., homogeneous linear separation.}
    \label{fig:lin_class}
\end{figure}

To realize dichotomies that are not linearly separable, more general separating surfaces (rather than affine hyperplanes) are required. This can be accomplished by passing to nonlinear transformations $\Phi \colon \mathbb R^M \to \mathbb R^{M'}$. The underlying idea is that such a transformation maps the points in the pattern space to another space, referred to as the \emph{feature space}, where they become linearly separable, as illustrated in \cref{fig:nonlinear_transformation}. Consequently, this yields a nonlinear separating surface in the pattern space, exemplified by the nonlinear transformation in \cref{fig:nonlinear_transformation}, which realizes a circle in the pattern space. We emphasize that in order to implement certain dichotomies, one often employs nonlinear transformations with $M \neq M'$. In particular, when the dimension of the pattern space $M$ is significantly smaller than the size of the dataset $N$, the input data are often mapped into some higher dimensional feature space, i.e., $M<M'$, where a separating hyperplane is then constructed; see, e.g., \cite{hsu2003practical}. The next definition formalizes this idea of using nonlinear transformations to obtain homogeneous linear separability in the feature space.
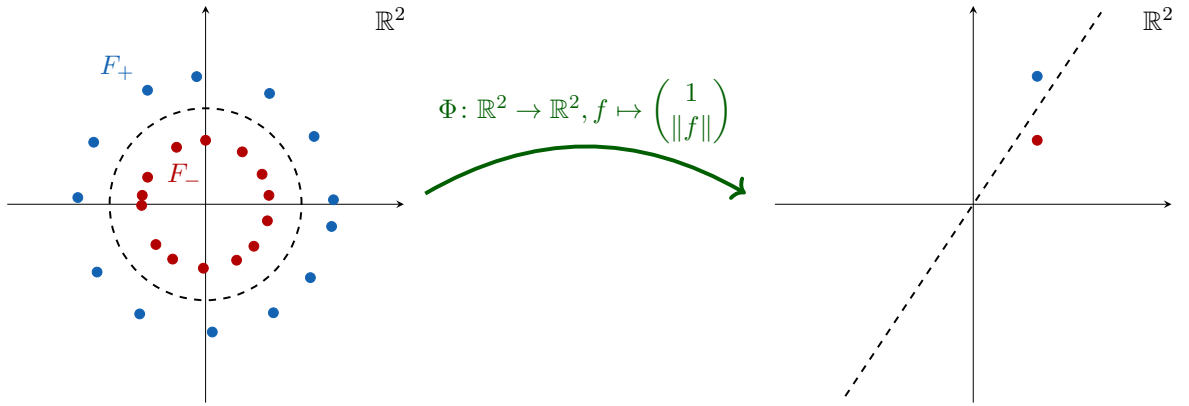
\begin{figure}
    \centering
    \resizebox{\columnwidth}{!}{
        \definecolor{blue2}{RGB}{20,99,178}
\definecolor{green2}{RGB}{0,95,1}
\colorlet{red2}{red!70!black}
\begin{tikzpicture}[scale=1, every node/.style={scale=1}]
    \begin{axis}[
        unit vector ratio*=1 1 1,
        grid=none,
        ticks =none,
        axis lines=middle,
        xmin=-3.1,
        xmax=3.1,
        ymin=-3.1,
        ymax=3.1,
        xticklabels={,,},
        yticklabels={,,},
        at={(0,0)}
    ]
    \node[] at (2.9,2.9){$\mathbb R^2$};
    \foreach \angle in {2, 32, 60, 94, 117, 151, 177, 212, 239, 273, 302, 325, 350}{
        \edef\temp{\noexpand\addplot[only marks, mark=*, color=blue2] coordinates {({2*cos(\angle)},{2*sin(\angle)})};}
        \temp
    }
    \foreach \angle in {8, 28, 55, 90, 117, 155, 172, 181, 219, 239, 268, 299, 319, 345}{
        \edef\temp{\noexpand\addplot[only marks, mark=*, color=red2] coordinates {({cos(\angle)},{sin(\angle)})};}
        \temp
    }
    \draw[thick,dashed] (0, 0) circle[radius=1.5];
    \node[color=blue2, anchor=south west] at (-1.8, 1.8) {$F_+$};
    \node[color=red2, anchor=south west] at (-0.75, 0.1) {$F_-$};
    \end{axis}
    \draw[->,ultra thick,color=green2] (6,3) to[bend left] node[midway,above,inner sep=0pt] {$\Phi\colon \mathbb R^2 \to \mathbb R^2, f \mapsto \begin{pmatrix} 1 \\ \lVert f \rVert \end{pmatrix}$} (10.6,3);
    \begin{axis}[
        unit vector ratio*=1 1 1,
        grid=none,
        ticks =none,
        axis lines=middle,
        xmin=-3.1,
        xmax=3.1,
        ymin=-3.1,
        ymax=3.1,
        xticklabels={,,},
        yticklabels={,,},
        at={(1200,0)}
    ]
    \node[] at (2.9,2.9){$\mathbb R^2$};
    \addplot[only marks, mark=*, color=red2] coordinates {(1,1)};
    \addplot[only marks, mark=*, color=blue2] coordinates {(1,2)};
    \draw[thick,dashed]{} (-2,-3)--(2,3);
    \end{axis}
\end{tikzpicture}}
    \caption{Spherical separating surface realized by application of a nonlinear transformation $\Phi$ followed by homogeneous linear separation in the feature space.}
    \label{fig:nonlinear_transformation}
\end{figure}

\begin{definition}\label{def:phi-sep}
    For $M,M',N \in \mathbb N$, let $F \coloneqq \{f_1,\ldots,f_N\} \subset \mathbb R^M$, and let $\Phi \colon \mathbb R^M \to \mathbb R^{M'}$. A dichotomy $\{F_+,F_-\}$ of $F$ is said to be \emph{$\Phi$-separable} if there exists a vector $w \in \mathbb R^{M'}$ such that
    \begin{align*}
        \innerprod{\Phi(f)}{w} &> 0, \quad \text{if $f\in F_+$}, \\
        \innerprod{\Phi(f)}{w} &< 0, \quad \text{if $f\in F_-$}. 
    \end{align*}
    We call $\{f \in \mathbb R^M \colon \innerprod{\Phi(f)}{w} = 0\}$ the \emph{separating $\Phi$-surface}.
\end{definition}

To develop a quantitative measure for the classification capability of a (nonlinear) transformation $\Phi \colon \mathbb R^M \to \mathbb R^{M'}$, 
one may naturally be interested in the number of $\Phi$-separable dichotomies of an $N$-point set $F \subset \mathbb R^M$, and particularly how it compares to the maximum possible $2^N$ dichotomies that a set with $N$ points can admit. Although the number of $\Phi$-separable dichotomies of $F$ generally depends on both $F$ and $\Phi$ and cannot be determined in closed form, it can be computed precisely, depending only on $M'$ and $N$, if the points in $F$ are ``typical'' in the following sense. 
\begin{definition}\label{def:phi-gen-pos}
    For $M,M'\in \mathbb N$, let $\Phi \colon \mathbb R^M \to \mathbb R^{M'}$. The set $F \coloneqq \{f_1, \ldots, f_N\} \subset \mathbb R^M$, $N \in \mathbb N$, is said to be in \emph{$\Phi$-general position} if every subset of $k$ elements of $\left\{\Phi(f_1), \ldots, \Phi(f_N)\right\} \subset \mathbb R^{M'}$ is linearly independent for all $k \leq \min\{M',N\}$. If this holds for $\Phi = \mathrm{Id} \colon \mathbb R^M \to \mathbb R^M, f \mapsto f$, we simply say $F$ is in \emph{general position}.  
\end{definition}
The number of $\Phi$-separable dichotomies of a set of points $F$ that is in $\Phi$-general position is provided by Cover's celebrated \emph{function-counting theorem}.
\begin{theorem}[Function-counting theorem, \cite{cover1965geometrical}]\label{thm:function-counting}
For $M,M',N \in \mathbb N$, let $F \coloneqq \{f_1,\ldots,f_N\}\subset \mathbb R^M$, and let $\Phi \colon \mathbb R^M \to \mathbb R^{M'}$. The number of $\Phi$-separable dichotomies of $N$ points in $\Phi$-general position in $\mathbb R^M$ is
\begin{align*}
    C(N,M') \coloneqq 2\sum_{k=0}^{M'-1} \binom{N-1}{k}.
\end{align*}
\end{theorem}
\begin{remark}\label{rem:fct-thm-not-gen-pos}
    For $N\leq M'$, we have $C(N,M') = 2\sum_{k=0}^{N-1} \binom{N-1}{k}= 2^N$, i.e., all possible dichotomies can be realized whenever $F$ is in $\Phi$-general position with $N\leq M'$. 
    We further note that if the points in $F$ are not in $\Phi$-general position, there will be fewer than $C(N,M')$ $\Phi$-separable dichotomies (see, e.g., \cite{mitchison1989bounds}).
\end{remark}

Based on the function-counting theorem, let us now discuss examples \cite{cover1965geometrical} characterizing the classification capability of a transformation $\Phi \colon \mathbb R^M \to \mathbb R^{M'}$ to motivate the definition of \emph{separation capacity}, which will be stated later. To this end, assume for now that the $N$-point set $F\subset \mathbb R^M$ is in $\Phi$-general position so that by the function-counting theorem, the number of $\Phi$-separable dichotomies of $F$ is given by $C(N,M')$. We first compare $C(N,M')$ to the maximum number of possible dichotomies by studying the ratio $P(N,M')\coloneqq C(N,M')/2^N$. If a dichotomy of $F$ is chosen uniformly at random from the $2^N$ possible dichotomies, then $P(N,M')$ can be viewed as the probability of separability. In \cref{fig:prob-sep}, the graph of the function $N \mapsto P(N,M')$ for fixed $M' \in \mathbb N$ is shown. We note that this function exhibits a threshold effect at $N=2M'$. In particular, for fixed $\epsilon<1$, we have \cite{cover1965geometrical}
\begin{align}\label{eq:prob-sep-asymptotic}
    \lim_{M' \to \infty} P\left(\lceil 2M'(1-\epsilon) \rceil,M'\right) = \begin{cases}
        1, & \text{if $\epsilon > 0$,} \\
        1/2, & \text{if $\epsilon = 0$,} \\
        0, & \text{if $\epsilon < 0$.}
    \end{cases}
\end{align}
This means that in the regime $M' \to \infty$, a uniformly at random chosen dichotomy of $F$ is $\Phi$-separable with probability tending to one as long as $N<2M'$. In contrast to this, if $N>2M'$, then the probability of this event tends to zero. We also note that $P(2M',M')=\frac 12$, which can be proved using a symmetry argument, see \cref{appendix:sym} for the detailed derivation.
In other words, $50\%$ of all possible dichotomies are $\Phi$-separable if and only if $N=2M'$. 

It is further shown in \cite{cover1965geometrical} that $2M'$ reappears as a critical number in other contexts, such as storing random patterns and ambiguous generalization. Regarding the former, \cite{cover1965geometrical} establishes that the expected value and median of the maximum integer $N$ for which a uniformly at random chosen dichotomy of the $N$-point set $F$ is $\Phi$-separable equals $2M'$. The latter is concerned with the question whether or not a new point can be assigned uniquely to a given dichotomy of the $N$-point set $F$, and \cite{cover1965geometrical} shows that unambiguous generalization becomes possible at $N=2M'$. Both of these concepts, however, are not directly relevant to this paper and will not be discussed in detail here. For a thorough discussion, we refer to \cite{cover1965geometrical}.

These observations lead to the notion of separation capacity of a (nonlinear) transformation $\Phi$, which is not formally stated in \cite{cover1965geometrical}. In the spirit of \cite{mitchison1989bounds,kowalczyk1994separating}, we define it as follows.
\begin{definition}\label{def:sep-cap}
    For $M,M' \in \mathbb N$, let $\Phi \colon \mathbb R^M \to \mathbb R^{M'}$. Denote by $\sepcap{\Phi}$ the largest $N \in \mathbb N$ such that for $(\mathcal{L}^{M})^N$-a.e. $N$-tuple\footnote{To simplify notation, whenever it is clear from the context, we use from now on $F$ to denote both the set $\{f_1,\ldots,f_N\} \subset \mathbb R^M$ and the $N$-tuple $(f_1,\ldots,f_N) \in (\mathbb R^M)^N$.} $F\coloneqq(f_1,\ldots,f_N) \in (\mathbb R^M)^N$ at least $50\%$ of all possible dichotomies of $F$ are $\Phi$-separable. If there is no such $N \in \mathbb N$, set $\sepcap{\Phi} \coloneqq 0$. 
    We call $\sepcap{\Phi}$ the \emph{separation capacity} of $\Phi$.
\end{definition}
Thus, under the assumption that
\begin{align}\label{eq:assumption_phi-gen-pos}
    (\mathcal{L}^{M})^N\!\left(\left\{F \in (\mathbb R^M)^N \colon \text{$F$ is not in $\Phi$-general position}\right\}\right) = 0, \quad \text{for every $N \in \mathbb N$,}
\end{align}
we have, by the above discussion, that $\sepcap{\Phi}$ is the largest $N \in \mathbb N$ such that
\begin{align*}
    \frac{C(N,M')}{2^N} = 2^{-N+1}\sum_{k=0}^{M'-1} \binom{N-1}{k} \geq \frac{1}{2}.
\end{align*}
Using the aforementioned symmetry argument (carried out in \cref{appendix:sym}), it follows that
\begin{align}\label{eq:sep-cap-2dof}
    \sepcap{\Phi}=2M'.
\end{align}

\paragraph{Relation to VC dimension} The celebrated Vapnik--Chervonenkis (VC) dimension \cite{vapnik1971uniform} closely relates to the concept of separation capacity, with both providing measures for the classification capabilities of function classes.
Within Cover's framework, the function classes we consider are induced by (nonlinear) transformations $\Phi \colon \mathbb R^M \to \mathbb R^{M'}$, namely, $\mathcal{H}_\Phi \coloneqq \{f \mapsto \mathrm{sign}(\innerprod{\Phi(f)}{w}) \colon w \in \mathbb R^{M'}\}$. The VC dimension of $\mathcal{H}_\Phi$, denoted $\mathrm{VCdim}(\mathcal{H}_\Phi)$, is then given by the largest $N \in \mathbb N$ for which there exists an $N$-point set $F \subset \mathbb R^M$ such that all possible $2^N$ dichotomies of $F$ are $\Phi$-separable. 
We emphasize that in contrast to this, the separation capacity takes into account $(\mathcal{L}^M)^N$-a.e. $N$-tuple in $(\mathbb R^M)^N$, but requiring only that $50\%$ of all dichotomies be $\Phi$-separable. In general, neither of the two quantities can be upper bounded by the other. Indeed, for the homogeneous linear classifier $\mathcal{H}_{\mathrm{Id}}$, we have $\mathrm{VCdim}(\mathcal{H}_{\mathrm{Id}}) = M$ by \cite[Theorem 9.2]{shalev2014understanding} and $\sepcap{\mathrm{Id}} = 2M$ by \cref{eq:sep-cap-2dof}, as \cref{eq:assumption_phi-gen-pos} clearly holds for $\Phi=\mathrm{Id}$. On the other hand, if 
\begin{align*}
    \Phi(f) = \begin{cases}
        f, &\text{if $f \in \mathbb R^M\setminus A$,} \\
        0, &\text{if $f \in A$,}
    \end{cases} \quad f \in \mathbb R^M,
\end{align*}
where $A \subset \mathbb R^M$ is such that $\mathcal{L}^M(A)>0$ and $\mathcal{L}^M(\mathbb R^M\setminus A)>0$, then $\mathrm{VCdim}(\mathcal{H}_\Phi) = M$ by the same argument as in the case $\mathcal{H}_{\mathrm{Id}}$. However, no dichotomy of any $1$-point set $\{f\}$ with $f \in A$ is $\Phi$-separable, so that $\sepcap{\Phi} = 0$ by \cref{def:sep-cap}. 
In the context of feature extraction, studying the separation capacity rather than the VC dimension is more natural as the former provides a more intuitive geometric perspective. Specifically, the separation capacity sheds light on the geometry of the decision surface by describing its degrees of freedom. This will be further illustrated in the next subsection.

\begin{figure}
    \centering
    \resizebox{\columnwidth}{!}{
            \definecolor{blue2}{RGB}{20,99,178}
\definecolor{green2}{RGB}{0,95,1}
\colorlet{red2}{red!70!black}
\begin{tikzpicture}[scale=1, every node/.style={scale=0.4},
 evaluate={
        function C(\n,\m) {
            if \m == 1 then {
                return 1;
            } else {
                return C2(\n,\m-1)+binom(\n-1,\m-1);
            };
        };
        function C2(\n,\m) {
            if \m == 1 then {
                return 1;
            } else {
                return C(\n, \m-1)+binom(\n-1,\m-1);
            };
        };
        function binom(\n,\k) {
            if \k>\n then {
                return 0;
            } else {
                return \n!/((\n-\k)!*\k!);
            };
        };
    },
]
    \begin{axis}[
        unit vector ratio*=1 1 1,
        at={(0,0)},
        xmin=-0.1,
        xmax=4,
        ymin=-0.1,
        ymax=1.1,
        clip=false,
        xtick={0,1,2,3,4},
        ytick={0,0.25,...,1},
        yticklabels={0,0.25,...,1},
        xticklabels={0,1,2,3,4},
        ylabel={$C(N,M')/2^N$},
        xlabel={$N/M'$},
        legend cell align={left}
    ]
    \addplot[color=blue2,mark=*, mark size=0.75pt,samples at={0.2, 0.4, 0.6, 0.8, 1.0, 1.2, 1.4, 1.6, 1.8, 2.0, 2.2, 2.4, 2.6, 2.8, 3.0, 3.2, 3.4, 3.6, 3.8, 4.0},variable=\x] {(1/2)^(\x*5)*2*C(5*\x,5)};
    \addplot[color=red2,mark=*, mark size=0.75pt,samples at = {0.02,0.2, 0.4, 0.6, 0.8, 1.0, 1.2, 1.4, 1.6, 1.7, 1.8, 1.9, 1.96, 2.0, 2.06, 2.1, 2.2, 2.3, 2.4, 2.6, 2.8, 3.0, 3.2, 3.4, 3.6, 3.8, 4.0},variable=\x] {(1/2)^(\x*50)*2*C(50*\x,50)};
    \legend{$M'=5$,$M'=50$};
    \draw[dotted]{} (0,1)--(2,1);
    \draw[dotted]{} (2,1)--(2,0);
    \draw[dotted]{} (2,0)--(4,0);
\end{axis}
\end{tikzpicture}
    }
    \caption{Probability of separability. The gray dotted line corresponds to \cref{eq:prob-sep-asymptotic}.}
    \label{fig:prob-sep}
\end{figure}
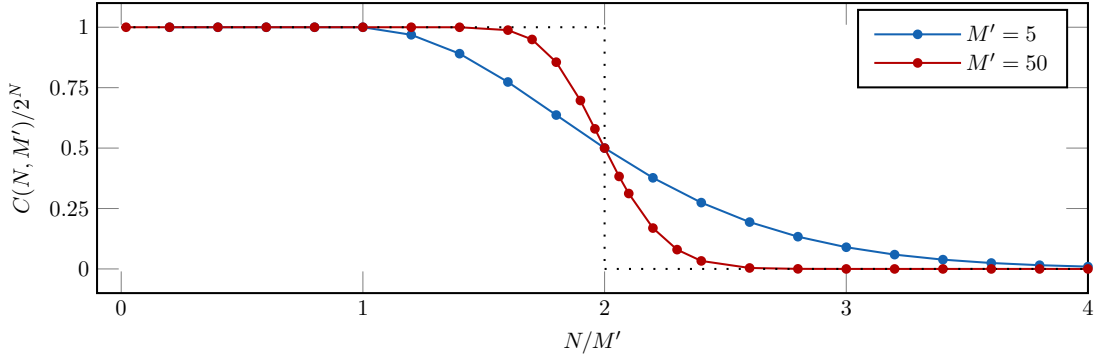

\subsection{Separation capacity of general (nonlinear) transformations}
We have seen above that the separation capacity of a (nonlinear) transformation $\Phi \colon \mathbb R^M \to \mathbb R^{M'}$ is given by $\sepcap{\Phi}=2M'$ under the assumption that \cref{eq:assumption_phi-gen-pos} holds. Naturally, two guiding questions arise that extend beyond Cover's \cite{cover1965geometrical} original framework: 
\begin{enumerate}
    \item Is there a simple way of verifying which (nonlinear) transformations $\Phi$ satisfy \cref{eq:assumption_phi-gen-pos}? Such a method would facilitate separation capacity computations, allowing us not only to determine $\sepcap{\Phi}$ directly according to \cref{eq:sep-cap-2dof} but also to enhance our understanding of the notion of $\Phi$-general position from a measure-theoretic perspective. 
    \item How can the separation capacity be determined if \cref{eq:assumption_phi-gen-pos} does not hold? Generally, we cannot expect a nonlinear transformation $\Phi$ to always satisfy \cref{eq:assumption_phi-gen-pos}, such as when the image of $\Phi$ lies on a linear subspace of $\mathbb R^{M'}$. This can occur, notably, in the case of scattering networks. Therefore, it is crucial to develop a method for determining the separation capacity under this condition, as this is a key step in our objective of computing the separation capacity of scattering networks.
\end{enumerate}
In this subsection, these questions will be addressed.  
The method derived from the first question will serve as the foundation for answering the second. Notably, we provide a conceptually insightful and practically useful formulation for the separation capacity that applies to general transformations $\Phi$, with the sole requirement that $\Phi$ be Lebesgue measurable.
To do so, let us first extend the notion of $\Phi$-general position. 
\begin{definition}
    For $M,M'\in \mathbb N$, let $\Phi \colon \mathbb R^M \to \mathbb R^{M'}$, and fix $M^\natural \in \mathbb N$ with $M^\natural \leq M'$. The set $F \coloneqq \{f_1, \ldots, f_N\} \subset \mathbb R^M$, $N \in \mathbb N$, is said to be in \emph{$(M^\natural,\Phi)$-general position} if every subset of $k$ elements of $\left\{\Phi(f_1), \ldots, \Phi(f_N)\right\} \subset \mathbb R^{M'}$ is linearly independent for all $k \leq \min\{M^\natural,N\}$. 
\end{definition}
Recalling \cref{def:phi-gen-pos}, it becomes evident that $(M',\Phi)$-general position is equivalent to $\Phi$-general position for $\Phi \colon \mathbb R^M \to \mathbb R^{M'}$. We next establish the following key lemma, which provides inter alia a necessary and sufficient condition for \cref{eq:assumption_phi-gen-pos} to hold.
\begin{lemma}\label{L:equiv_span_phi-gen_pos}
    For $M,M' \in \mathbb N$, let $\Phi \colon \mathbb R^M \to \mathbb R^{M'}$ be Lebesgue measurable. Let $M^\natural, N\in \mathbb N$ with $M^\natural \leq M' \leq N$. The set of $N$-tuples $F \coloneqq (f_1, \ldots, f_N) \in (\mathbb R^M)^N$ which are not in $(M^\natural,\Phi)$-general position has $(\mathcal{L}^{M})^N$-measure zero if and only if there is no $\mathcal{L}^M$-measurable set $A \subseteq \mathbb R^M$ with $\mathcal{L}^M(A)>0$ such that 
    \begin{align}\label{eq:cover_equiv_phi-gen_pos}
        \dimR{\spanR{\Phi(A)}} < M^\natural.
    \end{align} 
\end{lemma}
\begin{remark}[Measurability assumption]
    The assumption of $\Phi$ being Lebesgue measurable ensures that the set of $N$-tuples $F \coloneqq (f_1, \ldots, f_N) \in (\mathbb R^M)^N$ which are not in $(M^\natural,\Phi)$-general position, denoted by $P_{M^\natural,\Phi}^N$, is $(\mathcal{L}^{M})^N$-measurable, specifically if it is not a nullset. Indeed, we have
    \begin{align*}
        P_{M^\natural,\Phi}^N = \bigcup_{1 \leq j_1 < \cdots < j_{M^\natural} \leq N} \pi^{-1}_{j_1, \ldots, j_{M^\natural}} \left( \bigcap_{1 \leq k_1 < \cdots < k_{M^\natural} \leq M'}\delta_{k_1,\ldots,k_{M^\natural}}^{-1}(\{0\})\right),
    \end{align*}
    where $\pi_{j_1, \ldots, j_{M^\natural}} \colon (\mathbb R^M)^{N} \to (\mathbb R^M)^{M^\natural}, (f_1,\ldots,f_N) \mapsto (f_{j_1}, \ldots, f_{j_{M^\natural}})$ is the canonical projection onto the coordinates $(j_1, \ldots, j_{M^\natural})$ and
    \begin{align*}
        \delta_{k_1,\ldots,k_{M^\natural}} \colon (\mathbb R^M)^{M^\natural} \to \mathbb R, (f_1,\ldots,f_{M^\natural}) \mapsto \det \begin{pmatrix}
            \Phi_{k_1}(f_1) &\cdots & \Phi_{k_1}(f_{M^\natural}) \\
            \vdots & \ddots & \vdots \\
            \Phi_{k_{M^\natural}}(f_1) &\cdots & \Phi_{k_{M^\natural}}(f_{M^\natural})
        \end{pmatrix}.    
    \end{align*}
    It follows that $P_{M^\natural,\Phi}^N$ is $(\mathcal{L}^{M})^N$-measurable whenever $\Phi$ is Lebesgue measurable.
\end{remark}
\begin{proof}[Proof of \cref{L:equiv_span_phi-gen_pos}]
    We first show the contrapositive of the ``only if'' statement.
    Suppose that there exists an $\mathcal{L}^M$-measurable set $A \subseteq \mathbb R^M$ with $\mathcal{L}^M(A)>0$ such that \cref{eq:cover_equiv_phi-gen_pos} holds. Then, 
    \begin{align*}
        \dimR{\spanR{\{\Phi(f_1),\ldots,\Phi(f_N)\}}} < M^\natural, \qquad \text{for all $f_1,\ldots,f_{N}\in A$},
    \end{align*}
     which implies that every subset of $M^\natural$ elements of $\{\Phi(f_1),\ldots,\Phi(f_N)\}$ is linearly dependent. We can hence conclude that all $N$-tuples in $A^N$ are not in $(M^\natural,\Phi)$-general position. Since $(\mathcal{L}^{M})^N(A^N)=N\mathcal{L}^M(A)>0$, the ``only if'' statement follows. 
     
    We next establish the converse statement. Suppose that there is no $\mathcal{L}^M$-measurable set $A \subseteq \mathbb R^M$ of positive $\mathcal{L}^M$-measure such that \cref{eq:cover_equiv_phi-gen_pos} holds. To prove the assertion, we proceed by induction on $N$. For $N=1$, $\{f_1\}$ is in $(M^\natural,\Phi)$-general position if and only if $\Phi(f_1) \neq 0$. Setting $A \coloneqq \{f \in \mathbb R^M \colon \Phi(f)=0\}$, we have by our assumption \cref{eq:cover_equiv_phi-gen_pos}, $\mathcal{L}^M(A)=0$ as $\dimR{\spanR{A}}=0$. This proves the claim for $N=1$. Now, suppose that the claim is true for $N-1$, i.e., the set of $(N-1)$-tuples which are not in $(M^\natural,\Phi)$-general position has $(\mathcal{L}^M)^{N-1}$-measure zero. Let $F \in (\mathbb R^M)^{N-1}$ be in $(M^\natural,\Phi)$-general position, and fix an arbitrary $f_N \in \mathbb R^M$. Then, $(f_1,\ldots,f_{N-1},f_N)$ is in $(M^\natural,\Phi)$-general position if and only if $\Phi(f_N) \notin \spanRinline{\{\Phi(f_{j_\ell})\}_{\ell=1}^{L-1}}$ for every $1\leq j_1 < \cdots < j_{L-1} \leq N-1$, where $L\coloneqq \min\{M^\natural,N\}$. Set $A_{j_1,\ldots,j_{L-1}} \coloneqq \{f \in \mathbb R^M \colon \Phi(f) \in \spanRinline{\{\Phi(f_{j_\ell})\}_{\ell=1}^{L-1}}\}$. As $\dimRinline{\spanRinline{\Phi(A_{j_1,\ldots,j_{L-1}})}} \leq L-1<M^\natural$ by definition of $A_{j_1,\ldots,j_{L-1}}$, it follows from our assumption \cref{eq:cover_equiv_phi-gen_pos} that $\mathcal{L}^M(A_{j_1,\ldots,j_{L-1}})=0$.
    This, in turn, implies that the set $A(f_1,\ldots,f_{N-1}) \coloneqq \bigcup_{1\leq j_1 < \cdots < j_{L-1} \leq N-1}A_{j_1,\ldots,j_{L-1}}$ has $\mathcal{L}^M$-measure zero. 
    Denoting by $S_N$ the set of $N$-tuples which are not in $(M^\natural,\Phi)$-general position, we thus have
    \begin{align*}
        S_N \subseteq \left(S_{N-1} \times \mathbb R^M\right) \cup \left(\left\{F \in (\mathbb R^M)^N \colon (f_k)_{k=1}^{N-1} \in (\mathbb R^M)^{N-1}\setminus S_{N-1},\,\, f_N \in A(f_1,\ldots,f_{N-1}) \right\}\right).
    \end{align*}
    Application of the induction hypothesis (i.e., $(\mathcal{L}^M)^{N-1}(S_{N-1}) =0$) together with Fubini's theorem and the fact that $\mathcal{L}^M(A(f_1,\ldots,f_{N-1}))=0$, for all $(f_k)_{k=1}^{N-1} \in (\mathbb R^M)^{N-1}\setminus S_{N-1}$, yields $(\mathcal{L}^M)^{N}(S_N)=0$. This completes the proof.  
\end{proof}

We can now state the main result of this section, namely the following formulation for the separation capacity of a Lebesgue measurable transformation $\Phi$.

\begin{theorem}\label{thm:sep-cap}
    Let $M,M' \in \mathbb N$, and let $\Phi \colon \mathbb R^M \to \mathbb R^{M'}$ be Lebesgue measurable. The separation capacity of $\Phi$ is given by
    \begin{align}\label{eq:sep_cap_formula}
        \sepcap{\Phi} = 2\min_{\substack{A \subseteq \mathbb R^M \\ \mathcal{L}^M(A)>0}} \dimR{\spanR{\Phi(A)}}. 
    \end{align}
    In particular, if there is no  $\mathcal{L}^M$-measurable set $A \subseteq \mathbb R^M$ with $\mathcal{L}^M(A)>0$ such that 
    \begin{align}\label{eq:cover_cond_cap}
        \dimR{\spanR{\Phi(A)}} < \dimR{\spanR{\Phi(\mathbb R^M)}}, 
    \end{align}
    then the separation capacity of $\Phi$ is 
    \begin{align}\label{eq:cover_cap}
        \sepcap{\Phi} = 2 \cdot \dimR{\spanR{\Phi(\mathbb R^M)}}.
    \end{align}
\end{theorem}
\begin{proof}
    Let us first prove that \cref{eq:cover_cap} holds if \cref{eq:cover_cond_cap} is satisfied. Set $\widetilde{M}\coloneqq \dimRinline{\spanRinline{\Phi(\mathbb R^M)}}$. If $\widetilde{M} = M'$, \cref{eq:cover_cap} follows immediately from \cref{L:equiv_span_phi-gen_pos} and \cref{eq:sep-cap-2dof}. Otherwise, if $\widetilde{M} < M'$, there exists a linear map $\widetilde{\pi} \colon \mathbb R^{M'} \to \mathbb R^{\widetilde{M}}$ such that $\widetilde{\Phi} \coloneqq \widetilde{\pi} \circ \Phi$ satisfies $\dimRinline{\spanRinline{\widetilde{\Phi}(\mathbb R^M)}} = \widetilde{M}$, i.e., $\ker(\widetilde{\pi}) \cap \spanR{\Phi(\mathbb R^M)} = \{0\}$. Then, $\Phi$-separability and $\widetilde{\Phi}$-separability are equivalent. Indeed, it immediately follows from the definition of $\widetilde{\Phi}$ that $\widetilde{\Phi}$-separability implies ${\Phi}$-separability. The reverse implication holds because $(\Phi(\mathbb R^M))^\perp = \ker(\widetilde{\pi})$. Application of \cref{L:equiv_span_phi-gen_pos} and \cref{eq:sep-cap-2dof} yields \cref{eq:cover_cap}.

    We now turn to the proof of \cref{eq:sep_cap_formula}. Let $A^{\natural} \subseteq \mathbb R^M$ be an $\mathcal{L}^M$-measurable set of positive $\mathcal{L}^M$-measure such that
    \begin{align*}
        \dimR{\spanR{\Phi(A^\natural)}} = \min_{\substack{A \subseteq \mathbb R^M \\ \mathcal{L}^M(A)>0}} \dimR{\spanR{\Phi(A)}} \eqqcolon M^{\natural}.
    \end{align*}
   Similarly to the argument in the first part of the proof, there exists a linear map $\pi^\natural \colon \mathbb R^{M'} \to \mathbb R^{M^\natural}$ such that $\Phi^\natural \coloneqq \pi^\natural \circ \Phi$ satisfies $\dimRinline{\spanRinline{\Phi^\natural(A^\natural)}} = M^\natural$, and 
   it follows that every $N$-tuple $F \in (A^{\natural})^N$ is $\Phi$-separable if and only if it is $\Phi^\natural$-separable, where $N \in \mathbb N$. As the number of $\Phi^\natural$-separable dichotomies of $F \in (A^{\natural})^N$ is at most $C(N,M^\natural)$, see \cref{rem:fct-thm-not-gen-pos}, and as $(\mathcal{L}^M)^N((A^\natural)^N)=N\mathcal{L}^M(A^\natural)>0$, we have by \cref{eq:sep-cap-2dof},
   $\sepcap{\Phi}\leq 2 M^\natural$. It remains to show that equality holds. Using the definition of $M^\natural$ and \cref{L:equiv_span_phi-gen_pos}, we can deduce that for $(\mathcal{L}^{M})^N$-a.e. $N$-tuple $F = (f_1,\ldots,f_N) \in (\mathbb R^M)^N$, every subset of $\{\Phi(f_1),\ldots,\Phi(f_N)\}$ of $M^\natural$ elements is linearly independent, where $N \geq M^\natural$. 
   As a result of \cite[Theorem 2.1]{haberle2026function}, the number of $\Phi$-separable dichotomies of $(\mathcal{L}^{M})^N$-a.e. $F \in (\mathbb R^M)^N$ is at least $C(N,M^{\natural})$. Thanks to \cref{eq:sep-cap-2dof}, we obtain $\sepcap{\Phi}\geq 2 M^\natural$, and thus $\sepcap{\Phi}= 2 M^\natural$, which completes the proof.      
\end{proof}

This theorem thus establishes a method for separation capacity computations of general measurable transformations $\Phi$. We emphasize that it remains valid in particular if \cref{eq:assumption_phi-gen-pos} does not hold, thereby addressing our second guiding question.    
Additionally, \cref{eq:sep_cap_formula} offers other significant advantages. It bypasses the technicalities associated with working in the space of $N$-tuples. 
Furthermore, \cref{eq:sep_cap_formula} and hence condition \cref{eq:cover_cond_cap} circumvents dealing with the notion of $\Phi$-general position, greatly simplifying practical computations. Specifically, verifying \cref{eq:cover_cond_cap} rather than \cref{eq:assumption_phi-gen-pos} facilitates these computations, addressing our first guiding question.  
For instance, this becomes evident when studying the separation capacity of real-analytic transformations $\Phi$, as we shall see in the proof of the following corollary.
\begin{corollary}\label{cor:real-analytic} 
     Let $M,M' \in \mathbb N$. The separation capacity of a real-analytic $\Phi \colon \mathbb R^M \to \mathbb R^{M'}$ is given by
     \begin{align*}
        \sepcap{\Phi} = 2 \cdot \dimR{\spanR{\Phi(\mathbb R^M)}}.
    \end{align*}
\end{corollary}
\begin{proof}
    The claim follows because condition \cref{eq:cover_cond_cap} holds in particular for real-analytic $\Phi \colon \mathbb R^M \to \mathbb R^{M'}$. Indeed, let $h \in \mathbb R^{M'}$ and $A \subseteq \mathbb R^M$ be an $\mathcal{L}^M$-measurable set of positive $\mathcal{L}^M$-measure. Then,
    \begin{align*}
        \left(\innerprod{h}{\Phi(f)} =0, \forall f \in A\right) \implies \left(\innerprod{h}{\Phi(f)} =0, \forall f \in \mathbb R^M\right), 
    \end{align*}
    as $f \mapsto \innerprod{h}{\Phi(f)}$ is real-analytic and as zero sets of nontrivial real-analytic functions are of $\mathcal{L}^M$-measure zero \cite{krantz2002primer,mityagin2020zerorealanalytic}. We thus have 
    \begin{align*}
        \dimR{\left(\Phi(A)\right)^\perp} = \dimR{\left(\Phi(\mathbb R^M)\right)^\perp}, 
    \end{align*}
    which is equivalent to \cref{eq:cover_cond_cap}. 
\end{proof}
The fact that the separation capacity of a real-analytic map $\Phi \colon \mathbb R^M \to \mathbb R^{M'}$ is given by \cref{eq:cover_cap} is also established in \cite{kowalczyk1994separating}, albeit based on a different approach, which pertains to real-analytic maps only.  

Finally, we would like to highlight that the usefulness of our extension of Cover's framework is further demonstrated in \cref{sec:sep-cap-CNN}, where we analyze the separation capacity of scattering networks.

\section{Scattering Networks}\label{sec:feature-extractors}
Shifting our attention, we now review scattering networks as presented in \cite{mallat2012group,wiatowski2017deep} and then show how they can be viewed as feature extractors in the framework developed in the previous section. Our focus is on a finite-dimensional version of scattering networks, specifically those defined on finite cyclic groups.
Scattering networks achieve remarkable classification results of, e.g., image \cite{bruna2013invariant} or audio \cite{anden2014deep} signals, when used in conjunction with SVMs. Additionally, they have been successfully applied to biomedical data \cite{warrick2020arrhythmia}, multi-scale time series such as, e.g., financial and turbulence time series  \cite{leonarduzzi2019maximum,morel2022scale}, the estimation of quantum molecular energies \cite{eickenberg2017solid}, and various astrophysical applications \cite{cheng2020new,valogiannis2022towards,cheng2024scattering}.

Consider inputs in $\mathbb C^{\mathbb Z /M \mathbb Z}$, the space of complex-valued functions on the finite cyclic group $\mathbb Z/M\mathbb Z$, where $M \in \mathbb N$. This is in contrast to \cite{mallat2012group,wiatowski2017deep}, where the inputs are functions in $L^2(\mathbb R^d)$. 
The basic building blocks of a scattering network are the elements of a sequence of the form $\Omega \coloneqq \{(\Psi_n,\rho_n,P_n)\}_{n \in \mathbb N}$, where we associate with the $n$th network layer the triplet $(\Psi_n,\rho_n,P_n)$ consisting of the following objects: (i) a \emph{frame} (i.e., a redundant spanning set, see, e.g., \cite{christensen2003introduction,kaiser1994friendly}) $\Psi_n$ generated by the family of functions $\{\chi_n\} \cup \{g_{\lambda_n}\}_{\lambda_n \in \Lambda_n} \subseteq \mathbb C^{\mathbb Z/M\mathbb Z}$, where $\Lambda_n$ is a countable index set, satisfying the frame condition\footnote{Note that $\Psi_n = \bigcup_{k=0}^{M-1} \left(\{T_k\chi_n^*\} \cup \{T_kg_{\lambda_n}^*\}_{\lambda_n \in \Lambda} \right)$, where $T_k$ is the translation operator, and where the superscript $*$ denotes involution.} \cite{ali1993continuous}
\begin{align}\label{eq:frame-condition}
    A_n \lVert f \rVert^2 \leq \lVert f \ast \chi_n \lVert^2 + \sum_{\lambda_n \in \Lambda_n} \lVert f \ast g_{\lambda_n} \rVert^2 \leq B_n \lVert f \rVert^2, \quad f \in \mathbb C^{\mathbb Z/M \mathbb Z},
\end{align}
with $0<A_n\leq B_n < \infty$, 
(ii) a \emph{nonlinearity} $\rho_n \colon \mathbb C \to \mathbb C$, and (iii) a \emph{pooling operator} $P_n \colon \mathbb C^{\mathbb Z/M \mathbb Z} \to \mathbb C^{\mathbb Z/M \mathbb Z}$. The sequence $\Omega=\{(\Psi_n,\rho_n,P_n)\}_{n \in \mathbb N}$ is referred to as \emph{module sequence} in \cite{wiatowski2017deep}.
For each $\lambda_n \in \Lambda_n$, define the operator $U[\lambda_n] \colon \mathbb C^{\mathbb Z/M\mathbb Z} \to \mathbb C^{\mathbb Z/M\mathbb Z}$ according to
\begin{align*}
    U[\lambda_n]f \coloneqq P_n(\rho_n(f \ast g_{\lambda_n})), \quad f \in \mathbb C^{\mathbb Z/M \mathbb Z},
\end{align*}
where $\left(\rho_n(f \ast g_{\lambda_n})\right)(k)\coloneqq \rho_n((f \ast g_{\lambda_n})(k))$, $k \in \mathbb Z/M\mathbb Z$. 
Extend this operator to paths $q = (\lambda_1, \ldots, \lambda_n) \in \Lambda_1 \times \cdots \times \Lambda_n \eqqcolon \Lambda_1^n$ according to
\begin{align*}
    U[q]f \coloneqq U[\lambda_n] \cdots U[\lambda_1]f, \quad f \in \mathbb C^{\mathbb Z/M \mathbb Z}.
\end{align*}
We further set $\Lambda_1^0 \coloneqq \{e\}$ and $U[e]f = f$, where $e \coloneqq \emptyset$ denotes the empty path. For $f \in \mathbb C^{\mathbb Z/M \mathbb Z}$ and $q \in \Lambda^n_1$ with $n \in \mathbb N_0$, the function $U[q]f$ is frequently called \emph{feature map}.  
The scattering network of depth $n_{\mathrm{d}} \in \mathbb N$ is given by
\begin{align}\label{eq:CNN-feature-extractor}
    \Phi \colon \mathbb C^{\mathbb Z/M \mathbb Z} \to \left(\mathbb C^{\mathbb Z/M \mathbb Z}\right)^{\bigcup_{n=0}^{n_{\mathrm{d}}} \Lambda_1^n}, \quad f \mapsto \bigcup_{n=0}^{n_{\mathrm{d}}} \Phi^n(f),
\end{align}
where $\Phi^n(f) = \{(U[q]f) \ast \chi_{n+1}\}_{q \in \Lambda_1^n}$ denotes the output of the $n$th network layer, see \cref{fig:scattering-tree}. The function $\Phi(f)$, where $f \in \mathbb C^{\mathbb Z/M \mathbb Z}$, is often referred to as \emph{feature vector}. 

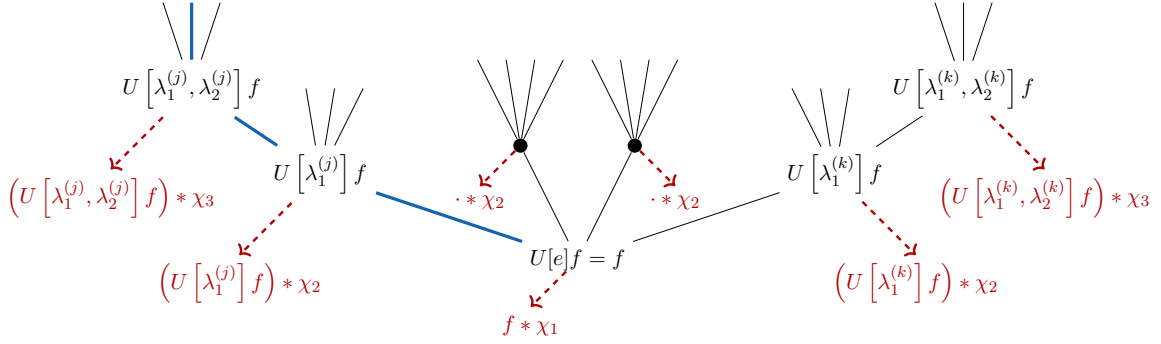
\begin{figure}[t!]
    \centering
    \resizebox{\columnwidth}{!}{
\definecolor{colour1}{rgb}{0.0, 0.0, 0.0}
\definecolor{colour2}{rgb}{0.0, 0.0, 0.0}
\definecolor{blue2}{RGB}{20,99,178}
\definecolor{green2}{RGB}{0,95,1}
\colorlet{red2}{red!70!black}
\begin{tikzpicture}[
	level distance = 15mm,
	level 1/.style={sibling distance=30mm},
	level 2/.style={sibling distance=15mm},
        level 3/.style={sibling distance=5mm},
	grow'=up]
	\node[color=colour1] (a0) {$U[e]f=f$}
	child[draw=blue2,ultra thick]{
		node[opacity=1] (b1) {$U\left[\lambda_1^{(j)}\right]f$}
		child[draw=blue2,ultra thick] {
                node[opacity=1] (c1) {$U\left[\lambda_1^{(j)},\lambda_2^{(j)}\right]f$}
                child[draw=black,thin] child[draw=blue2,ultra thick] child[draw=black,thin]
            } 
            child[sibling distance=5mm,draw=black,thin] child[sibling distance=5mm,draw=black,thin] child[sibling distance=5mm,draw=black,thin]
	}
        child[sibling distance=20mm, level distance=20mm]{
		node[fill,color=colour2,circle,inner sep=2.5pt] (b2) {}
		child[sibling distance=5mm,level distance=15mm] child[sibling distance=5mm,level distance=15mm] child[sibling distance=5mm,level distance=15mm] child[sibling distance=5mm,level distance=15mm]
	}
	child[sibling distance=20mm,level distance=20mm]{
		node[fill,color=colour2,circle,inner sep=2.5pt] (b3) {}
		child[sibling distance=5mm,level distance=15mm] child[sibling distance=5mm,level distance=15mm] child[sibling distance=5mm,level distance=15mm] child[sibling distance=5mm,level distance=15mm]
	}
        child{
		node[opacity=1] (b4) {$U\left[\lambda_1^{(k)}\right]f$}
		child[sibling distance=5mm] child[sibling distance=5mm] child[sibling distance=5mm]
            child {
                node[opacity=1] (c2) {$U\left[\lambda_1^{(k)},\lambda_2^{(k)}\right]f$}
                child child child
            }
	};
	\draw[->,dashed,very thick,red2] (-0.2,-0.2) -- ++(225:0.875cm) node[below] {$f\ast \chi_1$};
	\draw[->,dashed,very thick,red2] (b1) -- ++(225:2cm) node[below] {$\left(U\left[\lambda_1^{(j)}\right]f\right) \ast \chi_2$};
        \draw[->,dashed,very thick,red2] (c1) -- ++(225:2cm) node[below] {$\left(U\left[\lambda_1^{(j)},\lambda_2^{(j)}\right]f\right) \ast \chi_3$};
	\draw[->,dashed,very thick,red2] (b2) -- ++(225:1cm) node[below] {$\cdot \ast \chi_2$};
	\draw[->,dashed,very thick,red2] (b3) -- ++(315:1cm) node[below] {$\cdot \ast \chi_2$};
        \draw[->,dashed,very thick,red2] (b4) -- ++(315:2cm) node[below] {$\left(U\left[\lambda_1^{(k)}\right]f\right)\ast \chi_2$};
        \draw[->,dashed,very thick,red2] (c2) -- ++(315:2cm) node[below] {$\left(U\left[\lambda_1^{(k)},\lambda_2^{(k)}\right]f\right)\ast \chi_3$};
\end{tikzpicture}}
    \caption{Structure of a scattering network. The path $(\lambda_1^{(j)}, \lambda_2^{(j)},\ldots)$ is indicated in blue. The outputs of each node are highlighted in red.}
    \label{fig:scattering-tree}
\end{figure}

\begin{remark}\label{rem:scattering-mallat}
    Scattering networks, as introduced by Mallat \cite{mallat2012group}, are built from the module sequence $\{(\Psi_\mathrm{wvt},\lvert \cdot \rvert, \mathrm{Id})\}_{n \in \mathbb N}$, where $\Psi_{\mathrm{wvt}}$ is a so-called \emph{wavelet frame} \cite{daubechies1992ten,vashisht2017necessary} and $\lvert \cdot \rvert$ is the modulus nonlinearity. The framework in \cite{wiatowski2017deep} allows for general frames, Lipschitz continuous nonlinearities, and Lipschitz continuous pooling operators. 
\end{remark}

We conclude this section by noting that, in the language of \cref{sec:sep-cap}, the scattering network realizes the nonlinear transformation $\Phi$. This becomes evident by the chain of identifications $\mathbb C^{\mathbb Z/M\mathbb Z} \simeq \mathbb C^M \simeq \mathbb R^{2M}$. Regarding the latter identification, let us discuss its interpretation in the context of binary classification. To elaborate, consider the nonlinear transformation $\Phi \colon \mathbb C^M \to \mathbb C^{M'}$. 
The identification $\mathbb C^{M'} \simeq \mathbb R^{2M'}$ suggests that a dichotomy $\{F_+,F_-\}$ of an $N$-point set $F \subset \mathbb C^M$ is $\Phi$-separable if there is a $w \in \mathbb C^{M'}$ such that\footnote{Equivalently, one may consider the sign of $\Im(\innerprod{\Phi(f)}{w})$.} 
\begin{align*}
    \Re(\innerprod{\Phi(f)}{w}) &>0, \quad \text{if $f \in F_+$},\\
    \Re(\innerprod{\Phi(f)}{w}) &<0, \quad \text{if $f \in F_-$}.
\end{align*}
Indeed, we have
\begin{align*}
    \Re(\innerprod{\Phi(f)}{w}) = \innerprod{\Re(\Phi(f))}{\Re(w)} + \innerprod{\Im(\Phi(f))}{\Im(w)} = \innerprod{\begin{pmatrix}
        \Re(\Phi(f)) \\ \Im(\Phi(f))
    \end{pmatrix}}{\begin{pmatrix}
        \Re(w) \\ \Im(w)
    \end{pmatrix}}.
\end{align*}
Thus, the separation capacity of the complex-valued map $\Phi \colon \mathbb C^M \to \mathbb C^{M'}$ is understood as the separation capacity of the associated real-valued map 
\begin{align}\label{eq:real-val-phi}
    \mathbb C^M \to \mathbb R^{2M'},\,\, f \mapsto \begin{pmatrix}
    \Re(\Phi(f)) \\ \Im(\Phi(f))
\end{pmatrix}.
\end{align}

\section{Separation Capacity of Scattering Networks} \label{sec:sep-cap-CNN}
We now turn to analyzing the separation capacity of scattering networks of the form \cref{eq:CNN-feature-extractor}. Before proceeding, a more detailed explanation is in order as to why such an analysis will contribute to a better understanding of the reasons behind the practical success of scattering networks. Evaluating the separation capacity allows identifying the architectural strengths and bottlenecks of a model. Concretely, our analysis will pinpoint which components (i.e., operations in the network) contribute most to the classification capabilities of scattering networks. Furthermore, separation capacity computations enable understanding the efficiency of a model, specifically in the sense of achieving the highest possible separation capacity with the minimal architectural cost. In the case of scattering networks, by architectural cost we mean the network depth $n_\mathrm{d}$ and the frame sizes $\{\lvert \Lambda_n \rvert\}_{n \in \mathbb N}$.
Finally, we note that a high separation capacity is desirable as it implies that the model demonstrates robustness to the underlying data structure. Indeed, by definition, the separation capacity takes into account almost every tuple with entries in the pattern space.  

We begin our analysis with a simple example.
\subsection{Example: A Weyl--Heisenberg frame}\label{subsec:WHmod2}
We analyze the scattering network whose layers are all built from the same frame and the modulus squared nonlinearity $\lvert \cdot \rvert^2$, without subsequent pooling. More formally, in the notation of \cref{sec:feature-extractors}, we have the module sequence $\{(\Psi_{\mathrm{WH}},\lvert \cdot \rvert^2, \mathrm{Id})\}_{n \in \mathbb N}$. 
As noted in \cref{rem:scattering-mallat}, the modulus $\lvert \cdot \rvert$ is the traditional choice for the nonlinearity. Since $\lvert \cdot \rvert$ and $\lvert \cdot \rvert^2$ exhibit similar behavior, namely both show a demodulation and bandwidth doubling effect and result in a real-valued signal with conjugate symmetric spectrum (see, e.g., \cite{wiatowski2017topology}), but $\lvert \cdot \rvert^2$ is easier to analyze, we opt for $\lvert \cdot \rvert^2$ in our first example. Indeed, applying the modulus squared nonlinearity pointwise to a signal simply doubles its spectral support, which, as we shall see below, simplifies our analysis.
The frame $\Psi_{\mathrm{WH}}$, considered in this example, is a so-called \emph{Weyl--Heisenberg frame} \cite{grochenig2001foundations}, formed by the family of functions $\{\chi\} \cup \{g_\lambda\}_{\lambda \in \Lambda}$, which are also called \emph{atoms}, where $ \{g_\lambda\}_{\lambda \in \Lambda}$ are obtained through modulation from the prototype function $\chi \in \mathbb C^{\mathbb Z/M \mathbb Z}$. The details of this construction will be clarified below. Identifying $\mathbb C^{\mathbb Z/M \mathbb Z} \simeq \mathbb C^M$, the functions $\{\chi\} \cup \{g_\lambda\}_{\lambda \in \Lambda}$ can be defined in their vector representations. Namely, set
\begin{align}
    \widehat{\chi}_k &\coloneqq \begin{cases}
        1, &\quad \text{if $0 \leq k \leq m_0$ or $M-m_0 \leq k\leq M-1$,} \\
        0, &\quad \text{otherwise,}
    \end{cases} \label{eq:WH-chi-hat}
\intertext{and}
    (\widehat{g_\lambda})_k &\coloneqq \widehat{g}_{\lambda,k} \coloneqq \widehat{\chi}_{(k-\lambda R) \bmod{M}}, \quad k \in \{0,\dots,M-1\}, \label{eq:WH-modulation}
\end{align}
where $R \coloneqq (2m_0 +1) \in \{1,\dots,\lfloor M/2 \rfloor\}$ is such that $M \equiv 0 \Mod{R}$, and where $\Lambda \coloneqq \{1,\dots,L\}$ with $ L\coloneqq M/R-1$, see \cref{fig:WH-frame}. Then, the following frame condition holds:
\begin{align}\label{eq:Parseval_frame_condition}
    \lVert f \ast \chi \rVert^2+ \sum_{\lambda \in \Lambda } \lVert f \ast g_\lambda \rVert^2 = \lVert f \rVert^2, \quad \text{ for all $f \in \mathbb C^M$.}
\end{align}
Indeed, application of Parseval's identity together with the convolution property of the DFT equivalently yields 
\begin{align*}
    \frac 1M \sum_{k=0}^{M-1}\lvert \widehat{f}_k \rvert^2 \lvert \widehat{\chi}_k \rvert^2+ \frac 1M \sum_{\lambda \in \Lambda } \sum_{k=0}^{M-1}\lvert \widehat{f}_k \rvert^2 \lvert \widehat{g}_{\lambda,k} \rvert^2 = \frac 1M \sum_{k=0}^{M-1}\lvert \widehat{f}_k \rvert^2, \quad \text{ for all $f \in \mathbb C^M$.}
\end{align*}
Upon inspection of \cref{eq:WH-modulation}, it thus follows that \cref{eq:Parseval_frame_condition} is equivalent to the Littlewood--Paley condition
\begin{align*}
    \sum_{\lambda \in \Lambda \cup \{0\}} \lvert \widehat{\chi}_{(k-\lambda R) \bmod{M}} \rvert^2 = 1, \quad k \in \{0, \dots, M-1\},
\end{align*}
which holds by construction.

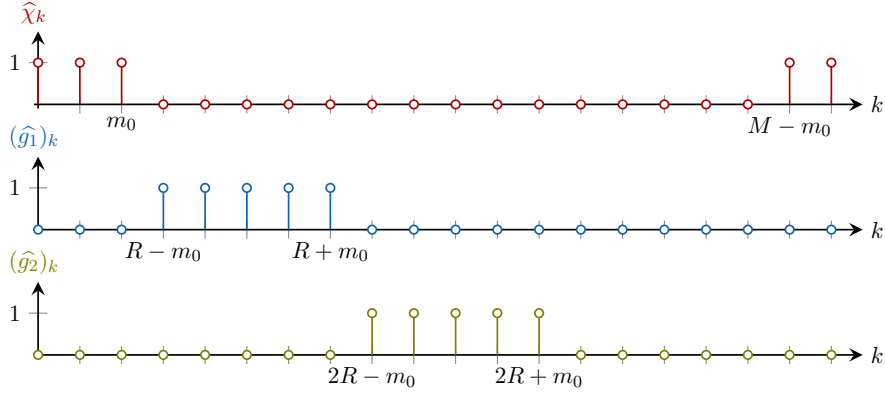
\begin{figure}[t]
    \centering
    \resizebox{0.8\columnwidth}{!}{
    \colorlet{colour_chi}{red!70!black}
\definecolor{colour_g1}{RGB}{20,99,178}
\definecolor{colour_g2}{rgb}{0.5, 0.5, 0.0}
\begin{tikzpicture}[scale=1, every node/.style={scale=0.5}]
    \begin{axis}[
        unit vector ratio*=1 1 1,
        grid=none,
        xmin=-0.1,
        xmax=19.75,
        ymin=-0.1,
        ymax=1.75,
        clip=false,
        axis lines=middle,
        xtick={0,1,2,...,19},
        ytick={0,1},
        xticklabels={,,},
        yticklabels={0,1},
        xlabel style = {at={(axis description cs:1,0)},anchor=west,yshift=2pt},
        ylabel style = {at={(axis description cs:0,1)},anchor=south,color=colour_chi},
        ylabel={$\widehat{\chi}_k$},
        xlabel={$k$}
    ]
        \addplot[thin,ycomb,mark=*,mark options={scale=1, fill=white},mark size=1pt,color=colour_chi] coordinates {%
        (0,1) (1,1) (2,1) (3,0) (4,0) (5,0) (6,0) (7,0) (8,0) (9,0) (10,0) (11,0) (12, 0) (13,0) (14,0) (15,0) (16,0) (17,0) (18,1) (19,1)
      };
      \node[below,yshift=-2pt] at (2,0) {$m_0$};
      \node[below,yshift=-2pt] at (18,0) {$M-m_0$};
    \end{axis}
    \begin{axis}[
        at={(0,-3)},
        unit vector ratio*=1 1 1,
        grid=none,
        xmin=-0.1,
        xmax=19.75,
        ymin=-0.1,
        ymax=1.75,
        clip=false,
        axis lines=middle,
        xtick={0,1,2,...,19},
        ytick={0,1},
        xticklabels={,,},
        yticklabels={0,1},
        xlabel style = {at={(axis description cs:1,0)},anchor=west,yshift=2pt},
        ylabel style = {at={(axis description cs:0,1)},anchor=south,color=colour_g1},
        ylabel={$(\widehat{g_1})_k$},
        xlabel={$k$}
    ]
        \addplot[thin,ycomb,mark=*,mark options={scale=1, fill=white},mark size=1pt,color=colour_g1] coordinates {%
        (0,0) (1,0) (2,0) (3,1) (4,1) (5,1) (6,1) (7,1) (8,0) (9,0) (10,0) (11,0) (12, 0) (13,0) (14,0) (15,0) (16,0) (17,0) (18,0) (19,0)
      };
      \node[below,yshift=-2pt] at (3,0) {$R-m_0$};
      \node[below,yshift=-2pt] at (7,0) {$R+m_0$};
    \end{axis}
    
    \begin{axis}[
        at={(0,-6)},
        unit vector ratio*=1 1 1,
        grid=none,
        xmin=-0.1,
        xmax=19.75,
        ymin=-0.1,
        ymax=1.75,
        clip=false,
        axis lines=middle,
        xtick={0,1,2,...,19},
        ytick={0,1},
        xticklabels={,,},
        yticklabels={0,1},
        xlabel style = {at={(axis description cs:1,0)},anchor=west,yshift=2pt},
        ylabel style = {at={(axis description cs:0,1)},anchor=south,color=colour_g2},
        ylabel={$(\widehat{g_2})_k$},
        xlabel={$k$}
    ]
        \addplot[thin,ycomb,mark=*,mark options={scale=1, fill=white},mark size=1pt,color=colour_g2] coordinates {%
        (0,0) (1,0) (2,0) (3,0) (4,0) (5,0) (6,0) (7,0) (8,1) (9,1) (10,1) (11,1) (12, 1) (13,0) (14,0) (15,0) (16,0) (17,0) (18,0) (19,0)
      };
      \node[below,yshift=-2pt] at (8,0) {$2R-m_0$};
      \node[below,yshift=-2pt] at (12,0) {$2R+m_0$};
    \end{axis}
\end{tikzpicture}}
    \caption{Atoms of the Weyl--Heisenberg frame $\Psi_{\mathrm{WH}}$, i.e., $\{\chi\} \cup \{g_\lambda\}_{\lambda \in \Lambda}$.}
    \label{fig:WH-frame}
\end{figure}

\paragraph{Single-layer network}
We commence with the separation capacity of the single-layer network constructed from $\{(\Psi_{\mathrm{WH}},\lvert \cdot \rvert^2, \mathrm{Id})\}_{n \in \mathbb N}$. Namely, 
consider the single-layer network $\Phi \colon \mathbb C^M \to \mathbb C^{M(L+1)}$, defined according to
\begin{align*}
    f \mapsto \begin{pmatrix}
        f \ast \chi \\
        \lvert f \ast g_1 \rvert^2 \ast \chi \\
        \vdots \\
        \lvert f \ast g_L \rvert^2 \ast \chi
    \end{pmatrix}.
\end{align*}
To determine the separation capacity of $\Phi$, we apply our method from \cref{thm:sep-cap}. With $\mathbb C \simeq \mathbb R^2$, the network $\Phi$ can be viewed as a map 
\begin{align*}
\widetilde{\Phi}\colon \mathbb R^{2M} \to \mathbb R^{2M(L+1)}, \quad \begin{matrix}\begin{pmatrix}
       f'\\f''
    \end{pmatrix} \mapsto \begin{pmatrix}
        \Re(\Phi(f'+i f'')) \\
        \Im(\Phi(f'+i f''))
    \end{pmatrix}\end{matrix} 
\end{align*}
We claim that $\widetilde{\Phi}$ is real-analytic. 
Indeed, first note that by definition of the nonlinearity $\lvert \cdot \rvert^2$, one can write $\Phi(f) = \widetilde{\widetilde{\Phi}}(f,\overline{f})$, $f \in \mathbb C^M$, where $\widetilde{\widetilde{\Phi}}\colon\mathbb C^M \times \mathbb C^M \to \mathbb C^{M(L+1)}$ is polynomial in both arguments. This yields the commutative\footnote{That is, all map compositions with the same start and end lead to the same result.} diagram
\[
\begin{tikzcd}[row sep=1cm, column sep=3.5cm]
    \mathbb C^M \times \mathbb C^M \arrow[d,"\widetilde{\widetilde{\Phi}}"] &
    \mathbb R^{2M} \arrow[l,swap,
     "(f'+i f''\comma f'-i f'')\mapstoleft\begin{pmatrix}
         f'\\f''
     \end{pmatrix}"
    ] \arrow[d,
    "\widetilde{\Phi}
    "
    ] &
    \mathbb C^M \arrow[l,swap,"\begin{pmatrix}
        \Re(f)\\ \Im(f)
    \end{pmatrix}\mapstoleft f"]\arrow[d,"\Phi"] \\
    \mathbb C^{M(L+1)} \arrow[r,swap,"\varphi \mapsto \begin{pmatrix}\Re(\varphi) \\ \Im(\varphi)\end{pmatrix}"] & 
    \mathbb R^{2M(L+1)} \arrow[r,swap,"\begin{pmatrix}
        \varphi'\\\varphi''
    \end{pmatrix}\mapsto\varphi'+i\varphi''"] & 
    \mathbb C^{M(L+1)}.
\end{tikzcd}
\]
Furthermore, we note that $\Re(f) = \frac 12 (f+\overline{f})$ and $\Im(f) = \frac{1}{2i} (f-\overline{f})$ are both polynomial in $f$ and $\Bar{f}$. From the above commutative diagram it thus follows that $\widetilde{\Phi}$ is a composition of polynomials and hence, in particular, is real-analytic.
\begin{definition}
    We call $\Phi\colon \mathbb C^M \to \mathbb C^{M'}$ \emph{real-analytic} if 
    \begin{align*}
\widetilde{\Phi}\colon \mathbb R^{2M} \to \mathbb R^{2M'}, \quad \begin{matrix}\begin{pmatrix}
       f'\\f''
    \end{pmatrix} \mapsto \begin{pmatrix}
        \Re(\Phi(f'+i f'')) \\
        \Im(\Phi(f'+i f''))
    \end{pmatrix}\end{matrix} 
\end{align*}
    is real-analytic.
\end{definition}

Application of \cref{cor:real-analytic}, together with \cref{eq:real-val-phi}, now shows that the separation capacity of $\Phi$ is given by the following expression, which is, however, not obvious to evaluate:
\begin{align}\label{eq:sc-span-real}
    \sepcap{\Phi} = 2 \cdot \dimR{\spanR{\left\{\begin{pmatrix}\Re(\Phi(f)) \\ \Im(\Phi(f)) \end{pmatrix} \colon f \in \mathbb C^M\right\}}}. 
\end{align}
As we shall see, it is more convenient to study the linear span of the image of $\Phi$ over the field $\mathbb C$, as this approach allows us to work entirely within complex Euclidean spaces. This not only facilitates the use of the DFT, for example, to leverage the convolution property of the DFT, but also enables the application of tools from complex analysis.
Let us thus establish a relation between the latter quantity and the right-hand side (RHS) of \cref{eq:sc-span-real}. To this end, note that
\begin{align}\label{eq:spanC-spanR}
    \begin{split}
    &\dimC{\spanC{\left\{\Phi(f)\colon f \in \mathbb C^M\right\}}} \\&= \frac 12 \dimR{\spanR{\left\{\begin{pmatrix}\Re(\Phi(f)) \\ \Im(\Phi(f)) \end{pmatrix} \colon f \in \mathbb C^M\right\}} + \spanR{\left\{T\begin{pmatrix}\Re(\Phi(f)) \\ \Im(\Phi(f)) \end{pmatrix} \colon f \in \mathbb C^M\right\}}},
    \end{split}
\end{align}
where $T \coloneqq \begin{pmatrix}
    0 & - I_{M'} \\ I_{M'} & 0
\end{pmatrix} \in \mathbb R^{2M'\times 2M'}$. This is an immediate consequence of the following simple observation: for every $N\in \mathbb N$, $\{\alpha_k\}_{k=1}^N \subset \mathbb C$, and $\{f_k\}_{k=1}^N \subset \mathbb C^M$, we have
\begin{align*}
    \begin{pmatrix}
        \Re(\sum_{k=1}^N \alpha_k\Phi(f_k)) \\
        \Im(\sum_{k=1}^N \alpha_k\Phi(f_k))
    \end{pmatrix} = \sum_{k=1}^N \Re(\alpha_k)\begin{pmatrix}\Re(\Phi(f_k)) \\ \Im(\Phi(f_k)) \end{pmatrix} + \sum_{k=1}^N \Im(\alpha_k) \begin{pmatrix}-\Im(\Phi(f_k))\\ \Re(\Phi(f_k)) \end{pmatrix}.   
\end{align*}
As $T$ has full rank, applying the dimension formula for the sum of linear subspaces to the RHS of \cref{eq:spanC-spanR} and substituting this into \cref{eq:sc-span-real} yields 
\begin{align}
    \sepcap{\Phi} &= 2\cdot\dimR{\spanR{\left\{\begin{pmatrix}\Re(\Phi(f)) \\ \Im(\Phi(f)) \end{pmatrix} \colon f \in \mathbb C^M\right\}}} \nonumber\\
    &= 2\cdot\dimC{\spanC{\left\{\Phi(f)\colon f \in \mathbb C^M\right\}}}\nonumber\\&\quad+2\cdot \dimR{\spanR{\left\{\begin{pmatrix}\Re(\Phi(f)) \\ \Im(\Phi(f)) \end{pmatrix} \colon f \in \mathbb C^M\right\}} \cap \spanR{\left\{T\begin{pmatrix}\Re(\Phi(f)) \\ \Im(\Phi(f)) \end{pmatrix} \colon f \in \mathbb C^M\right\}}}. \label{eq:SC-spanC}
\end{align}
Note that the module sequence $\{(\Psi_{\mathrm{WH}}, \lvert \cdot \rvert^2, \mathrm{Id})\}_{n \in \mathbb N}$ induces scattering networks whose outputs are real-valued at every layer except the $0$th. This follows because $\chi$ is real-valued, a consequence of the conjugate symmetry of $\widehat{\chi}$ (see \cref{eq:WH-chi-hat}). In other words, if we write $\Phi(f) = (\Phi^0(f), \Phi^{\setminus0}(f))^{\mathsf{T}}$, where $\Phi^0(f) = f \ast \chi$, $f \in \mathbb C^M$, then $\Im(\Phi^{\setminus 0}(f)) = 0$, for all $f \in \mathbb C^M$. As the atoms of the frame $\Psi_{\mathrm{WH}}$ are spectrally disjoint, we have
\begin{align}\label{eq:SC-decomp}
    \sepcap{\Phi} = \sepcap{\Phi^0} + \sepcap{\Phi^{\setminus 0}}.
\end{align}
The map $\Phi^0 \colon \mathbb C^M \to \mathbb C^M$ is linear, and hence its image is a linear subspace of $\mathbb C^M$. Consequently, we have
\begin{align}\label{eq:SCPhi0}
    \sepcap{\Phi^0} &= 2\cdot\dimR{\spanR{\left\{\begin{pmatrix}\Re(\Phi^0(f)) \\ \Im(\Phi^0(f)) \end{pmatrix} \colon f \in \mathbb C^M\right\}}} 
    = 4\cdot\dimC{\spanC{\Phi^0(\mathbb C^M)}}.
\end{align}
We next observe that for real-valued maps, the second term in \cref{eq:SC-spanC} vanishes. In particular, as $\Phi^{\setminus 0}$ is real-valued, it follows that
\begin{align}\label{eq:SCPhiMinus0}
    \sepcap{\Phi^{\setminus 0}} &= 2\cdot\dimC{\spanC{\Phi^{\setminus0}(\mathbb C^M)}}.
\end{align}
Substituting \cref{eq:SCPhi0,eq:SCPhiMinus0} into \cref{eq:SC-decomp} leads to
\begin{align}\label{eq:SC-WHmod2}
    \sepcap{\Phi} = 4\cdot\dimC{\spanC{\Phi^0(\mathbb C^M)}} + 2\cdot\dimC{\spanC{\Phi^{\setminus 0}(\mathbb C^M)}}.
\end{align}

To analyze $\Phi^{\setminus 0}$, let us first consider one node in the first layer of the network, i.e., the map $f \mapsto \lvert f \ast g_\lambda \rvert^2$, for some $\lambda \in \Lambda$. We have the following result.
\begin{lemma}\label{L:dim_conv_mod2}
Consider the atoms $\{g_\lambda\}_{\lambda \in \Lambda}$ of the frame $\Psi_{\mathrm{WH}}$. For every $\lambda \in \Lambda$, it holds that
    \begin{align}\label{eq:dim_conv_mod2}
        \dimC{\spanC{\left\{ \lvert f \ast g_\lambda \lvert^{2} \colon f \in \mathbb C^M\right\}}} = 2R-1. 
    \end{align}
\end{lemma}
\begin{proof}
    See \cref{app:dim_conv_mod2}. 
\end{proof}
\begin{figure}[t]
    \centering
    \resizebox{0.8\columnwidth}{!}{
        \colorlet{colour_chi}{red!70!black}
\definecolor{colour_g1}{RGB}{20,99,178}
\definecolor{colour_g2}{rgb}{0.5, 0.5, 0.0}
\begin{tikzpicture}[scale=6, every node/.style={scale=0.5}]
    \begin{axis}[
        unit vector ratio*=1 1 1,
        grid=none,
        xmin=-0.1,
        xmax=19.75,
        ymin=-0.1,
        ymax=2.5,
        clip=false,
        axis lines=middle,
        xtick={0,1,2,...,19},
        ytick={0,1,2},
        xticklabels={,,},
        yticklabels={0,1},
        xlabel style = {at={(axis description cs:1,0)},anchor=west,yshift=2pt},
        ylabel style = {at={(axis description cs:0,1)},anchor=south,color=colour_g1},
        ylabel={$\lvert (\widehat{f\ast g_1})_k \rvert$},
        xlabel={$k$}
    ]
        \addplot[thin,ycomb,mark=*,mark options={scale=1, fill=white},mark size=1pt,color=colour_g1] coordinates {
        (0,0) (1,0) (2,0) (3,0) (4,2.4) (5,1.7) (6,1.8) (7,0.8) (8,1.4) (9,0) (10,0) (11,0) (12, 0) (13,0) (14,0) (15,0) (16,0) (17,0) (18,0) (19,0)
      };
      \draw [decorate, decoration={brace,amplitude=5pt,mirror,raise=2pt}] (4,0) -- (8,0) node [midway,below,yshift=-14pt]{$R$};
    \end{axis}
    \begin{axis}[
        unit vector ratio*=1 1 1,
        at={(0,-4)},
        grid=none,
        xmin=-0.1,
        xmax=19.75,
        ymin=-0.1,
        ymax=2.5,
        clip=false,
        axis lines=middle,
        xtick={0,1,2,...,19},
        ytick={0,1,2},
        xticklabels={,,},
        yticklabels={0,1},
        xlabel style = {at={(axis description cs:1,0)},anchor=west,yshift=2pt},
        ylabel style = {at={(axis description cs:0,1)},anchor=south,color=colour_g1},
        ylabel={$\lvert (\widehat{\lvert f\ast g_1 \rvert^2})_k \rvert$},
        xlabel={$k$}
    ]
        \addplot[thin,ycomb,mark=*,mark options={scale=1, fill=white},mark size=1pt,color=colour_g1] coordinates {
        (0,1.83) (1,0.65) (2,0.62) (3,0.35) (4,0.4) (5,0) (6,0) (7,0) (8,0) (9,0) (10,0) (11,0) (12, 0) (13,0) (14,0) (15,0) (16,0.4) (17,0.35) (18,0.62) (19,0.65)
      };
      \node[below,yshift=-2pt] at (4,0) {$R-1$};
      \node[below,yshift=-2pt] at (16,0) {$M-R+1$};
    \end{axis}
\end{tikzpicture}}
    \caption{Computations in a node in the first layer.}
    \label{fig:WH-mod2-node}
\end{figure}
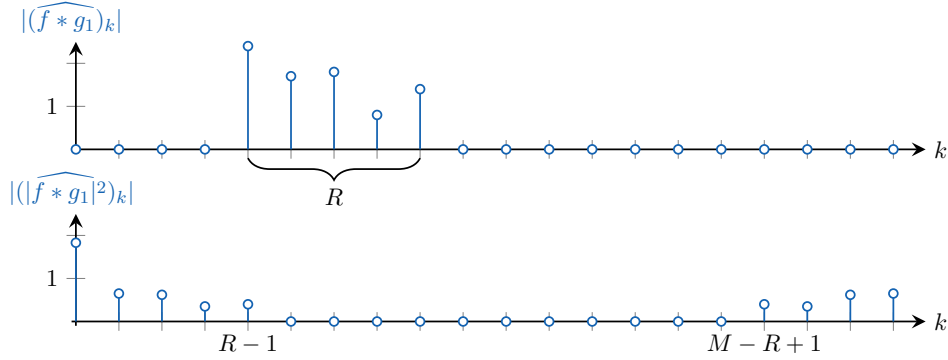
To understand the underlying mechanisms behind this result, observe that convolving the input signal $f$ with the filter $g_\lambda$ corresponds to an orthogonal projection onto a linear subspace of dimension $\lvert \supp{\widehat{g_\lambda}}\rvert = R$. This is due to the fact that the DFT matrix $F_M$ diagonalizes the circulant matrix induced by $g_\lambda$, with eigenvalues given by the entries of the vector $\widehat{g_\lambda}$ taking values in $\{0,1\}$. Consequently, the set $\{f \ast g_\lambda \colon f \in \mathbb C^M\}$ spans an $R$-dimensional $\mathbb C$-vector space. Now, by applying the modulus squared nonlinearity pointwise to the filtered signal $f \ast g_\lambda$, a bandwidth doubling effect occurs, as previously mentioned and illustrated in \cref{fig:WH-mod2-node}. More precisely, the signal $\lvert f \ast g_\lambda \rvert^2$ has spectral support of at most $2R-1$. We conclude that pointwise application of the nonlinearity $\lvert \cdot \rvert^2$ to the filtered signals in $\{f \ast g_\lambda \colon f \in \mathbb C^M\}$ yields a set which lives in a higher-dimensional space, provided $R>1$. The dimension of this space remains unchanged if $R=1$. As we will explore further later, this observation is key to understanding the separation capacity of scattering networks.      

Based on \cref{L:dim_conv_mod2}, we can now determine the separation capacity of the single-layer network built from the module sequence $\{(\Psi_{\mathrm{WH}},\lvert \cdot \rvert^2,\mathrm{Id})\}_{n \in \mathbb N}$. Indeed, thanks to $\{\chi\}\cup \{g_\lambda\}_{\lambda \in \Lambda}$ being spectrally disjoint, all nodes in this network can be analyzed independently according to \cref{L:dim_conv_mod2}, and the results can then be combined. In doing so, we obtain the following expression for the separation capacity. 
\begin{proposition}[Separation capacity of single-layer network]\label{prop:singlelayernet}
Consider the module sequence $\{(\Psi_{\mathrm{WH}},\lvert \cdot \rvert^2,\mathrm{Id})\}_{n \in \mathbb N}$. For the single-layer network
\begin{align*}
    \Phi \colon \mathbb C^M \to \mathbb C^{M(L+1)}, \quad f \mapsto \begin{pmatrix}
        f \ast \chi \\
        \lvert f \ast g_1 \rvert^2 \ast \chi \\
        \vdots \\
        \lvert f \ast g_L \rvert^2 \ast \chi
    \end{pmatrix},  
\end{align*}
we have
\begin{align*}
    \sepcap{\Phi} = 2(M+R).
\end{align*}
\end{proposition}
\begin{proof}
    See \cref{app:singlelayernet}. 
\end{proof}
Recalling \cref{eq:SC-WHmod2}, we observe that the image of the single-layer network $\Phi$ spans only an $M$-dimensional $\mathbb C$-vector space, while the dimension of the codomain of $\Phi$ is $M(L+1)$. Intuitively, this means that $\Phi$ only fills out a fraction of its codomain. In fact, compared to the input space $\mathbb C^M$, we have no gain in the sense of the image of $\Phi$ spanning a higher-dimensional space.
As a high separation capacity is associated with filling out the codomain well (see \cref{eq:SC-WHmod2}), we conclude that this network achieves a rather low separation capacity and is suboptimal among transformations of the form $\mathbb C^M \to \mathbb C^{M(L+1)}$. 
Thus, the single-layer feature extractor $\Phi$ built from the module sequence $\{(\Psi_{\mathrm{WH}},\lvert \cdot \rvert^2,\mathrm{Id})\}_{n \in \mathbb N}$ is not a favorable choice in the sense of achieving high separation capacity.

\paragraph{Multi-layer network}
We now proceed to analyze multi-layer networks $\Phi$ of depth $n_\mathrm{d}\geq 2$ that are built from the module sequence $\{(\Psi_{\mathrm{WH}},\lvert \cdot \rvert^2,\mathrm{Id})\}_{n \in \mathbb N}$. As in the single-layer case, the assumption of \cref{cor:real-analytic}, i.e., real analyticity of $\Phi$, is satisfied thanks to the properties of the nonlinearity $\lvert \cdot \rvert^2$. Thus, the derivation of \cref{eq:SC-WHmod2} remains valid, and consequently, the separation capacity of such multi-layer networks $\Phi$ is given by 
\begin{align}\label{eq:sepcap-multi-WHmod2}
    \sepcap{\Phi} = 4\cdot\dimC{\spanC{\Phi^0(\mathbb C^M)}} + 2\cdot\dimC{\spanC{\Phi^{\setminus 0}(\mathbb C^M)}},
\end{align}
where we again use the decomposition $\Phi(f) = (\Phi^0(f),\Phi^{\setminus0}(f))^{\mathsf{T}}$ with $\Phi^0(f)=f \ast \chi$, $f \in \mathbb C^M$.
Evaluating this expression for multi-layer networks is more challenging. Specifically, in contrast to the derivation of the result in \cref{prop:singlelayernet}, nodes in higher-order layers whose paths coincide in the first entry cannot be studied independently, and the results thereof cannot be simply combined. The next lemma, which arises from the symmetry of the atoms of our Weyl--Heisenberg frame $\Psi_\mathrm{WH}$, explains why this approach is not applicable to multi-layer networks. 
\begin{lemma}\label{L:symmetry}
    Consider the frame $\Psi_{\mathrm{WH}}$, and let $f \in \mathbb C^M$ be real-valued. Then, for all $\lambda, \lambda' \in \Lambda$ such that\,\footnote{Recall that for $A \subseteq \mathbb Z /M\mathbb Z$, the reflection of $A$ is defined to be $A^r \coloneqq \{-a\colon a \in A\}$.} $\mathrm{supp}(\widehat{g_{\lambda}}) = (\mathrm{supp}(\widehat{g_{\lambda'}}))^r$, we have $\lvert f \ast g_\lambda \rvert ^2 = \lvert f \ast g_{\lambda'} \rvert^2$.
\end{lemma}
\begin{proof}
    See \cref{app:symmetry-lemma}. 
\end{proof}
Upon noting that the inputs to nodes in higher-order layers are real-valued because of the nonlinearity $\lvert \cdot \rvert^2$, one can deduce from \cref{L:symmetry} that there exist pairs of nodes in higher-order layers whose outputs coincide. Consequently, simply adding up the dimensions of the vector spaces in which the outputs of the nodes live, as done in the single-layer case (\cref{prop:singlelayernet}), is not possible.
Nevertheless, the choice of our module sequence $\{(\Psi_{\mathrm{WH}},\lvert \cdot \rvert^2,\mathrm{Id})\}_{n \in \mathbb N}$ induces the following properties of the feature maps, which simplify the analysis of the multi-layer case. 
\begin{lemma}\label{L:multilayer_net}
     For the module sequence $\{(\Psi_{\mathrm{WH}},\lvert \cdot \rvert^2,\mathrm{Id})\}_{n \in \mathbb N}$, we have 
     \begin{enumerate}[label=(\roman*)]
         \item\label{it:L-multilayer-i} $U[(\lambda_1,\lambda_2)] = 0$, for every $(\lambda_1,\lambda_2) \in \Lambda \times (\Lambda\setminus\{1,L\})$,
         \item\label{it:L-multilayer-ii} $U[(\lambda_1,1)] = U[(\lambda_1,L)]$, for every $\lambda_1 \in \Lambda$, and  
         \item\label{it:L-multilayer-iii} $U[q] = 0 $, for every $q \in \Lambda^n$ with $n \geq 3$.
     \end{enumerate}
\end{lemma}
\begin{proof}
    See \cref{app:multilayer_net-lemma}. 
\end{proof}
Thus, every multi-layer network $\Phi$ built from $\{(\Psi_{\mathrm{WH}},\lvert \cdot \rvert^2,\mathrm{Id})\}_{n \in \mathbb N}$ reduces to a two-layer network, depicted in \cref{fig:WH-mod2-tree}, in the sense that outputs from higher-order layers are trivial. Moreover, for the separation capacity computation of $\Phi$, \cref{L:multilayer_net} shows that it suffices to study the nodes along the paths $\{(\lambda,1)\}_{\lambda \in \Lambda}$. \cref{fig:WH-mod2-layer2_1,fig:WH-mod2-layer2_L}, which illustrate the computation of $U[(1,1)]$ and $U[(1,L)]$, respectively, provide intuition and informal justification for both of these conclusions. 
As the atoms $\{\chi\}\cup \{g_\lambda\}_{\lambda \in \Lambda}$ are spectrally disjoint, the only groups of nodes that cannot be analyzed independently are the ones in the first and second layer along each path $(\lambda,1)$, $\lambda \in \Lambda$. The next lemma presents the resulting expression.
\begin{lemma}\label{L:dim_path_1l}
Consider the module sequence $\{(\Psi_{\mathrm{WH}},\lvert \cdot \rvert^2,\mathrm{Id})\}_{n \in \mathbb N}$. For $\lambda \in \Lambda$, it holds that
    \begin{align*}
        \dimC{\spanC{\left\{\begin{pmatrix}(U[\lambda]f) \ast \chi \\ (U[\lambda,1]f)\ast \chi\end{pmatrix} \colon f \in \mathbb C^M\right\}}} = R + (R-2)_+.
    \end{align*}
\end{lemma}
\begin{proof}
See \cref{app:dim_path_1l-lemma}.
\end{proof}

\begin{figure}[t]
    \centering
    \resizebox{\columnwidth}{!}{
\definecolor{colour1}{rgb}{0.0, 0.0, 0.0}
\definecolor{coloursup}{RGB}{190, 190, 190}
\definecolor{coloursup2}{RGB}{211, 211, 211}
\definecolor{blue2}{RGB}{20,99,178}
\definecolor{green2}{RGB}{0,95,1}
\colorlet{red2}{red!70!black}
\begin{tikzpicture}[
	level distance = 15mm,
	level 1/.style={sibling distance=30mm},
	level 2/.style={sibling distance=15mm},
        level 3/.style={sibling distance=5mm},
	grow'=up]
	\node[color=colour1] (a0) {$U[e]f=f$}
	child[sibling distance=50mm,level distance=30mm,draw=black,thin]{
		node[opacity=1] (b1) {$U[1]f$}
		child[sibling distance=20mm,draw=black,thin] {
                node[opacity=1] (c1) {$U[1,1]f$}
            } 
            child[sibling distance=20mm,draw=coloursup,thin] {
                node[opacity=1,color=coloursup] (c2) {$U[1,L]f$}
            }
	}
        child[sibling distance=70mm,level distance=30mm,draw=black,thin]{
		node[opacity=1] (b2) {$U[2]f$}
		child[sibling distance=20mm,draw=black,thin] {
                node[opacity=1] (c21) {$U[2,1]f$}
            } 
            child[sibling distance=20mm,draw=coloursup,thin] {
                node[opacity=1,color=coloursup] (c22) {$U[2,L]f$}
            }
	}
        child[sibling distance=90mm,level distance=30mm]{
		node[opacity=1] (b4) {$U[L]f$}
		child[sibling distance=20mm]{
                node[opacity=1] (c3) {$U[L,1]f$}
            }
            child[sibling distance=20mm,draw=coloursup,thin] {
                node[opacity=1,color=coloursup] (c4) {$U[L,L]f$}
            }
	};
	\draw[->,dashed] (-0.2,-0.2) -- ++(225:0.875cm) node[below] {$f\ast \chi$};
	\draw[->,dashed] (b1) -- ++(225:2cm) node[below] {$\left(U[1]f\right) \ast \chi$};
        \draw[->,dashed] (c1) -- ++(225:2cm) node[below] {$\left(U[1,1]f\right) \ast \chi$};
        \draw[->,dashed,color=coloursup] (c2) -- ++(315:1cm) node[below] {$\cdot \ast \chi$};
	\draw[->,dashed] (b2) -- ++(225:1cm) node[below] {$\cdot \ast \chi$};
        \draw[->,dashed] (c21) -- ++(225:1cm) node[below] {$\cdot \ast \chi$};
        \draw[->,dashed,color=coloursup] (c22) -- ++(315:1cm) node[below] {$\cdot \ast \chi$};
        \draw[->,dashed] (b4) -- ++(315:2cm) node[below] {$\left(U[L]f\right)\ast \chi$};
        \draw[->,dashed] (c3) -- ++(225:1cm) node[below] {$\cdot\ast \chi$};
        \draw[->,dashed,color=coloursup] (c4) -- ++(315:2cm) node[below] {$\left(U[L,L]f\right)\ast \chi$};
        \node[] at (4,3) {\Huge\textbf{\dots}};
\end{tikzpicture}}
    \caption{Tree structure of every multi-layer network built from $\{(\Psi_{\mathrm{WH}},\lvert \cdot \rvert^2,\mathrm{Id})\}_{n \in \mathbb N}$, comprising only the nontrivial nodes. The gray part is superfluous as $U[\lambda,1]=U[\lambda,L]$, for every $\lambda \in \Lambda$.}
    \label{fig:WH-mod2-tree}
\end{figure}
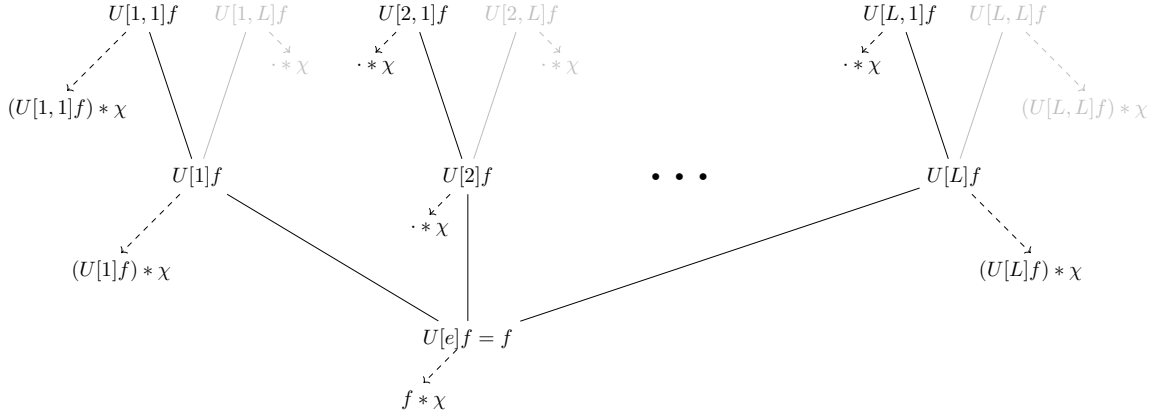
\begin{figure}[t]
    \centering
    \resizebox{0.8\columnwidth}{!}{
    \colorlet{colour_chi}{red!70!black}
\definecolor{colour_g1}{RGB}{20,99,178}
\definecolor{colour_g2}{rgb}{0.5, 0.5, 0.0}
\begin{tikzpicture}[scale=1, every node/.style={scale=0.5}]
    \begin{axis}[
        unit vector ratio*=1 1 1,
        at={(0,0)},
        grid=none,
        xmin=-0.1,
        xmax=19.75,
        ymin=-0.1,
        ymax=2.5,
        clip=false,
        axis lines=middle,
        xtick={0,1,2,...,19},
        ytick={0,1,2},
        xticklabels={,,},
        yticklabels={0,1},
        xlabel style = {at={(axis description cs:1,0)},anchor=west,yshift=2pt},
        ylabel style = {at={(axis description cs:0.05,1)},anchor=south,color=colour_g1},
        ylabel={$\lvert (\widehat{\lvert f\ast g_1 \rvert^2})_k \rvert${{\color{black},} {\color{colour_chi} $(\widehat{g_1})_k$}}},
        xlabel={$k$}
    ]
        \addplot[thin,ycomb,mark=*,mark options={scale=1, fill=white},mark size=1pt,color=colour_g1] coordinates {%
        (0,1.83) (1,0.65) (2,0.62) (3,0.35) (4,0.4) (5,0) (6,0) (7,0) (8,0) (9,0) (10,0) (11,0) (12, 0) (13,0) (14,0) (15,0) (16,0.4) (17,0.35) (18,0.62) (19,0.65)
      };
      \addplot[thin,ycomb,mark=*,mark options={scale=1, fill=white},mark size=1pt,color=colour_chi] coordinates {%
        (0,0) (1,0) (2,0) (3,1) (4,1) (5,1) (6,1) (7,1) (8,0) (9,0) (10,0) (11,0) (12, 0) (13,0) (14,0) (15,0) (16,0) (17,0) (18,0) (19,0)
      };
    \end{axis}
    \begin{axis}[
        unit vector ratio*=1 1 1,
        at={(0,-4)},
        grid=none,
        xmin=-0.1,
        xmax=19.75,
        ymin=-0.1,
        ymax=2.5,
        clip=false,
        axis lines=middle,
        xtick={0,1,2,...,19},
        ytick={0,1,2},
        xticklabels={,,},
        yticklabels={0,$10^{-2}$},
        xlabel style = {at={(axis description cs:1,0)},anchor=west,yshift=2pt},
        ylabel style = {at={(axis description cs:0,1)},anchor=south,color=colour_g1},
        ylabel={$\lvert (\widehat{U[1,1]f})_k \rvert$},
        xlabel={$k$}
    ]
        \addplot[thin,ycomb,mark=*,mark options={scale=1, fill=white},mark size=1pt,color=colour_g1] coordinates {%
        (0,1.4) (1,0.7) (2,0) (3,0) (4,0) (5,0) (6,0) (7,0) (8,0) (9,0) (10,0) (11,0) (12, 0) (13,0) (14,0) (15,0) (16,0) (17,0) (18,0) (19,0.7)
      };
    \end{axis}
\end{tikzpicture}}
    \caption{Computation of the feature maps associated with the path $(1,1)$.}
    \label{fig:WH-mod2-layer2_1}
\end{figure}
\begin{figure}[t]
    \centering
    \resizebox{0.8\columnwidth}{!}{
        \colorlet{colour_chi}{red!70!black}
\definecolor{colour_g1}{RGB}{20,99,178}
\definecolor{colour_g2}{rgb}{0.5, 0.5, 0.0}
\begin{tikzpicture}[scale=1, every node/.style={scale=0.5}]
    \begin{axis}[
        unit vector ratio*=1 1 1,
        at={(0,0)},
        grid=none,
        xmin=-0.1,
        xmax=19.75,
        ymin=-0.1,
        ymax=2.5,
        clip=false,
        axis lines=middle,
        xtick={0,1,2,...,19},
        ytick={0,1,2},
        xticklabels={,,},
        yticklabels={0,1},
        xlabel style = {at={(axis description cs:1,0)},anchor=west,yshift=2pt},
        ylabel style = {at={(axis description cs:0.05,1)},anchor=south,color=colour_g1},
        ylabel={$\lvert (\widehat{\lvert f\ast g_1 \rvert^2})_k \rvert${{\color{black},} {\color{colour_g2} $(\widehat{g_L})_k$}}},
        xlabel={$k$}
    ]
        \addplot[thin,ycomb,mark=*,mark options={scale=1, fill=white},mark size=1pt,color=colour_g1] coordinates {%
        (0,1.83) (1,0.65) (2,0.62) (3,0.35) (4,0.4) (5,0) (6,0) (7,0) (8,0) (9,0) (10,0) (11,0) (12, 0) (13,0) (14,0) (15,0) (16,0.4) (17,0.35) (18,0.62) (19,0.65)
      };
      \addplot[thin,ycomb,mark=*,mark options={scale=1, fill=white},mark size=1pt,color=colour_g2] coordinates {%
        (0,0) (1,0) (2,0) (3,0) (4,0) (5,0) (6,0) (7,0) (8,0) (9,0) (10,0) (11,0) (12, 0) (13,1) (14,1) (15,1) (16,1) (17,1) (18,0) (19,0)
      };
    \end{axis}
    \begin{axis}[
        unit vector ratio*=1 1 1,
        at={(0,-4)},
        grid=none,
        xmin=-0.1,
        xmax=19.75,
        ymin=-0.1,
        ymax=2.5,
        clip=false,
        axis lines=middle,
        xtick={0,1,2,...,19},
        ytick={0,1,2},
        xticklabels={,,},
        yticklabels={0,$10^{-2}$},
        xlabel style = {at={(axis description cs:1,0)},anchor=west,yshift=2pt},
        ylabel style = {at={(axis description cs:0,1)},anchor=south,color=colour_g1},
        ylabel={$\lvert (\widehat{U[1,L]f})_k \rvert$},
        xlabel={$k$}
    ]
        \addplot[thin,ycomb,mark=*,mark options={scale=1, fill=white},mark size=1pt,color=colour_g1] coordinates {%
        (0,1.4) (1,0.7) (2,0) (3,0) (4,0) (5,0) (6,0) (7,0) (8,0) (9,0) (10,0) (11,0) (12, 0) (13,0) (14,0) (15,0) (16,0) (17,0) (18,0) (19,0.7)
      };
    \end{axis}
\end{tikzpicture}}
    \caption{Computation of the feature maps associated with the path $(1,L)$.}
    \label{fig:WH-mod2-layer2_L}
\end{figure}

Finally, this result allows us to derive the separation capacity of multi-layer networks built from $\{(\Psi_{\mathrm{WH}},\lvert \cdot \rvert^2,\mathrm{Id})\}_{n \in \mathbb N}$.
\begin{theorem}\label{thm:multi-layer_mod2}
For the multi-layer network $\Phi$ of depth $n_\mathrm{d} \geq 2$ constructed from the module sequence $\{(\Psi_{\mathrm{WH}},\lvert \cdot \rvert^2,\mathrm{Id})\}_{n \in \mathbb N}$, we have
    \begin{align*}
        \sepcap{\Phi} = 2 \left(M + R + L(R-2)_+\right).
    \end{align*}
\end{theorem}
\begin{proof}
    Thanks to \cref{L:multilayer_net} and using that the support sets of $\{\widehat{\chi}\}\cup \{\widehat{g_\lambda}\}_{\lambda \in \Lambda}$ are disjoint, \cref{eq:sepcap-multi-WHmod2} reads 
    \allowdisplaybreaks
    \begin{align*}
        \sepcap{\Phi} &= 4\cdot\dimC{\spanC{\Phi^0(\mathbb C^M)}} + 2\cdot\dimC{\spanC{\Phi^{\setminus 0}(\mathbb C^M)}} \\
        &= 4\cdot\dimC{\spanC{\left\{f \ast \chi \colon f \in \mathbb C^M\right\}}} \\
        &\quad+2 \sum_{\lambda \in \Lambda}\dimC{\spanC{\left\{\begin{pmatrix}(U[\lambda]f) \ast \chi \\ (U[\lambda,1]f)\ast \chi\end{pmatrix} \colon f \in \mathbb C^M\right\}}} \\
        &= 4R+ 2\sum_{\lambda \in \Lambda}\left(R+(R-2)_+\right) \\
        &= 2(2+L)R+ 2L(R-2)_+ \\
        &= 2(M + R+ L(R-2)_+), 
    \end{align*}
    where the second equality is by \cref{L:dim_path_1l}.
\end{proof}
We have thus established a precise expression for the separation capacity of an arbitrary scattering network of depth $n_\mathrm{d} \in \mathbb N$ that is constructed from $\{(\Psi_{\mathrm{WH}},\lvert \cdot \rvert^2,\mathrm{Id})\}$ in terms of the dimension of the input space, $M$, and the cardinality of the support sets of the atoms, $R$. (Recall $L=\frac MR -1$.) 
Note that for $R\leq 2$ (i.e., $R=1$, as $R=2m_0+1$ must be odd), the single-layer case reappears. 
The key insight from \cref{thm:multi-layer_mod2} is that the scattering network $\Phi$ of depth $n_\mathrm{d}$ only fills out a very small portion of its codomain, especially, if $n_\mathrm{d}$ is large. Indeed, $\Phi$ takes the form $\mathbb C^M \to \mathbb C^{M(1+L+L^2+\cdots +L^{n_\mathrm{d}})}$, so that the codomain dimension over $\mathbb C$, i.e., $M(1+L+L^2+\cdots +L^{n_\mathrm{d}})$, is significantly smaller than the dimension of the $\mathbb C$-vector space spanned by the image of $\Phi$.
To conclude, the module sequence $\{(\Psi_{\mathrm{WH}},\lvert \cdot \rvert^2,\mathrm{Id})\}$ does not induce a feature extractor of high separation capacity for all network depths $n_\mathrm{d} \in \mathbb N$. 

\subsection{General case}
The example we just discussed naturally raises the questions of which module sequences yield a high separation capacity and what the driving and limiting factors are for achieving it. To address these questions, we will now consider a general single-layer scattering network, constructed from the module sequence $\{(\Psi, \rho, P)\}_{n \in \mathbb{N}}$, in this subsection. Here, $\Psi$ is an arbitrary frame formed by the family of functions $\{\chi\} \cup \{g_\lambda\}_{\lambda \in \Lambda}$, $\rho \colon \mathbb C \to \mathbb C$ is a nonlinearity, and $P \colon \mathbb C^M \to \mathbb C^M$ is a pooling operator. The computation of the separation capacity of such a scattering network, based on \cref{thm:sep-cap}, involves, among other steps, analyzing the dimension of the vector space spanned by the image of the operator $U[\lambda]\colon \mathbb C^M\to \mathbb C^M, f \mapsto P(\rho(f \ast g_\lambda))$, $\lambda \in \Lambda$, under $\mathcal{L}^{2M}$-measurable sets $A \subseteq \mathbb C^M \simeq \mathbb R^{2M}$ of positive $\mathcal{L}^{2M}$-measure. To this end, consider the following upper bound for the operator $f \mapsto \rho(f \ast g_\lambda)$, which establishes a fundamental limit on the separation capacity of scattering networks.
\begin{lemma}\label{L:bound-subgroup}
    For $g_\lambda \in \mathbb C^M$, it holds that 
    \begin{align}\label{eq:bound-subgroup}
        \dimC{\spanC{\left\{\rho(f \ast g_\lambda) \colon f \in \mathbb C^M\right\}}}
            \leq \left \lvert\left\langle \supp{\widehat{g_\lambda}} \right\rangle\right\rvert.
    \end{align}
    Here, $\left\langle \supp{\widehat{g_\lambda}} \right\rangle$ denotes the subgroup generated by $\supp{\widehat{g_\lambda}} \subseteq \mathbb Z/M\mathbb Z$, i.e., the smallest subgroup of $\mathbb Z/M\mathbb Z$ containing $\supp{\widehat{g_\lambda}}$. 
\end{lemma}
\begin{proof}
    Since the DFT is linear and invertible, we can equivalently consider the space 
    \allowdisplaybreaks
    \begin{align*}
        \mathcal{U} \coloneqq \spanC{\left\{ F_M\left(\rho (f \ast g_\lambda)\right) \colon f \in \mathbb C^M\right\}},
    \end{align*}
    and show that $\dimC{\mathcal{U}} \leq \widetilde{M}$, where $\widetilde{M} \coloneqq \left \lvert\left\langle \supp{\widehat{g_\lambda}} \right\rangle\right\rvert$. We can assume that $\widetilde{M}<M$ because otherwise $\dimC{\mathcal{U}} \leq \widetilde{M}=M$ holds trivially. By Lagrange's theorem \cite[Theorem 1.5.2]{hall2018theory}, $\widetilde{M}$ divides $M$; moreover, the subgroup $\left\langle \supp{\widehat{g_\lambda}} \right\rangle$ is unique \cite[Theorem 3.1.1]{hall2018theory}. Thus, if $z \coloneqq f \ast g_\lambda$ with $f \in \mathbb C^M$, then 
    \begin{align}\label{eq:subgroup-sym-DFT}
        z_{(k+\widetilde{M}) \bmod{M}} = z_k, \quad k \in \{0,\ldots,M-1\}. 
    \end{align}
    Indeed, to see that \cref{eq:subgroup-sym-DFT} holds, rewrite $z_k$ as  
    \begin{align*}
        z_k &= \frac 1M \sum_{\ell=0}^{M-1} \widehat{z}_\ell \,e^{2\pi ik\ell/M} 
        = \frac 1M \sum_{s=0}^{\widetilde{M}-1} \widehat{z}_{sM/\widetilde{M}} \,e^{2\pi iks/\widetilde{M}}, 
    \end{align*}
    where we used that $\widehat{z}_\ell = 0$ if $\ell\neq sM/\widetilde{M}$ (i.e, if $\ell \notin \left\langle \supp{\widehat{g_\lambda}} \right\rangle$). Then, \cref{eq:subgroup-sym-DFT} follows easily according to
    \begin{align*}
        z_{(k+\widetilde{M}) \bmod{M}} = \frac 1M \sum_{s=0}^{\widetilde{M}-1} \widehat{z}_{sM/\widetilde{M}} \,e^{2\pi i((k+\widetilde{M}) \bmod{M})s/\widetilde{M}} = \frac 1M \sum_{s=0}^{\widetilde{M}-1} \widehat{z}_{sM/\widetilde{M}} \,e^{2\pi iks/\widetilde{M}} = z_k.
    \end{align*}
    Next, compute, for $k \in \{0,\ldots,M-1\}$,
    \begin{align*}
        \left(F_M\left( \rho(f \ast g_\lambda)\right)\right)_k &= \widehat{\rho(z)}_k \\
        &= \sum_{\ell=0}^{M-1} \rho(z_\ell) \,e^{-2\pi i k\ell/M} \\
        &= \sum_{r=0}^{M/\widetilde{M}-1} \sum_{s=0}^{\widetilde{M}-1} \rho\left(z_{r\widetilde{M}+s}\right) e^{-2\pi i k(r \widetilde{M}+s)/M} \\
        &= \sum_{s=0}^{\widetilde{M}-1} \rho\left(z_{s}\right) e^{-2\pi is/M} \sum_{r=0}^{M/\widetilde{M}-1} e^{-2\pi i kr/(M/\widetilde{M})},   
    \end{align*}
    where the last step is by \cref{eq:subgroup-sym-DFT}. Upon noting that
    \begin{align*}
        \left(\sum_{r=0}^{M/\widetilde{M}-1} \,e^{-2\pi i kr/(M/\widetilde{M})} \neq 0\right) \Longleftrightarrow \left(k=s M/\widetilde{M} \text{, for }s\in \mathbb Z\right),
    \end{align*}
    one can deduce that 
    \begin{align*}
        \supp{F_M(\rho(f \ast g_\lambda))} \subseteq \left\langle \supp{\widehat{g_\lambda}} \right\rangle, \quad \text{for every }f \in \mathbb C^M,
    \end{align*}
    and hence $\dimC{\mathcal{U}} \leq \left \lvert\left\langle \supp{\widehat{g_\lambda}} \right\rangle\right\rvert$.  
\end{proof}
We emphasize that the bound in \cref{eq:bound-subgroup} holds for \emph{every} nonlinearity $\rho \colon \mathbb C \to \mathbb C$ that is applied pointwise. In particular, the proof does not require any assumptions on $\rho$.  
As our aim is to identify module sequences inducing high separation capacity, we will now focus on characterizing the nonlinearities that achieve the upper bound in \cref{eq:bound-subgroup} when restricted to $\mathcal{L}^{2M}$-measurable sets of positive $\mathcal{L}^{2M}$-measure, i.e.,
\begin{align*}
    \dimC{\spanC{\left\{\rho(f \ast g_\lambda) \colon f \in A \right\}}} =
    \left \lvert\left\langle \supp{\widehat{g_\lambda}} \right\rangle\right\rvert,
\end{align*}
for all $\mathcal{L}^{2M}$-measurable $A \subseteq \mathbb C^M \simeq \mathbb R^{2M}$ with $\mathcal{L}^{2M}(A)>0$.
To this end, consider the general class of nonlinearities of the form
\begin{align}\label{eq:nonlinearity}
    \rho(z) = \varrho_{\mathrm{a}}(z)\overline{\varrho_{\mathrm{b}}(z)}, \quad z \in \mathbb C,
\end{align}
where $\varrho_{\mathrm{a}},\varrho_{\mathrm{b}} \colon \mathbb C \to \mathbb C$ are holomorphic off some set $S_{\mathrm{a,b}} \subseteq \mathbb C \simeq \mathbb R^2$ with $\mathcal{L}^2(S_{\mathrm{a,b}})=0$. Here, $S_{\mathrm{a,b}}$ is assumed to be closed with respect to the usual topology on $\mathbb C$. This encompasses a large class of nonlinearities such as holomorphic functions with isolated singularities (e.g., $\tan$ and $\tanh$), or functions that are holomorphic off some branch cut (e.g., fractional powers). In the context of Mallat's \cite{mallat2012group} construction of scattering networks, let us highlight that the modulus nonlinearity $\lvert \cdot \rvert$ is also of this form since $\lvert z \rvert = z^{1/2} \overline{z^{1/2}}$, $z \in \mathbb C$.      

\begin{theorem}\label{thm:general-node}
    Consider the nonlinearity $\rho$ of the form \cref{eq:nonlinearity}, and suppose that the following assumptions hold:
    \begin{enumerate}[label=(\roman*)]
        \item\label{it:general-node-i} There exists no connected component of $\mathbb{C} \setminus S_{\mathrm{a,b}}$ on which both $\varrho_{\mathrm{a}}$ and $\varrho_{\mathrm{b}}$ are polynomials.
        \item\label{it:general-node-ii} The filter $g_\lambda \in \mathbb{C}^M$ satisfies $\lvert \suppinline{\widehat{g_\lambda}} \rvert > 1$.
    \end{enumerate}
    Then, the map $\mathbb C^M \to \mathbb C^M, f \mapsto \rho(f \ast g_\lambda)$ satisfies condition \cref{eq:cover_cond_cap} and 
    \begin{align}\label{eq:thm-general-node}
        \dimC{\spanC{\left\{\rho(f \ast g_\lambda) \colon f \in \mathbb C^M\right\}}} =
        \left \lvert\left\langle \supp{\widehat{g_\lambda}} \right\rangle\right\rvert.
    \end{align}
\end{theorem}
\begin{proof}
    We will show that, for every $\mathcal{L}^{2M}$-measurable set $A \subseteq \mathbb C^M \simeq \mathbb R^{2M}$ with $\mathcal{L}^{2M}(A)>0$,
    \begin{align*}
        \dimC{\spanC{\left\{\rho(f \ast g_\lambda) \colon f \in A \right\}}} \geq
        \left \lvert\left\langle \supp{\widehat{g_\lambda}} \right\rangle\right\rvert.
    \end{align*}
    The claim will then follow since, by \cref{L:bound-subgroup}, 
    \begin{align*}
        \dimC{\spanC{\left\{\rho(f \ast g_\lambda) \colon f \in A \right\}}} \leq \dimC{\spanC{\left\{\rho(f \ast g_\lambda) \colon f \in \mathbb C^M\right\}}} \leq  \left \lvert\left\langle \supp{\widehat{g_\lambda}} \right\rangle\right\rvert.
    \end{align*}  
    Fix an $\mathcal{L}^{2M}$-measurable set $A \subseteq \mathbb C^M \simeq \mathbb R^{2M}$ with $\mathcal{L}^{2M}(A)>0$. For $k \in \{0,\ldots,M-1\}$, let $\xi_{\lambda,k} \colon \mathbb C^M \to \mathbb C, f \mapsto(f \ast g_\lambda)_k$. Define $\Delta_{\mathrm{a,b}} \coloneqq \mathbb C^M \setminus\bigcup_{k=0}^{M-1} \xi_{\lambda,k}^{-1}(S_{\mathrm{a,b}})$. Then,
    \begin{align}\label{eq:real-analytic_map}
        \Delta_{\mathrm{a,b}} \to \mathbb C^M, f \mapsto \rho (f \ast g_\lambda)
    \end{align}
    is real-analytic. Note that $\Delta_{\mathrm{a,b}}$ is an open subset of $\mathbb C^M$ by the continuity of $\xi_{\lambda,k}$, $k \in \{0,\ldots,M-1\}$. Moreover, we have $\mathcal{L}^{2M}\left(\xi_{\lambda,k}^{-1}(S_{\mathrm{a,b}})\right) = 0$, for all $k \in \{0,\ldots,M-1\}$, because $g_\lambda \neq 0$ and $\mathcal{L}^2(S_{\mathrm{a,b}})=0$ (see, e.g., \cite[Theorem 2]{ponomarev1987submersions}). Define $\Delta_A \coloneqq \Delta_{\mathrm{a,b}} \cap A$. Since $\mathbb C^M$ and $\Delta_{\mathrm{a,b}}$ differ only by a set of $\mathcal{L}^{2M}$-measure zero, it follows that $\mathcal{L}^{2M}(\Delta_A) = \mathcal{L}^{2M}(A)>0$. 
    Consider the space
    \begin{align*}
        \mathcal{U} \coloneqq \spanC{\left\{ F_M\left(\rho (f \ast g_\lambda)\right) \colon f \in \Delta_A\right\}}.
    \end{align*}
    Clearly, as $\Delta_A \subseteq A$, we have 
    \begin{align*}
        \dimC{\mathcal{U}} \leq \dimC{\spanC{\left\{\rho(f \ast g_\lambda) \colon f \in A\right\}}}.
    \end{align*} 
    Thus, it suffices to show that $\dimC{\mathcal{U}} \geq \left \lvert\left\langle \supp{\widehat{g_\lambda}} \right\rangle\right\rvert$. To do so, we will prove that $\dimC{\mathcal{U}^\perp} \leq M - \left\lvert\left\langle \supp{\widehat{g_\lambda}} \right\rangle\right\rvert$. Fix $a \in \{f \ast g_\lambda \colon f \in \Delta_A\}$. Then, $a_k \in \mathbb C \setminus S_{\mathrm{a,b}}$, for all $k \in \{0,\ldots,M-1\}$. Since $\mathbb C \setminus S_{\mathrm{a,b}}$ is open, it holds that, for $\sigma \in \{\mathrm{a}, \mathrm{b}\}$ and for all $k \in \{0,\dots,M-1\}$, there exists an $r^{(\sigma)}_k >0$ such that
    \begin{align*}
        \varrho_{\sigma}(z_k) = \sum_{\ell \in \mathbb N_0} c_{\sigma,k,\ell} (z_k-a_k)^\ell, \quad \text{$z_k \in \mathbb C$ with } \lvert z_k-a_k \rvert < r^{(\sigma)}_k, 
    \end{align*}
    where $c_{\sigma,k,\ell} = \frac{1}{\ell!}\restr{\frac{d^\ell}{dz^\ell}\varrho_\sigma(z)}{z=a_k}$, for all $\ell \in \mathbb N_0$.
    Set $r \coloneqq \min_{\sigma,k} r_k^{(\sigma)}>0$. With slight abuse of notation, we can write
    \begin{align*}
        \varrho_{\sigma}(z) \coloneqq \begin{pmatrix}
            \varrho_{\sigma}(z_0) \\
            \vdots \\
            \varrho_{\sigma}(z_{M-1})
        \end{pmatrix}= \sum_{\ell \in \mathbb N_0} c_{\sigma,\ell} (z-a)^\ell, \quad \text{$z \in \mathbb C^M$ with } \lVert z-a \rVert < r,
    \end{align*}
    where $c_{\sigma,\ell} \coloneqq \left(c_{\sigma,k,\ell} \right)_{0 \leq k \leq M-1} \in \mathbb C^M$. Here, the exponent and the vector multiplication are taken to be pointwise.
    Let $f \in \Delta_A$ such that $z\coloneqq f \ast g_\lambda$ satisfies $\lVert z- a \rVert < r$. Let $h \in \mathbb C^M$, and compute
    \allowdisplaybreaks
    \begin{align}
        \left\langle h, F_M(\rho(f \ast g_\lambda)) \right\rangle &= \left\langle h, \widehat{\rho(z)} \right\rangle 
        = \frac 1M \left\langle h, \widehat{\varrho_{\mathrm{a}}(z)} \ast \widehat{\overline{\varrho_{\mathrm{b}}(z)}} \right\rangle 
        = \frac 1M \left \langle C_h \widehat{\varrho_{\mathrm{b}}(z)}, \widehat{\varrho_{\mathrm{a}}(z)}\right \rangle. \nonumber
    \end{align}
    By the continuity of the DFT, we have
    \begin{align*}
        \widehat{\varrho_{\sigma}(z)} &= \sum_{\ell \in \mathbb N_0} \frac{1}{M^{\ell+1}} \widehat{c_{\sigma,\ell}} \ast (\widehat{z}-\widehat{a})^{*\ell},\quad \sigma \in \{\mathrm{a},\mathrm{b}\},
    \end{align*}
    and hence
    \begin{align}\label{eq:Ch_rho_a_rho_b}
        \left\langle h, F_M(\rho(f \ast g_\lambda)) \right\rangle &= \sum_{\ell,\ell' \in \mathbb N_0} \frac{1}{M^{\ell+\ell'+2}} \left\langle C_h C_{\widehat{c_{\mathrm{b},\ell}}}(\widehat{z}-\widehat{a})^{*\ell}, C_{\widehat{c_{\mathrm{a},\ell'}}}(\widehat{z}-\widehat{a})^{*\ell'}\right\rangle. 
    \end{align}
    Note that, for $\ell \in \mathbb N_0$,
    \begin{align*}
        \left((\widehat{z}-\widehat{a})^{*\ell}\right)_k &= \sum_{\substack{j_1,\ldots,j_\ell \in \{0,\ldots,M-1\} \\ j_1+\cdots +j_\ell \equiv k \Mod{M}}} (\widehat{z}-\widehat{a})_{j_1} \cdots (\widehat{z}-\widehat{a})_{j_\ell}\\
        &=\sum_{\substack{\alpha \in \mathbb N_0^M\\ \lvert \alpha \rvert = \ell \\ \sum_{t=0}^{M-1}t \alpha_t \equiv k \Mod{M}}} \binom{\ell}{\alpha}(\widehat{z}-\widehat{a})^\alpha, \quad k \in \{0, \dots,M-1\}.
    \end{align*}
     Set $\widetilde{z} \coloneqq (\widehat{z}-\widehat{a})_{k \in \supp{\widehat{g_\lambda}}}\in \mathbb C^R$, where $R \coloneqq \lvert \supp{\widehat{g_\lambda}} \rvert$.
    For every $\ell \in \mathbb N_0$, define 
    \begin{align*}
        \widetilde{\mathcal{Z}}_{\ell} \coloneqq \begin{pmatrix}
            \widetilde{z}_0^\ell,\ell \widetilde{z}_0^{\ell-1}\widetilde{z}_1, \dots, \widetilde{z}_{R-1}^\ell  
        \end{pmatrix} \in \mathbb C^{\binom{R-1+\ell}{\ell}},
    \end{align*}
    i.e., $\widetilde{\mathcal{Z}}_{m}$ contains the monomials $\binom{\ell}{\alpha} \widetilde{z}^\alpha$, $\alpha \in \mathbb N_0^R$ with $\lvert \alpha \rvert =\ell$, in degree lexicographic order.
    Then, we can write
    \begin{align*}
        (\widehat{z}-\widehat{a})^{*\ell} = A_\ell \widetilde{\mathcal{Z}}_{\ell},
    \end{align*}
    for some $A_\ell \in \mathbb R^{M\times \binom{R-1+\ell}{\ell}}$ with entries taking values in $\{0,1\}$. Note that each column of $A_\ell$ has exactly one nonzero entry and that nonzero rows of $A_\ell$ correspond to
    \begin{align*}
        \left\langle \supp{\widehat{g_\lambda}}\right\rangle_\ell \coloneqq  \underbrace{\supp{\widehat{g_\lambda}} \cdots \supp{\widehat{g_\lambda}}}_{\ell \text{ times}}, 
    \end{align*}
    where 
    \begin{align*}
        \left \lvert \left\langle \supp{\widehat{g_\lambda}}\right\rangle_\ell \right \rvert = \left \lvert \left\{ \left(\sum_{t=0}^{M-1} t \alpha_t\right) \bmod{M} \colon \alpha \in \mathbb N_0^M, \supp{\alpha} \subseteq \supp{\widehat{g_\lambda}}\right\} \right\rvert.
    \end{align*}
    Thus, $\mathrm{rank}(A_\ell) = \lvert \langle \suppinline{\widehat{g_\lambda}}\rangle_\ell \rvert$.
    Computing
    \begin{align*}
        \left\langle C_h C_{\widehat{c_{\mathrm{b},\ell}}}(\widehat{z}-\widehat{a})^{*\ell}, C_{\widehat{c_{\mathrm{a},\ell'}}}(\widehat{z}-\widehat{a})^{*\ell'}\right\rangle &= \left\langle C_{\widehat{c_{\mathrm{a},\ell'}}}^{\mathsf{H}}C_h C_{\widehat{c_{\mathrm{b},\ell}}}(\widehat{z}-\widehat{a})^{*\ell}, (\widehat{z}-\widehat{a})^{*\ell'}\right\rangle \\
        &= \left\langle C_{\widehat{c_{\mathrm{a},\ell'}}}^{\mathsf{H}}C_h C_{\widehat{c_{\mathrm{b},\ell}}}A_\ell \widetilde{\mathcal{Z}}_{\ell}, A_{\ell'} \widetilde{\mathcal{Z}}_{\ell'}\right\rangle \\
        &=\left\langle A_{\ell'}^{\mathsf{H}}C_{\widehat{c_{\mathrm{a},\ell'}}}^{\mathsf{H}}C_h C_{\widehat{c_{\mathrm{b},\ell}}}A_\ell \widetilde{\mathcal{Z}}_{\ell}, \widetilde{\mathcal{Z}}_{\ell'}\right\rangle,
    \end{align*}
    and substituting this into \cref{eq:Ch_rho_a_rho_b} yields
    \begin{align}\label{eq:power-series-general-node}
        \left\langle h, F_M(\rho(f \ast g_\lambda)) \right\rangle &= \sum_{\ell,\ell' \in \mathbb N_0}\frac{1}{M^{\ell+\ell'+2}} \left\langle A_{\ell'}^{\mathsf{H}}C_{\widehat{c_{\mathrm{a},\ell'}}}^{\mathsf{H}}C_h C_{\widehat{c_{\mathrm{b},\ell}}}A_\ell \widetilde{\mathcal{Z}}_{\ell}, \widetilde{\mathcal{Z}}_{\ell'}\right\rangle.  
    \end{align}
    If $h \in \mathcal{U}^\perp$, then the left-hand side (LHS) of \cref{eq:power-series-general-node} vanishes for all $f \in \Delta_A$. Note, however, that the LHS is real-analytic in $f$ on the open set $\Delta_{\mathrm{a,b}}$ and that $\mathcal{L}^{2M}(\Delta_A)>0$, which implies that the LHS vanishes on an open connected set $V \subseteq\Delta_{\mathrm{a,b}}$ (as $\Delta_{\mathrm{a,b}} \subseteq \mathbb C^M$ has at most countably many open connected components). Since the map $f \mapsto \widetilde{z}$ is affine linear and surjective, it follows by the open mapping theorem \cite[Theorem 2.11]{rudin1991functional} that there is an $\widetilde{r}>0$ such that the RHS of \cref{eq:power-series-general-node} vanishes, for all $\widetilde{z} \in \mathbb C^R$ with $\lVert \widetilde{z} \rVert < \widetilde{r}$. We now take the partial derivatives $\partial^{\ell+\ell'}/\partial \widetilde {z}^\alpha\partial \overline{\widetilde{z}}^\beta$ of the RHS of \cref{eq:power-series-general-node}, where $\ell,\ell' \in \mathbb N_0$, $\alpha, \beta \in \mathbb N_0^R$ with $\lvert \alpha \rvert =\ell$, $\lvert \beta \rvert =\ell'$. By Abel's lemma \cite[Lemma 1]{narasimhan1971scv} both the RHS of \cref{eq:power-series-general-node} and the series of derivatives 
    \begin{align*}
        \sum_{\ell,\ell' \in \mathbb N_0}\frac{1}{M^{\ell+\ell'+2}} \frac{\partial^{m+m'}}{\partial \widetilde {z}^\gamma\partial \overline{\widetilde{z}}^\delta}\left\langle A_{\ell'}^{\mathsf{H}}C_{\widehat{c_{\mathrm{a},\ell'}}}^{\mathsf{H}}C_h C_{\widehat{c_{\mathrm{b},\ell}}}A_\ell \widetilde{\mathcal{Z}}_{\ell}, \widetilde{\mathcal{Z}}_{\ell'}\right\rangle,
    \end{align*}
    converge uniformly on $\{\widetilde{z} \in \mathbb C^R\colon \lVert \widetilde{z} \rVert < \widetilde{r}\}$, for all $m,m' \in \mathbb N_0$, $\gamma, \delta \in \mathbb N_0^R$ with $\lvert \gamma \rvert =m$, $\lvert \delta \rvert =m'$,
    so that partial differentiation $\partial^{\ell+\ell'}/\partial \widetilde {z}^\alpha\partial \overline{\widetilde{z}}^\beta$ and summation may be interchanged.  
    Evaluating the resulting expression at $\widetilde{z}=0$ yields
    \begin{align*}
        A_{\ell'}^{\mathsf{H}}C_{\widehat{c_{\mathrm{a},\ell'}}}^{\mathsf{H}}C_h C_{\widehat{c_{\mathrm{b},\ell}}}A_\ell = 0, \quad \text{for all $\ell,\ell' \in \mathbb N_0$}. 
    \end{align*}
    If $\varrho_{\mathrm{a}}$ is not a polynomial, we can pick $\ell' \in \mathbb N_0$ large enough such that $\mathrm{rank}(A_{\ell'}) = \left\lvert \left\langle \supp{\widehat{g_\lambda}}\right\rangle \right\rvert$, as $\mathrm{rank}(A_\ell) = \lvert \langle \suppinline{\widehat{g_\lambda}}\rangle_\ell \rvert$, and such that $c_{\mathrm{a},\ell'}$ contains only nonzero entries (i.e., $C_{\widehat{c_{\mathrm{a},\ell'}}}^{\mathsf{H}}$ is of full rank). Moreover, choose $\ell \in \mathbb N_0$ such that $c_{\mathrm{b},\ell}$ contains only nonzero entries (i.e., $C_{\widehat{c_{\mathrm{b},\ell}}}^{\mathsf{H}}$ is of full rank). Recall the structure of the matrices $A_{\ell'}^{\mathsf{H}}$ and $A_{\ell}$: each row of $A_{\ell'}^{\mathsf{H}}$ contains exactly one nonzero entry and the nonzero columns of $A_{\ell'}^{\mathsf{H}}$ correspond to $\langle\supp{\widehat{g_\lambda}} \rangle$; on the other hand, each column of $A_\ell$ contains exactly one nonzero entry and the nonzero rows correspond to $\left\langle \supp{\widehat{g_\lambda}}\right\rangle_\ell \subseteq \left\langle \supp{\widehat{g_\lambda}}\right\rangle$. It follows that the submatrix of $C_{\widehat{c_{\mathrm{a},\ell'}}}^{\mathsf{H}}C_h C_{\widehat{c_{\mathrm{b},\ell}}}$ obtained from the rows and columns corresponding to $\langle\supp{\widehat{g_\lambda}} \rangle$ and $\left\langle \supp{\widehat{g_\lambda}}\right\rangle_\ell$, respectively, must be zero. Since $C_{\widehat{c_{\mathrm{a},\ell'}}}^{\mathsf{H}}C_h C_{\widehat{c_{\mathrm{b},\ell}}}$ is circular, we impose $\left\lvert \left\langle \supp{\widehat{g_\lambda}}\right\rangle \right\rvert$ linearly independent conditions. Using that $C_{\widehat{c_{\mathrm{a},\ell'}}}^{\mathsf{H}}$ and $C_{\widehat{c_{\mathrm{b},\ell}}}$ are invertible, it follows
    $\dimCinline{\mathcal{U}^\perp} \leq M-\left\lvert \left\langle \supp{\widehat{g_\lambda}}\right\rangle \right\rvert$. Similarly, we can derive this result if $\varrho_{\mathrm{b}}$ is not a polynomial. This completes the proof.   
\end{proof}
\begin{remark}[Discussion of assumptions in \cref{thm:general-node}] If one of the assumptions \cref{it:general-node-i,it:general-node-ii} in \cref{thm:general-node} is violated, then \cref{eq:thm-general-node} no longer holds in general. Indeed, for \cref{it:general-node-i}, \cref{L:dim_conv_mod2} constitutes a counterexample with $\varrho_{\mathrm{a}}(z)=\varrho_{\mathrm{b}}(z)=z$, $z \in \mathbb C$. Note that, in the setting of \cref{L:dim_conv_mod2}, the filter $g_\lambda$ satisfies \cref{it:general-node-ii} whenever $R>1$. But in this case, we have $\lvert \langle\suppinline{\widehat{g_\lambda}}\rangle\rvert = M \neq 2R-1$, and hence \cref{eq:thm-general-node} is false. 
For \cref{it:general-node-ii}, consider the case where $\suppinline{\widehat{g_\lambda}} = \{\lambda\}$, for some $\lambda \in \mathbb Z/M\mathbb Z$ with $\lambda \neq 0$, and $\varrho_{\mathrm{a}}(z)=\varrho_{\mathrm{b}}(z) =z^{1/2}$, so that $\rho(z)=\lvert z \rvert$, $z \in \mathbb C$. Clearly, \cref{it:general-node-i} holds. Now observe that for $f \in \mathbb C^M$ and $k \in \{0,\dots,M-1\}$,
    \begin{align*}
        (f \ast g_\lambda)_k &= \frac{1}{M} \sum_{\ell=0}^{M-1} \widehat{f}_\ell (\widehat{g_{\lambda}})_\ell e^{2\pi i k\ell/M}
        =\frac{1}{M} \widehat{f}_{\lambda} e^{2\pi i k\lambda/M}, 
    \end{align*}
    and hence
    \begin{align*}
        \lvert (f \ast g_\lambda)_k \rvert = \frac 1M \lvert \widehat{f}_{\lambda} \rvert.
    \end{align*}
    As $f\in \mathbb C^M$ was arbitrary, this immediately implies $\dimCinline{\spanCinline{{\rho(f \ast g_\lambda) \colon f \in \mathbb C^M\}}}} = 1$. But $\lvert \langle \supp{\widehat{g_\lambda}}\rangle \rvert > 1$, as $\lambda \neq 0$, and consequently, \cref{eq:thm-general-node} does not hold for general $M \in \mathbb N$.       
\end{remark}

From \cref{thm:general-node} we can thus conclude that several nonlinearities which are employed in practice, such as, e.g., modulus $\lvert \cdot \rvert$ or $\tanh$, are optimal in the sense that they achieve the bound in \cref{eq:bound-subgroup}, under the assumption that $\lvert \suppinline{\widehat{g_\lambda}} \rvert >1$, which is easily met in practice.
Moreover, we infer that desirable design choices of the frame $\Psi$ (in the sense of achieving a large separation capacity) are such that $\supp{\widehat{g_\lambda}}$ is \emph{not} a subset of a proper subgroup of $\mathbb Z/M\mathbb Z$ because then $\lvert\langle \supp{\widehat{g_\lambda}}\rangle\rvert=M$ is maximized. For instance, this is the case if $\supp{\widehat{g_\lambda}}$ contains two consecutive elements, a condition that is easy to fulfill in practice.
We emphasize that the nonlinearity being applied \emph{pointwise} is crucial. Indeed, if we use, for example, convolution power as nonlinearity, i.e., $\mathbb C^M \to \mathbb C^M, f \mapsto f^{\ast d}$ with $d \in \mathbb N$, which is clearly not pointwise, then 
\begin{align*}
    \dimC{\spanC{\left\{\left(f \ast g_\lambda\right)^{\ast d}\colon f \in \mathbb C^M\right\}}} \leq \lvert \supp{\widehat{g_\lambda}} \rvert. 
\end{align*}
Note that the cardinality of $\supp{\widehat{g_\lambda}}$ can be significantly smaller than $\lvert \langle \supp{\widehat{g_\lambda}}\rangle \rvert$, especially if $\supp{\widehat{g_\lambda}}$ is not a subgroup of $\mathbb Z/M\mathbb Z$, which is typically the case.   

\paragraph{Effect of pooling on the separation capacity} 
Let us proceed with our analysis of the operator $U[\lambda]\colon \mathbb C^M\to \mathbb C^M, f \mapsto P(\rho(f \ast g_\lambda))$ by focusing now on the characteristics of the pooling operator $P$. 
Pooling operators play a central role in feature extraction. Specifically, it is shown in \cite{wiatowski2017deep} that the presence of pooling is essential for the feature extractor to be (vertically) translation invariant. As in \cite{wiatowski2017deep}, we consider two classes of pooling operations which are often used in practice, namely, subsampling and averaging. 
\emph{Subsampling by a factor of} $S \in \{0,\dots,M-1\}$ is defined according to\footnote{The subscript in $h_{\mathrm{d}}$ stands for decimation.}
\begin{align}\label{eq:subsampling}
    \mathbb C^M \to \mathbb C^M, \quad f=(f_k)_{0 \leq k \leq M-1} \mapsto h_{\mathrm{d}} \coloneqq \left(f_{(kS) \bmod M}\right)_{0 \leq k \leq M-1}.
\end{align}
Note that for $S=0$, we obtain the constant signal $(h_{\mathrm{d}})_k = f_0$, for all $k \in \{0,\ldots,M-1\}$.
Key to analyzing the effect of subsampling on the separation capacity is, as in the derivations above, computing the DFT of $h_\mathrm{d}$. Indeed, as we shall see in the next lemma, subsampling potentially induces zeroes in $\widehat{h_\mathrm{d}}$ so that the dimension of the vector space spanned by $h_\mathrm{d}$ can be easily determined.
\begin{lemma}[Subsampling]\label{L:subsampling}
    Consider the pooling operation subsampling by a factor of $S \in \{0,\dots,M-1\}$ defined in \cref{eq:subsampling}. Set $\widetilde S \coloneqq S / \gcd(M,S)$ and $\widetilde{M} \coloneqq M /\gcd(M,S)$.
    For $k \in \{0,\dots,M-1\}$, it holds that
    \begin{align*}
        (\widehat{h_{\mathrm{d}}})_k &= \begin{cases}
            \sum_{r=0}^{\gcd(M,S)-1} \widehat{f}_{\left(\widetilde{S}^{-1}k/\gcd(M,S)+r \widetilde{M}\right) \bmod{M}}, &\quad\text{if $k \equiv 0 \Mod{\gcd(M,S)}$,} \\
            0,& \quad\text{otherwise,}
            \end{cases}
    \end{align*}
    where $\widetilde S^{-1}$ denotes the multiplicative inverse\,\footnote{Note that $\widetilde{S}^{-1}$ exists, as $\widetilde{S}$ and $\widetilde{M}$ are coprime.} of $\widetilde{S}$ in $\mathbb Z/{\widetilde{M}}\mathbb Z$. 
\end{lemma}
\begin{proof}
    See \cref{app:subsampling}.
\end{proof}
Before discussing the ramifications of this result on the separation capacity in more detail, let us conduct the same analysis for the average pooling operation.   
\emph{Average pooling} is defined by
\begin{align}\label{eq:averagepooling}
    \mathbb C^M \to \mathbb C^M, \quad f=(f_k)_{0 \leq k \leq M-1} \mapsto h_{\phi,\mathrm{d}} \coloneqq \left((f \ast \phi)_{(kS) \bmod M}\right)_{0 \leq k \leq M-1},
\end{align}
where $\phi \in \mathbb C^M$ is the averaging kernel and $S \in \{0,\dots,M-1\}$ the subsampling factor. From \cref{L:subsampling} we immediately obtain the following: 
\begin{lemma}[Average pooling]\label{L:avg_pooling}
    Consider the average pooling operation as defined in \cref{eq:averagepooling}. With $\widetilde S \coloneqq S / \gcd(M,S)$ and $\widetilde{M} \coloneqq M /\gcd(M,S)$, we have for $k \in \{0,\dots,M-1\}$, 
    \begin{align*}
        (\widehat{h_{\phi,\mathrm{d}}})_k &= \begin{cases}
            \sum_{r=0}^{\gcd(M,S)-1} (\widehat{f\ast \phi})_{\left(\widetilde{S}^{-1}k/\gcd(M,S)+r \widetilde{M}\right) \bmod{M}}, &\quad\text{if $k \equiv 0 \Mod{\gcd(M,S)}$,} \\
            0,& \quad\text{otherwise.}
            \end{cases}
    \end{align*}
    Here, $\widetilde S^{-1}$ is the multiplicative inverse of $\widetilde{S}$ in $\mathbb Z/{\widetilde{M}}\mathbb Z$.  
\end{lemma}
As an immediate consequence of the preceding two lemmata, the effect of pooling on the separation capacity can now be characterized.
\begin{theorem}
    Pooling by subsampling or averaging reduces the separation capacity if one of the following conditions holds: 
    \begin{enumerate}[label=(\roman*)]
        \item The subsampling factor $S$ and the dimension of the domain of the pooling operator $M$ are not coprime, i.e., $\gcd(M,S) \neq 1$.
        \item The averaging kernel $\phi$ is spectrally supported on a proper subset of $\mathbb Z/M\mathbb Z$, i.e., $\suppinline{\widehat{\phi}} \subsetneq \mathbb Z/M\mathbb Z$. 
    \end{enumerate}
\end{theorem}
More precisely, the extent to which the separation capacity is reduced increases as $\gcd(M,S)$ becomes larger or as $\lvert \mathrm{supp}({\widehat{\phi}}) \rvert$ becomes smaller. 
This yields a trade-off in the design of the scattering network, as pooling is necessary to obtain (vertical) translation invariance (see \cite{wiatowski2017deep}).

\paragraph{Multi-layer networks}
Exact and rigorous separation capacity computations for multi-layer networks employing general module sequences $\{(\Psi_n,\rho_n,P_n)\}_{n \in \mathbb N}$, as considered in \cref{sec:feature-extractors}, are complex due to potential linear dependencies of the outputs of nodes both between and within layers. Nevertheless, \cref{L:bound-subgroup} allows us to derive an upper bound on the separation capacity of a multi-layer network $\Phi$ of depth $n_{\mathrm{d}} \in \mathbb N$ constructed from $\{(\Psi_n,\rho_n,P_n)\}_{n \in \mathbb N}$, see \cref{eq:CNN-feature-extractor}. Namely, upon noting that the second term in \cref{eq:SC-spanC} is bounded by $2\cdot\dimCinline{\spanCinline{\Phi(\mathbb C^M)}}$ for general complex-valued $\Phi$, one obtains  
\begin{align*}
    \sepcap{\Phi} \leq 4 \cdot\dimC{\spanC{\Phi(\mathbb C^M)}}.
\end{align*}
Now observe that for the input signal $f \in \mathbb C^M$, the output of a node in the $n$th layer associated with the path $(q,\lambda_n) \in \Lambda_1^{n-1}\times \Lambda_n$ 
is given by $P_n(((U[q]f) \ast g_{\lambda_n}) \ast \chi_{n+1}$. Applying \cref{L:bound-subgroup} to $u \mapsto \rho_n(u \ast g_{\lambda_n})$ and summing over all nodes in the scattering tree yields
\begin{align*}
    \sepcap{\Phi} \leq 4\left\lvert\supp{\widehat{\chi_1}} \right\rvert + 4\sum_{n=1}^{n_\mathrm{d}}\sum_{(\lambda_1,\ldots,\lambda_n) \in \Lambda_1^n} \left\lvert\left\langle \supp{\widehat{g_{\lambda_n}}}\right\rangle \cap \supp{\widehat{\chi_{n+1}}} \right\rvert.
\end{align*} 

In practice, scattering networks employ only the feature maps from the first few layers. This is due to the phenomenon of energy decay. Specifically, it is shown in \cite{wiatowski2017energy} that the energy contained in the feature maps decays at least polynomially across layers, i.e., $\sum_{q \in \Lambda_1^n}\lVert U[q]f \rVert^2 \to 0$ as $n \to \infty$ at least polynomially fast, for all $f \in \mathbb C^M$. This decay effect can also be observed in \cref{fig:WH-mod2-layer2_1,fig:WH-mod2-layer2_L}. Consequently, in practice, the first few layers are of significant importance. We note however that the energy of the feature map does not have an impact on the separation capacity as long as the feature maps are nonzero (i.e., of positive energy). Nevertheless, as only the first few layers are relevant in practice, the scattering network should be designed such that the feature maps of the first layers completely fill out the codomain and hence achieve the maximum possible separation capacity within the first layers.

\subsection{Revisiting the Weyl--Heisenberg frame}
Having identified the driving and limiting factors for achieving high separation capacity, let us now revisit our example in \cref{subsec:WHmod2}, which failed to accomplish this, and discuss why this example fell short and how the module sequence can be adjusted to improve the separation capacity.

To see why the module sequence $\{(\Psi_{\mathrm{WH}},\lvert \cdot \rvert^2,\mathrm{Id})\}_{n \in \mathbb N}$, introduced in \cref{subsec:WHmod2}, did not result in a scattering network of high separation capacity, we first note that the nonlinearity $\lvert \cdot \rvert^2$ is of the form \cref{eq:nonlinearity}, but $(z,\overline{z})\mapsto\lvert z \rvert^2$ is a polynomial of degree $1$ in both $z$ and $\overline{z}$. In particular, the assumptions in \cref{thm:general-node} are not met. \cref{L:dim_conv_mod2} shows that the upper bound of $\lvert \langle \supp{\widehat{g_\lambda}} \rangle\rvert$ in \cref{L:bound-subgroup} is not attained, where $g_\lambda$ is an atom of $\Psi_{\mathrm{WH}}$.

This suggests that in order to improve the separation capacity of the resulting networks and to construct a scattering network of high separation capacity, one may employ a different nonlinearity. 
As noted in the previous subsection, the first layers are of significant importance in practice. Hence, it is crucial that the feature maps in the first layer are such that
\begin{align*}
    \dimC{\spanC{\left\{\left(\rho(f \ast g_\lambda)\right)_{\lambda \in \Lambda} \colon f \in \mathbb C^M\right\}}}
\end{align*}
is maximized, where $\rho \colon \mathbb C \to \mathbb C$ is a nonlinearity applied pointwise and $\{g_\lambda\}_{\lambda \in \Lambda}$ are the atoms a frame.   

Consider now the module sequence $\{(\Psi_{\mathrm{WH}}, \rho,\mathrm{Id})\}_{n \in \mathbb N}$, where $\Psi_{\mathrm{WH}}$ is the Weyl--Heisenberg frame introduced in \cref{subsec:WHmod2}, and where $\rho$ is a nonlinearity satisfying the assumption of \cref{thm:general-node}. Note that here $\supp{\widehat{g_\lambda}}$ is not a subset of a proper subgroup of $\mathbb Z/M \mathbb Z$ for every $\lambda \in \Lambda$ if $R>1$. Moreover, $\{\supp{\widehat{g_\lambda}}\}_{\lambda \in \Lambda}$ are disjoint. Thus, from \cref{thm:general-node}, we immediately obtain that
\begin{align}\label{eq:1layerWHrho}
    \dimC{\spanC{\left\{\left(\rho(f \ast g_\lambda)\right)_{\lambda \in \Lambda} \colon f \in \mathbb C^M\right\}}} = LM = (M/R-1)M,
\end{align}
which is maximized for $R=3$. Compare this to  
\begin{align}\label{eq:1layerWHmod2}
    \dimC{\spanC{\left\{\left(\lvert f \ast g_\lambda\rvert^2\right)_{\lambda \in \Lambda} \colon f \in \mathbb C^M\right\}}} = L(2R-1) &= (M/R-1)(2R-1)\nonumber\\ &= 2M-2R-M/R+1,
\end{align}
where we used \cref{L:dim_conv_mod2}. Here, the maximum is achieved at\footnote{Here, let us disregard the constraint that $R$ must be an odd integer.} $R= \sqrt{M/2}$. Observe that \cref{eq:1layerWHrho} scales quadratically in $M$, while \cref{eq:1layerWHmod2} exhibits at most a linear scaling behavior in $M$. Therefore, by \cref{thm:sep-cap}, we can conclude that $\{(\Psi_{\mathrm{WH}}, \rho,\mathrm{Id})\}_{n \in \mathbb N}$ yields a feature extractor of significantly higher separation capacity than the one obtained from $\{(\Psi_{\mathrm{WH}}, \lvert \cdot \rvert^2,\mathrm{Id})\}_{n \in \mathbb N}$ in the regime $M \to \infty$.      

\section{Insights for Scattering Network Design} \label{sec:design-insights}
Our analysis yields the following the design principle for scattering networks in practice. To attain a high separation capacity, the module sequence $\{(\Psi_n,\rho_n,P_n)\}_{n \in \mathbb N}$ should be such that the resulting scattering network fills out its codomain within the first few layers. More precisely, if $\Phi$ denotes the induced scattering network of depth $n_{\mathrm{d}}$, where $n_{\mathrm{d}}$ is small, then
$\dimCinline{\spanCinline{\Phi(\mathbb C^M)}}$ should be close to the dimension of the codomain, i.e., $M\sum_{n=0}^{n_{\mathrm{d}}} \lvert \Lambda_1^n\rvert$. Accordingly, when choosing the frame $\Psi_n$, the nonlinearity $\rho_n$, and the pooling operator $P_n$, the following aspects should be taken into account. 
\begin{enumerate}[label=(\alph*)]
    \item\label{it:design1} The frame--nonlinearity pair $(\Psi_n,\rho_n)$ should be selected jointly, as we have seen in the Weyl--Heisenberg frame example.  Specifically, pairing $\Psi_{\mathrm{WH}}$ with a polynomial nonlinearity, such as $\rho_n(z,\overline{z})=z\overline{z}$, $z \in \mathbb C$, yields a low separation capacity within the first few layers. On the other hand, using non-polynomial nonlinearities, such as $\rho_n(z,\overline{z})=z^{1/2}\overline{z}^{1/2}$, $z \in \mathbb C$, for $\Psi_{\mathrm{WH}}$ gives a significantly higher separation capacity. 
    \item\label{it:design2} As established in \cref{thm:general-node}, a high separation capacity within the first few layers can be achieved if (i) $\rho_n$ is not a polynomial and (ii) the atoms $\{g_{\lambda_n}\}_{\lambda_n \in \Lambda_n}$ of the frame $\Psi_n$ are such that $\suppinline{\widehat{g_{\lambda_n}}}$ is not a subgroup of $\mathbb Z/M\mathbb Z$.  
    \item Pooling by subsampling or averaging generally reduces separation capacity, specifically, if subsampling factor $S$ is such that $\gcd(M,S) \neq 1$ or spectral support of averaging kernel is proper subset of $\mathbb Z/M\mathbb Z$.
\end{enumerate}
To conclude the paper, we discuss examples of frame--nonlinearity pairs resulting in high separation capacities within the first few layers. To this end, motivated by the framework presented in \cite{vashisht2017necessary}, we define wavelet frames on finite cyclic groups as follows.   

\paragraph{Wavelet frames} 
Let $\psi,\phi \in \mathbb C^M$, which are often referred to as the mother and father wavelet, respectively.
The wavelet frame $\Psi_{\mathrm{wvt}}$ is formed by the atoms $\{\chi\} \cup \{g_\lambda\}_{\lambda \in \Lambda} \subset\mathbb C^M$, where $\chi \coloneqq \phi$ is set to be the father wavelet, and where $\{g_\lambda\}_{\lambda \in \Lambda}$ is obtained by dilating the mother wavelet $\psi$. Specifically, $g_\lambda$ is given according to
\begin{align*}
    (g_\lambda)_k \coloneqq \psi_{(\lambda k) \bmod{M}}, \quad k \in \{0,\dots,M-1\}.
\end{align*}
The index set $\Lambda$ is assumed to be such that $\Lambda \subseteq (\mathbb Z/M\mathbb Z)^{\times}$, where $(\mathbb Z/M\mathbb Z)^{\times}$ denotes the set of integers in $\{0,\dots,M-1\}$ that are coprime to $M$. The condition $\Lambda \subseteq (\mathbb Z/M\mathbb Z)^{\times}$ ensures, by \cref{L:subsampling}, that $(\widehat{g_\lambda})_k = \widehat{\psi}_{(\lambda^{-1}k) \bmod{M}}$, $k \in \{0,\dots,M-1\}$, where $\lambda^{-1}$ denotes the multiplicative inverse\footnote{The multiplicative inverse $\lambda^{-1}$ exists, as $\gcd(\lambda,M)=1$, for all $\lambda \in \Lambda$.} of $\lambda$ in $\mathbb Z/M\mathbb Z$.     
Notably, $\Lambda$ can be chosen to be a collection of dyadic dilations $\{(2^j \bmod{M}) \colon j \geq 0\}$ if $M$ is odd.
The Littlewood--Paley condition
\begin{align}\label{eq:LittlewoodPaleyWavelet}
    A \lVert f \rVert^2 \leq \lVert f \ast \chi \rVert^2+\sum_{\lambda \in \Lambda} \lVert f \ast g_\lambda \rVert^2 \leq B \lVert f \rVert^2, \quad \text{for all $f \in \mathbb C^M$,}
\end{align}
with $0<A\leq B < \infty$, holds if and only if $\phi,\psi$ are such that
\begin{align}\label{eq:equivalent_LP_condition}
    \left\lvert \widehat{\phi}_k \right\rvert^2+\sum_{\lambda \in \Lambda} \left\lvert \widehat{\psi}_{(\lambda^{-1}k) \bmod{M}} \right\rvert^2 >0, \quad \text{for all $k \in \{0,\dots,M-1\}$,}
\end{align}
Indeed, by Parseval's identity and the convolution property of the DFT, \cref{eq:LittlewoodPaleyWavelet} is equivalent to
\begin{align*}
    A \left\lVert \widehat{f} \right\rVert^2 \leq  \sum_{k=0}^{M-1} \left\lvert \widehat{f}_k \right\rvert^2 \left(\left\lvert \widehat{\chi}_k \right\rvert^2 +\sum_{\lambda \in \Lambda} \left\lvert (\widehat{g_\lambda})_k \right\rvert^2\right) \leq B \left\lVert \widehat{f} \right\rVert^2, \quad \text{for all $f \in \mathbb C^M$.}  
\end{align*}
The claim now follows as $\widehat{\chi}_k = \widehat{\phi}_k$ and $(\widehat{g_\lambda})_k = \widehat{\psi}_{(\lambda^{-1}k) \bmod{M}}$, $k \in \{0,\dots,M-1\}$, with frame bounds
\begin{align*}
    A &= \min_{k\in\{0,\ldots,M-1\}} \left(\left\lvert \widehat{\phi}_k \right\rvert^2+\sum_{\lambda \in \Lambda} \left\lvert \widehat{\psi}_{(\lambda^{-1}k) \bmod{M}} \right\rvert^2\right) >0 \\
    \intertext{and}
    B &= \max_{k\in\{0,\ldots,M-1\}} \left(\left\lvert \widehat{\phi}_k \right\rvert^2+\sum_{\lambda \in \Lambda} \left\lvert \widehat{\psi}_{(\lambda^{-1}k) \bmod{M}} \right\rvert^2\right) <\infty. 
\end{align*}
Note that the wavelet frame is given by $\Psi_{\mathrm{wvt}} = \bigcup_{k=0}^{M-1}(\{T_k\chi^*\} \cup \{T_kg_\lambda^*\}_{\lambda \in \Lambda})$, where $T_k$ is the translation operator and $*$ denotes involution.

\paragraph{Numerical experiment}
We now compute the separation capacities of scattering networks of depth $n_\mathrm{d}=3$ constructed from $\{(\Psi,\rho,\mathrm{Id})\}_{n\in \mathbb N}$ numerically, where $\Psi \in \{\Psi_{\mathrm{WH}},\Psi_\mathrm{wvt}\}$ and $\rho \in \{\lvert \cdot \rvert^2,\lvert \cdot \rvert,\tanh(\cdot),\sig{\cdot}\}$. Here, $\tanh(z) = \tanh(\Re(z))+i \tanh(\Im(z))$ with $\tanh(x) = (e^{x}-e^{-x})/(e^{x}+e^{-x})$
and $\sig{z} = \sig{\Re(z)}+ i \sig{\Im(z)}$ with $\sig{x} = 1/(1+\exp{(-x)})-1/2$, for $z \in \mathbb C$ and $x \in \mathbb R$. The parameters of $\Psi_{\mathrm{WH}}$ are chosen as $M=25$ and $R=5$. For $\Psi_\mathrm{wvt}$, we likewise fix $M=25$ and take $\phi = \chi$, where $\chi$ is the output-generating filter used in $\Psi_{\mathrm{WH}}$. We further set $\widehat{\psi}_k = \mathbbm{1}_{\{8\leq k \leq 17\}}$, $k \in \{0,\ldots,M-1\}$. The remaining wavelet atoms are generated via the dyadic dilations $\Lambda =\{(2^j \bmod M) \colon 0\leq j < 4\}$ of $\psi$, so that $\Psi_{\mathrm{WH}}$ and $\Psi_{\mathrm{wvt}}$ have the same number of atoms. Note that for $\Psi_{\mathrm{wvt}}$, the Littlewood--Paley condition \cref{eq:LittlewoodPaleyWavelet} holds, as can be inferred from \cref{fig:LittlewoodPaleyWavelet}.  
The results are reported in \cref{tab:numerical-experiment} and verify our design insights \cref{it:design1,it:design2} for the frame and the nonlinearity. Notably, $(\Psi_{\mathrm{WH}},\lvert \cdot \rvert^2)$ is outperformed by all other frame--nonlinearity pairs. 
While $\tanh(\cdot)$ yields the highest separation capacity for both frames, $\lvert \cdot \rvert^2$ results in the lowest. 
The main conclusion from \cref{tab:numerical-experiment} is that nonlinearities that are not polynomials achieve significantly higher separation capacities for both frames.  

\begin{figure}[t!]
    \centering
    \resizebox{0.75\textwidth}{!}{\colorlet{colour1}{red!70!black}
\begin{tikzpicture}[scale=1, every node/.style={scale=0.5}]
    \begin{axis}[
        unit vector ratio*=1 1 1,
        grid=none,
        xmin=-0.1,
        xmax=24.75,
        ymin=-0.1,
        ymax=3.75,
        clip=false,
        axis lines=middle,
        xtick={0,1,2,...,24},
        ytick={0,1,2,3},
        xticklabels={,,},
        yticklabels={0,1,2,3},
        xlabel style = {at={(axis description cs:1,0)},anchor=west,yshift=2pt},
        ylabel style = {at={(axis description cs:0,1)},anchor=south,color=colour1},
        ylabel={$\lvert \widehat{\phi}_k \rvert^2+\sum_{\lambda \in \Lambda} \lvert \widehat{\psi}_{(\lambda^{-1}k) \bmod{M}} \rvert^2$},
        xlabel={$k$}
    ]
        \addplot[thin,ycomb,mark=*,mark options={scale=1, fill=white},mark size=1pt,color=colour1] coordinates {%
        (0,1) (1,2) (2,2) (3,2) (4,1) (5,2) (6,1) (7,2) (8,1) (9,2)(10,2) (11,3) (12,2) (13,2) (14,3) (15,2) (16,2)(17,1) (18,2) (19,1) (20,2) (21,1) (22,2) (23,2) (24,2)
      };
      \node[below,yshift=-2pt] at (1,0) {$1$};
      \node[below,yshift=-2pt] at (24,0) {$24$};
    \end{axis}
\end{tikzpicture}}
    \caption{Atoms of the wavelet frame $\Psi_{\mathrm{wvt}}$ satisfy \cref{eq:equivalent_LP_condition} and hence the Littlewood--Paley condition \cref{eq:LittlewoodPaleyWavelet}.}
    \label{fig:LittlewoodPaleyWavelet}
\end{figure}
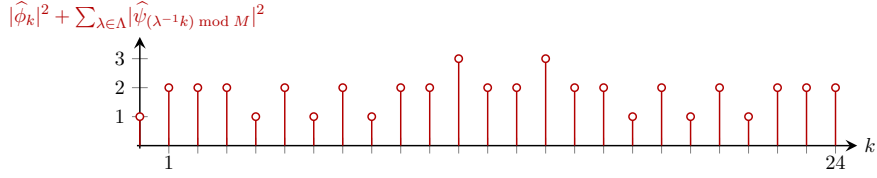

\begin{table}[t!]
    \centering
    \begin{tabular}{|>{\centering\arraybackslash}p{2cm}%
                ||>{\centering\arraybackslash}p{2cm}%
                |>{\centering\arraybackslash}p{2cm}%
                |>{\centering\arraybackslash}p{2cm}%
                |>{\centering\arraybackslash}p{2cm}|}
        \hline
         & $\lvert \cdot \rvert^2$ & $\lvert \cdot \rvert$ & $\tanh(\cdot)$ & $\sig{\cdot}$\\ \hline \hline
         $\Psi_{\mathrm{WH}}$ & $84$ & $300.6$ & $1722.6$ & $565$\\ \hline
         $\Psi_{\mathrm{wvt}}$ & $805$ & $860$ & $3187.6$ & $2869.2$\\ \hline
    \end{tabular}
    \caption{Separation capacities of scattering networks of depth $n_\mathrm{d}=3$ built from $\{(\Psi,\rho,\mathrm{Id})\}_{n \in \mathbb N}$ with $\Psi \in \{\Psi_{\mathrm{WH}},\Psi_\mathrm{wvt}\}$ and $\rho \in \{\lvert \cdot \rvert^2,\lvert \cdot \rvert,\tanh(\cdot),\sig{\cdot}\}$. Each configuration was evaluated over $10$ runs, and the results were averaged to mitigate minor numerical instabilities. Hence, the values are not necessarily even integers. 
    Note that the value obtained for $(\Psi_{\mathrm{WH}},\lvert \cdot \rvert^2)$ coincides with the result of \cref{thm:multi-layer_mod2}.
    }
    \label{tab:numerical-experiment}
\end{table}

\paragraph{Software availability} 
The code for numerically computing the separation capacity of scattering networks is available at the following \texttt{GitHub} repository:
\begin{center}
    \url{https://github.com/konstantin-haberle/ScatteringCapacity}
\end{center}

\appendix
\crefname{section}{Appendix}{Appendices}

\section{Notation}\label[appendix]{app:notation}
$\mathbb N$, $\mathbb N_0$, $\mathbb Z$, $\mathbb R$, and $\mathbb C$ denote the sets of natural numbers, nonnegative integers, integers, real numbers, and complex numbers, respectively.  
For $a,b\in \mathbb Z$ and $M \in \mathbb N$, we write $a \equiv b \Mod{M}$ whenever $M$ divides $(a-b)$. The unique number $r \in \{0,1,\dots,M-1\}$ such that $a \equiv r \Mod{M}$ will be denoted by $(a \bmod{M})$. We write $\mathbb Z/M\mathbb Z$ for the finite cyclic group, i.e., the quotient group of integers modulo $M$ consisting of equivalence classes, where $a,b \in \mathbb Z$ are equivalent if $a \equiv b \Mod{M}$. For $A \subseteq \mathbb Z /M\mathbb Z$, define its reflection to be $A^r \coloneqq \{-a\colon a \in A\}$. We use $\langle A \rangle$ for the subgroup generated by $A \subseteq \mathbb Z /M\mathbb Z$, i.e., the smallest subgroup of $\mathbb Z /M\mathbb Z$ containing $A$. The greatest common divisor of $a$ and $b$ is denoted by $\gcd(a,b)$.  
To represent the indicator of a statement $S$, we write $\mathbbm{1}_{\{S\}}$, which equals $1$ if $S$ is true and $0$ if $S$ is false. The cardinality of a set $A$ is denoted by $\lvert A \rvert$.
Let $x \in \mathbb R$. We write $\lfloor x \rfloor$ for the largest $k \in \mathbb Z$ such that $k \leq x$. Similarly, $\lceil x \rceil$ stands for the smallest $k \in \mathbb Z$ with $k \geq x$. We further set $x_+ \coloneqq \max\{0,x\}$.
For $\mathbb K \in \{\mathbb R, \mathbb C\}$, the standard Euclidean inner product of $x,y \in \mathbb K^M$, $M \in \mathbb N$, is denoted by $\langle x, y \rangle$, and its induced norm on $\mathbb K^M$ is $\lVert x \rVert \coloneqq \sqrt{\langle x, x \rangle}$. The orthogonal complement of a linear subspace $\mathcal{W}$ of $\mathbb K^M$ is given by $\mathcal{W}^\perp \coloneqq \{v \in \mathbb K^M \colon \innerprod{v}{w}=0, \forall w \in \mathcal{W}\}$. For a subset $S \subseteq \mathbb K^M$, $\mathrm{span}_{\mathbb K}(S)$ stands for the set of all finite linear combinations of vectors in $S$ with scalars in the field $\mathbb K$. Given a linear space $\mathcal{U}$ over $\mathbb K$, we write $\dim_{\mathbb K}(\mathcal{U})$ for its dimension. 
The $M$-dimensional Lebesgue measure on $\mathbb R^M$ is denoted by $\mathcal{L}^M$.   
The complex conjugate of $z \in \mathbb C$ is $\overline{z}$. We write $\Re(z)$ for the real and $\Im(z)$ for the imaginary part of $z$. For the matrix $A \in \mathbb C^{M\times N}$, $M,N \in \mathbb N$, $A^{\mathsf{T}}$ and $A^{\mathsf{H}}$ stand for its transpose and conjugate transpose, respectively. The $M$-dimensional identity matrix is denoted by $I_M$. The kernel of a linear map $L \colon \mathcal{U} \to \mathcal{V}$ between the linear spaces $\mathcal{U},\mathcal{V}$ is given by $\ker(L) \coloneqq \{u \in \mathcal{U} \colon L(u)=0\}$.
For a function $f \colon \mathbb Z/M\mathbb Z \to \mathbb C$, $M \in \mathbb N$, we will sometimes use the vector representation $f = \left(f_0, \dots, f_{M-1} \right)^{\mathsf{T}} \in \mathbb C^M$, where $f_k \coloneqq f(k)$, $k \in \{0,\dots,M-1\}$. The support of $f$, denoted $\mathrm{supp}(f)$, is defined to be the set of indices $k \in \{0,\dots,M-1\}$ for which $f_k \neq 0$. We denote by $\widehat f \coloneqq \left(\widehat f_0, \dots, \widehat f_{M-1}\right)^{\mathsf{T}} \in \mathbb C^M$ the discrete Fourier transform (DFT) of $f$ given by
\begin{align*}
    \widehat{f}_k \coloneqq \sum_{\ell=0}^{M-1} f_\ell \,e^{-2\pi i k\ell/M}, \quad k\in \{0,\dots,M-1\}.
\end{align*}
With the $(M \times M)$-DFT matrix
\begin{align*}
    F_M \coloneqq \begin{pmatrix}
        1 & 1 & 1 &\cdots & 1 \\
        1 & \omega_M & \omega_M^2 & \cdots & \omega_M^{M-1} \\
        1 & \omega_M^2 & \omega_M^4 & \cdots & \omega_M^{2(M-1)} \\
        \vdots & \vdots & \vdots & \ddots & \vdots \\
        1 & \omega_M^{M-1} & \omega_M^{2(M-1)} & \cdots & \omega_M^{(M-1)^2} 
    \end{pmatrix},
\end{align*}
where $\omega_M \coloneqq e^{-2\pi i/M}$, the DFT of $f$ can be written as 
$\widehat{f} = F_M f$.
For $k \in \{0,\dots,M-1\}$, we denote the translation operator by $(T_kf)_\ell \coloneqq f_{(\ell-k)\bmod{M}}$, $\ell\in \{0,\dots,M-1\}$. Involution is defined by $(f^*)_k \coloneqq \overline{f_{M-k}}$, for $k \in \{1,\dots,M-1\}$, and $(f^*)_0 = \overline{f_0}$.
Let $f,g\in \mathbb C^M$. The cyclic convolution of $f$ and $g$ is $(f \ast g) \in \mathbb C^M$ with
\begin{align*}
    (f\ast g)_k \coloneqq \sum_{\ell=0}^{M-1}f_\ell g_{(k-\ell) \bmod M}, \quad k \in \{0,\dots,M-1\}.
\end{align*}
Equivalently, we may write $f \ast g = C_g f$, where $C_g$ is the circulant matrix generated by $g$; that is,
\begin{align*}
    C_g \coloneqq \begin{pmatrix}
            g_0 & g_{M-1} & \cdots & g_1 \\
            g_1 & g_{0} & \cdots & g_2 \\
            \vdots & \vdots & \ddots & \vdots \\
            g_{M-1} & g_{M-2} & \cdots & g_0
        \end{pmatrix} \in \mathbb C^{M \times M}.
\end{align*}
For $k \in \mathbb N$, the $k$-fold convolution power of $f \in \mathbb C^M$ is 
\begin{align*}
    f^{\ast k} \coloneqq \underbrace{f \ast \cdots \ast f}_{\text{$k$ times}}.
\end{align*}
Let $\alpha = (\alpha_0, \dots, \alpha_{M-1})\in \mathbb N_0^M$ be a multi-index. The sum of its components will be denoted by $\lvert \alpha \rvert \coloneqq \alpha_0 +\cdots + \alpha_{M-1}$. For $n \coloneqq \lvert \alpha \rvert \in \mathbb N_0$, we set $\binom{n}{\alpha} \coloneqq n!/\alpha!$, where $\alpha ! \coloneqq \alpha_0! \cdots \alpha_{M-1}!$. For $z \in \mathbb C^M$, we define $z^\alpha \coloneqq z_0^{\alpha_0} \cdots z_{M-1}^{\alpha_{M-1}}$. We further use the standard multi-index notation for the partial derivative operator $\partial^{\lvert \alpha \rvert} /\partial z^\alpha \coloneqq \partial^{\lvert \alpha \rvert} /\left(\partial z_0^{\alpha_0}\cdots \partial z_{M-1}^{\alpha_{M-1}}\right)$.

\section{Symmetry Argument}\label[appendix]{appendix:sym}
In this section, we employ a symmetry argument to show that for $M',N\in \mathbb N$,
\begin{align}\label{eq:cover-symmetry-argument}
    \frac{C(N,M')}{2^N} = 2^{-N+1}\sum_{k=0}^{M'-1} \binom{N-1}{k}  \geq \frac{1}{2} \quad \text{if and only if} \quad N\leq 2M',
\end{align}
with equality if and only if $N=2M'$.
\begin{proof}
First note that by the binomial theorem,
\begin{align}
    1 &= 2^{-N+1}\sum_{k=0}^{N-1} \binom{N-1}{k} \nonumber\\
    &= 2^{-N+1}\sum_{k=0}^{M'-1} \binom{N-1}{k} + 2^{-N+1}\sum_{k=M'}^{N-1} \binom{N-1}{k} \nonumber\\
    &= 2^{-N+1}\sum_{k=0}^{M'-1} \binom{N-1}{k} + 2^{-N+1}\sum_{k=M'}^{N-1} \binom{N-1}{N-1-k} \label{eq:cover-symmetry-binomial-coeff}\\
    &= 2^{-N+1}\sum_{k=0}^{M'-1} \binom{N-1}{k} + 2^{-N+1}\sum_{k=0}^{N-1-M'} \binom{N-1}{k} \nonumber\\ 
    &= 2^{-N}\left(C(N,M')+C(N,N-M')\right), \label{eq:cover-fct-decomposition}
\end{align}
where \cref{eq:cover-symmetry-binomial-coeff} is by the symmetry of the binomial coefficient.
As each term in $2\sum_{k=0}^{M'-1}\binom{N-1}{k} = C(N,M')$ is nonnegative for all $M' \in \mathbb N$ and strictly positive whenever $1\leq M' \leq N$, $C(N,M')$ is nondecreasing in $M'$ and strictly increasing in $M'$ whenever $1\leq M'\leq N$. Therefore, 
\begin{align}
    C(N,N-M') &\leq C(N,M') \quad \text{if $N-M' \leq M'$}, \label{eq:Cnondecreasing1}\\
    C(N,N-M') &> C(N,M') \quad \text{if $N-M'> M'$}. \label{eq:Cnondecreasing2}
\end{align}
Now, if $N\leq 2M'$, \cref{eq:cover-fct-decomposition}, together with \cref{eq:Cnondecreasing1}, implies $2^{-N}C(N,M') \geq \frac 12$. Conversely, if $N>2M'$, we obtain from \cref{eq:cover-fct-decomposition,eq:Cnondecreasing2} that $2^{-N}C(N,M') < \frac 12$. This proves \cref{eq:cover-symmetry-argument}. To see that $2^{-N}C(N,M') = \frac 12$ if and only if $N=2M'$, we first note that the ``if'' part follows immediately from \cref{eq:cover-fct-decomposition}. On the other hand, if $2^{-N}C(N,M') = \frac 12$, then \cref{eq:cover-fct-decomposition} implies $C(N,M')=C(N,N-M')$, and consequently, $M'=N-M'$. This completes the proof of the assertion.        
\end{proof}

\section{Proofs of \texorpdfstring{\cref{subsec:WHmod2}}{Subsection~\ref{subsec:WHmod2}}}
\subsection{Proof of \texorpdfstring{\cref{L:dim_conv_mod2}}{Lemma \ref{L:dim_conv_mod2}}}\label[appendix]{app:dim_conv_mod2}
In the proof of this lemma, the following elementary fact is utilized.
\begin{lemma}[p. 84 in \cite{zhang2011matrix}]\label{L:pseudo_quadratic_form}
    Let $A \in \mathbb C^{M\times M}$. Then,
    \begin{align*}
        \left(\langle A z, z \rangle = 0, \quad \forall z \in \mathbb C^M\right) \Longleftrightarrow \left(A =0 \right).
    \end{align*}
\end{lemma}
\begin{remark}
Note, however, that $\left(\langle A x, x \rangle = 0, \quad \forall x \in \mathbb R^M\right) \Longleftrightarrow \left(A = -A^{\mathsf{T}} \right)$.
\end{remark}
\begin{proof}[Proof of \cref{L:dim_conv_mod2}]
    We first note that the DFT matrix $F_M$ diagonalizes the circulant matrix generated by $g_\lambda$. As it is easier to work with the resulting diagonal matrix, we analyze the space
\begin{align*}
    \mathcal{U} \coloneqq \spanC{\left\{ F_M\left(\lvert f \ast g_\lambda \lvert^{2}\right)\colon f \in \mathbb C^M\right\}}.
\end{align*}
This is equivalent to studying $\spanC{\{\lvert f \ast g_\lambda \lvert^{2}\colon f \in \mathbb C^M\}}$ because the DFT is linear and invertible. 
Let now $f \in \mathbb C^M$, and observe that
\begin{align*}
    F_M\left(\lvert f \ast g_{\lambda} \rvert^2\right) &= \frac 1M \left(\widehat{f \ast g_{\lambda}}\right) \ast \left(\widehat{\overline{f \ast g_{\lambda}}}\right).
\end{align*}
In particular, for $k \in \{0,\dots,M-1\}$, 
\begin{align*}
    \left(F_M\left(\lvert f \ast g_{\lambda} \rvert^2\right)\right)_k 
    &= \frac 1M \sum_{\ell \in V_{k}} \widehat{f}_{\ell} \overline{\widehat{f}_{(\ell-k) \bmod M}},
\end{align*}
with
\begin{align*}
    V_k \coloneqq \{\lambda R - m_0,\dots,\lambda R + m_0\} \cap \{\ell \bmod M \colon \lambda R - m_0 +k \leq \ell \leq \lambda R + m_0+k\},
\end{align*}
    where we used that $\widehat{g}_{\lambda,(\ell-k) \bmod M} = 1$ if $\lambda R -m_0 \leq (\ell-k) \bmod M \leq \lambda R +m_0$ and $\widehat{g}_{\lambda,(\ell-k) \bmod M} = 0$ otherwise.
It follows that
\begin{align*}
    V_k &= \begin{cases}
        \{\lambda R -m_0 +k, \dots, \lambda R+m_0\}, &\quad \text{if $0\leq k \leq R-1$,} \\
        \emptyset, &\quad \text{if $R \leq k \leq M-R$,} \\
        \{\lambda R -m_0 , \dots, \lambda R+m_0+k-M\}, &\quad \text{if $M-R+1\leq k \leq M-1$.}
    \end{cases}
\end{align*}
Therefore,
\begin{align}\label{eq:hat_fgl_mod2}
    \left(F_M\left(\lvert f \ast g_{\lambda} \rvert^2\right)\right)_k &= \begin{cases}
        \frac 1M \sum_{\ell =\lambda R - m_0+k}^{\lambda R +m_0} \widehat{f}_{\ell} \overline{\widehat{f}_{\ell-k}}, & \text{if $0\leq k \leq R-1$,} \\
        0, & \text{if $R \leq k \leq M - R$,} \\
        \frac 1M \sum_{\ell = \lambda R -m_0}^{\lambda R + m_0 +k-M} \widehat{f}_{\ell} \overline{\widehat{f}_{(\ell-k) \bmod M}}, & \text{if $M-R+1 \leq k \leq M -1$.}
    \end{cases} 
\end{align}
We refer to \cref{fig:WH-mod2-node} for an illustration of \cref{eq:hat_fgl_mod2}. Since $f \in \mathbb C^M$ was arbitrary, we obtain, by counting the number of nonzero components of $F_M\left(\lvert f \ast g_{\lambda} \rvert^2\right)$, that $\dimC{\mathcal{U}}\leq 2R-1$.

To show that equality holds, i.e., $\dimC{\mathcal{U}}= 2R-1$, we study the dimension of the orthogonal complement of $\mathcal{U}$, denoted by $\mathcal{U}^\perp$. Indeed, since $\mathcal{U} \oplus \mathcal{U}^\perp = \mathbb C^M$, $\dimC{\mathcal U} + \dimC{\mathcal{U}^\perp} = M$.
Note that
\begin{align*}
        \mathcal{U}^\perp &\coloneqq \left\{h \in \mathbb C^M \colon \langle h, u \rangle =0, \,\,\forall u \in \mathcal{U} \right\} \\
        &=\left\{h \in \mathbb C^M \colon \left\langle h, F_M\left(\lvert f \ast g_\lambda \lvert^{2}\right) \right\rangle =0, \,\,\forall f \in \mathbb C^M \right\}.
    \end{align*}
Let $h \in \mathbb C^M$, and compute
\begin{align*}
    \left\langle h, F_M\left(\lvert f \ast g_\lambda \lvert^{2}\right) \right\rangle &= \frac{1}{M} \left\langle h, \left(\widehat{f \ast g_\lambda}\right) \ast \left(\widehat{\overline{f \ast g_\lambda}}\right)\right\rangle \\
    &=\frac{1}{M} \left\langle h \ast \left(\widehat{f \ast g_\lambda}\right), \widehat{f \ast g_\lambda}\right\rangle \\
    &=\frac{1}{M} \left\langle C_h \left(\widehat{f \ast g_\lambda}\right), \widehat{f \ast g_\lambda}\right\rangle,
\end{align*}
where $C_h$ denotes the circulant matrix generated by $h$.
Setting $z \coloneqq (\widehat{f}_k \widehat{g}_{\lambda,k})_{\lambda R-m_0 \leq k \leq \lambda R +m_0} \in \mathbb C^R$, and letting $A \in \mathbb R^{M\times R}$ with entries taking values in $\{0,1\}$ be such that $Az = \widehat{f \ast g_\lambda}$, we obtain
\begin{align}
    \left\langle h, F_M\left(\lvert f \ast g_\lambda \lvert^{2}\right) \right\rangle &= \frac{1}{M} \left\langle C_h Az, Az\right\rangle
    = \frac{1}{M} \left\langle A^{\mathsf{H}}C_h Az, z\right\rangle. \label{eq:inner_prod_h_hat_fg_l}
\end{align}
Since each column of $A$ contains exactly one entry with value $1$, the matrix $A^{\mathsf{H}}C_h A$ forms a submatrix of $C_h$.  
From \cref{eq:hat_fgl_mod2,eq:inner_prod_h_hat_fg_l} it follows that all entries of $A^{\mathsf{H}}C_h A$ are elements of the set $\{h_{k \bmod{M}} \colon -R+1 \leq k \leq R-1 \}$.
Thanks to \cref{eq:inner_prod_h_hat_fg_l,L:pseudo_quadratic_form}, we have
\begin{align*}
    \left(\left\langle h, F_M\left(\lvert f \ast g_\lambda \lvert^{2}\right) \right\rangle =0, \,\,\forall f \in \mathbb C^M \right) \Longleftrightarrow \left(A^{\mathsf{H}}C_h A = 0\right).
\end{align*}
 Thus, if $h \in \mathcal{U}^\perp$, then $h_{k \bmod{M}} =0$ for $-R+1 \leq k \leq R-1$. In other words,
\begin{align*}
    \mathcal{U}^\perp \subseteq \{h \in \mathbb C^M \colon h_{k \bmod{M}} =0, \,\, -R+1 \leq k \leq R-1\},
\end{align*}
which implies $\dimC{\mathcal{U}^\perp} \leq M-(2R-1)$. 
But $\dimC{\mathcal{U}} \leq 2R-1$ so that $\dimC{\mathcal{U}} = 2R-1$, as desired.
\end{proof}
\subsection{Proof of \texorpdfstring{\cref{prop:singlelayernet}}{Proposition \ref{prop:singlelayernet}}}\label[appendix]{app:singlelayernet}
As the support sets of the functions $\widehat{\chi} \cup \{\widehat{g_\lambda}\}_{\lambda \in \Lambda}$ are disjoint, one may divide $\Phi(f)$ into the components $\{f \ast \chi\} \cup \{\lvert f \ast g_\lambda\rvert^2 \ast \chi\}_{\lambda \in \Lambda}$, for $f \in \mathbb C^M$, and analyze them separately. Specifically, using \cref{eq:SC-WHmod2}, we obtain 
\begin{align}\label{eq:WHmod2singlelayer1}
    \sepcap{\Phi} &= 4 \cdot \dimC{\spanC{\left\{f \ast \chi \colon f \in \mathbb C^M\right\}}}+ 2\sum_{\lambda \in \Lambda} \dimC{\spanC{\left\{\lvert f \ast g_\lambda\rvert^2 \ast \chi \colon f \in \mathbb C^M\right\}}}.
\end{align}    
Since $(\widehat{f \ast \chi})_k = \widehat{f}_k$ if $k \in \mathrm{supp}(\widehat{\chi})$ and $(\widehat{f \ast \chi})_k = 0$ otherwise, for $k \in \{0,\dots,M-1\}$, it immediately follows that
\begin{align}\label{eq:WHmod2singlelayer2}
    \dimC{\spanC{\left\{f \ast \chi \colon f \in \mathbb C^M\right\}}} = \lvert \mathrm{supp}(\widehat{\chi}) \rvert = R.
\end{align}
From \cref{eq:hat_fgl_mod2} we deduce that $\mathrm{supp}(F_M(\lvert f \ast g_\lambda \rvert^2)) \cap \mathrm{supp}(\widehat{\chi}) = \mathrm{supp}(\widehat{\chi})$. Application of \cref{L:dim_conv_mod2} results in
\begin{align}\label{eq:WHmod2singlelayer3}
    \dimC{\spanC{\left\{\lvert f \ast g_\lambda\rvert^2 \ast \chi \colon f \in \mathbb C^M\right\}}} = \lvert \mathrm{supp}(\widehat{\chi}) \rvert = R. 
\end{align}
Substituting \cref{eq:WHmod2singlelayer2,eq:WHmod2singlelayer3} into \cref{eq:WHmod2singlelayer1} results in  
\begin{align*}
    \sepcap{\Phi} = 4R+2LR = 4R+2(M/R-1)R = 2(M+R). 
\end{align*}
   
\subsection{Proof of \texorpdfstring{\cref{L:symmetry}}{Lemma \ref{L:symmetry}}}\label[appendix]{app:symmetry-lemma}
To show the claim, we make use of the fact that $\widehat{f}_k = \overline{\widehat{f}_{M-k}}$, for $k \in \{0,\dots,M-1\}$, whenever $f$ is real-valued. For $k \in \{0,\dots,M-1\}$, we compute
    \allowdisplaybreaks
    \begin{align*}
        (F_M(\lvert f \ast g_\lambda \rvert ^2))_k &= \frac 1M \sum_{\ell = 0}^{M-1} \left(F_M\left(f \ast g_{\lambda}\right)\right)_{\ell} \left(F_M\left(\overline{f \ast g_{\lambda}}\right)\right)_{(k-\ell) \bmod M} \\
        &= \frac 1M \sum_{\ell = 0}^{M-1} \widehat{f}_{\ell}\widehat{g}_{\lambda,\ell} \overline{\widehat{f}_{(\ell-k) \bmod M}} \widehat{g}_{\lambda, (\ell-k) \bmod M} \\
        &=\frac 1M \sum_{\ell = 0}^{M-1} \widehat{f}_{\ell}\widehat{g}_{\lambda',-\ell \bmod{M}} \overline{\widehat{f}_{(\ell-k) \bmod M}} \widehat{g}_{\lambda', (k-\ell) \bmod M} \\
        &=\frac 1M \sum_{\ell = 0}^{M-1} \widehat{f}_{-\ell \bmod{M}}\widehat{g}_{\lambda',\ell} \overline{\widehat{f}_{(-\ell-k) \bmod M}} \widehat{g}_{\lambda', (k+\ell) \bmod M} \\
        &=\frac 1M \sum_{\ell = 0}^{M-1} \overline{\widehat{f}_{\ell}}\widehat{g}_{\lambda',\ell} \widehat{f}_{(k+\ell) \bmod M} \widehat{g}_{\lambda', (k+\ell) \bmod M} \\
        &= \frac 1M \sum_{\ell = 0}^{M-1} \overline{(F_M(f \ast g_{\lambda'}))_\ell} \overline{(F_M(\overline{f\ast g_{\lambda'}}))_{(-k-\ell) \bmod{M}}} \\
        &= \overline{(F_M(\lvert f \ast g_{\lambda'} \rvert ^2))_{-k \bmod{M}}} \\
        &= (F_M(\lvert f \ast g_{\lambda'} \rvert ^2))_k,
    \end{align*}
    where we used that $\widehat{g}_{\lambda,\ell} = \widehat{g}_{\lambda',-\ell \bmod{M}}$ and the fact that $\widehat{g_\lambda}$ is real-valued. This completes the proof.
    
\subsection{Proof of \texorpdfstring{\cref{L:multilayer_net}}{Lemma \ref{L:multilayer_net}}} \label[appendix]{app:multilayer_net-lemma}
Recall from \cref{eq:hat_fgl_mod2} that $F_M(\lvert f \ast g_\lambda \rvert^2)$ is supported on $\{k \bmod{M} \colon -R+1 \leq k\leq R-1\}$ for $f \in \mathbb C^M$ and $\lambda \in \Lambda$. Moreover, $\widehat{g_\lambda}$ is supported on $\{k \colon \lambda R -m_0 \leq k \leq \lambda R +m_0\}$ for every $\lambda \in \Lambda$. Note that $\lambda R -m_0 \leq R-1$ only for $\lambda =1$ and $M-R+1 \leq \lambda R +m_0$ only for $\lambda =L$. Thus, $F_M(\lvert f \ast g_{\lambda_1} \rvert^2 \ast g_{\lambda_2}) = 0$ for $\lambda_1 \in \Lambda$ whenever $\lambda_2 \in \Lambda\setminus\{1,L\}$, see \cref{fig:WH-mod2-layer2_1,fig:WH-mod2-layer2_L}, establishing \cref{it:L-multilayer-i}. Furthermore, as $(\mathrm{supp}(\widehat{g_L}))^r = \mathrm{supp}(\widehat{g_1})$ and as $\lvert f \ast g_{\lambda_1} \rvert^2$ is real-valued, application of \cref{L:symmetry} yields $U[(\lambda_1,1)]f = U[(\lambda_1,L)]f$, which completes the argument for \cref{it:L-multilayer-ii}. Finally, since $\mathrm{supp}(U[\lambda_1,1]f) \subset \mathrm{supp}(\widehat{\chi})$, it holds that $U[q]f = 0$, for $q \in \Lambda^n$ with $n \geq 3$, thereby proving \cref{it:L-multilayer-iii}.

\subsection{Proof of \texorpdfstring{\cref{L:dim_path_1l}}{Lemma \ref{L:dim_path_1l}}}\label[appendix]{app:dim_path_1l-lemma}
Since the DFT is linear and invertible, we may equivalently study the space 
\begin{align*}
    \mathcal{V} \coloneqq\spanC{\left\{\begin{pmatrix}F_M\left((U[\lambda]f) \ast \chi \right)\\ F_M\left((U[\lambda,1]f)\ast \chi \right)\end{pmatrix} \colon f \in \mathbb C^M\right\}}.
\end{align*}
Let $f \in \mathbb C^M$, and set $u \coloneqq U[\lambda]f$. Then, $U[\lambda,1]f = \lvert u \ast g_1 \rvert^2$. Thanks to \cref{eq:hat_fgl_mod2}, we obtain, for $k \in \{0,\dots,M-1\}$,
\begin{align*}
    \left(F_M\left(U[\lambda,1]f\right)\right)_k 
    &= \begin{cases}
        \frac 1M \sum_{\ell =R - m_0+k}^{R +m_0} \widehat{u}_{\ell} \overline{\widehat{u}_{\ell-k}}, & \text{if $0\leq k \leq R-1$,} \\
        0, & \text{if $R \leq k \leq M - R$,} \\
        \frac 1M \sum_{\ell = R -m_0}^{R + m_0 +k-M} \widehat{u}_{\ell} \overline{\widehat{u}_{(\ell-k) \bmod M}}, & \text{if $M-R+1 \leq k \leq M -1$.}
        \end{cases}
\end{align*}
Applying \cref{eq:hat_fgl_mod2} again yields, for $k \in \{0,\dots,M-1\}$,
\begin{align}\label{eq:hat_u[lambda]f}
    \widehat{u}_k = \begin{cases}
        \frac 1M \sum_{\ell =\lambda R - m_0+k}^{\lambda R +m_0} \widehat{f}_{\ell} \overline{\widehat{f}_{\ell-k}}, & \text{if $0\leq k \leq R-1$,} \\
        0, & \text{if $R \leq k \leq M - R$,} \\
        \frac 1M \sum_{\ell = \lambda R -m_0}^{\lambda R + m_0 +k-M} \widehat{f}_{\ell} \overline{\widehat{f}_{(\ell-k) \bmod M}}, & \text{if $M-R+1 \leq k \leq M -1$.}
    \end{cases} 
\end{align}
Therefore, if $m_0=0$ (i.e., $R=1$), $U[\lambda,1]f =0$.
If $m_0>0$ (i.e., $R>1$),
\begin{align}\label{eq:hat_u[l1]f}
    \left(F_M\left(U[\lambda,1]f\right)\right)_k 
    &= \begin{cases}
        \frac 1M \sum_{\ell =R - m_0+k}^{R -1} \widehat{u}_{\ell} \overline{\widehat{u}_{\ell-k}}, & \text{if $0\leq k \leq m_0-1$,} \\
        0, & \text{if $m_0 \leq k \leq M - m_0$,} \\
        \frac 1M \sum_{\ell = R -m_0}^{R-1 +k-M} \widehat{u}_{\ell} \overline{\widehat{u}_{(\ell-k) \bmod M}}, & \text{if $M-m_0+1 \leq k \leq M -1$.}
        \end{cases}
\end{align}
In particular, for $k \in \{0,\dots,m_0-1\}$, we obtain
\begin{align}
    \left(F_M\left(U[\lambda,1]f\right)\right)_k 
    &= \frac 1M \sum_{\ell =R - m_0+k}^{R -1} \underbrace{\left(\frac 1M \sum_{\ell' =\lambda R - m_0+\ell}^{\lambda R +m_0} \widehat{f}_{\ell'} \overline{\widehat{f}_{\ell'-\ell}}\right)}_{=\widehat{u}_\ell} \underbrace{\left(\frac 1M \sum_{\ell'' =\lambda R - m_0+\ell-k}^{\lambda R +m_0} \overline{\widehat{f}_{\ell''}} \widehat{f}_{\ell''-\ell+k}\right)}_{=\overline{\widehat{u}_{\ell-k}}} \nonumber\\
    &= \frac{1}{M^3} \sum_{\ell =R - m_0+k}^{R -1} \sum_{\ell' =\lambda R - m_0+\ell}^{\lambda R +m_0} \sum_{\ell'' =\lambda R - m_0+\ell-k}^{\lambda R +m_0} \widehat{f}_{\ell'} \overline{\widehat{f}_{\ell'-\ell}} \overline{\widehat{f}_{\ell''}} \widehat{f}_{\ell''-\ell+k}. \label{eq:hat_Ul1f_1}
\end{align}
Since $U[\lambda,1]f$ is real-valued, 
\begin{align}\label{eq:hat_Ul1f_2}
    \left(F_M\left(U[\lambda,1]f\right)\right)_k = \overline{\left(F_M\left(U[\lambda,1]f\right)\right)_{M-k}}, \quad \text{for }M-m_0+1 \leq k\leq M-1.
\end{align} 
Denote the support of $F_M(U[\lambda,1]f)$ by
\begin{align*}
    S_{\lambda,1} \coloneqq \begin{cases}
        \{k \bmod{M} \colon -m_0+1 \leq k \leq m_0-1\}, & \text{if $m_0>0$}, \\
        \emptyset, & \text{if $m_0=0$}.
    \end{cases}
\end{align*} 
Note that $\lvert S_{\lambda,1} \rvert = (2m_0-1)_+ = (R-2)_+$. Moreover, from \cref{eq:hat_u[lambda]f} and the definition of $\chi$, it follows that $F_M((U[\lambda]f) \ast \chi)$ is supported on $S_{\lambda} \coloneqq \{k \bmod{M} \colon -m_0 \leq k \leq m_0\}$ with $\lvert S_{\lambda} \rvert = 2m_0+1 = R$. We thus have $\dimCinline{\mathcal{V}} \leq R+ (R-2)_+$. 

To show that equality holds, i.e., $\dimCinline{\mathcal{V}} = R+ (R-2)_+$, consider the orthogonal complement of $\mathcal{V}$, given by
\begin{align*}
    \mathcal{V}^\perp &\coloneqq \left\{h \in \mathbb C^{2M} \colon \langle h, v \rangle=0,\,\,\forall v \in \mathcal{V}\right\} \\
    &=\left\{h \in \mathbb C^{2M} \colon \left\langle h, \begin{pmatrix}F_M\left((U[\lambda]f) \ast \chi \right)\\ F_M\left((U[\lambda,1]f)\ast \chi \right)\end{pmatrix} \right\rangle=0,\,\,\forall f \in \mathbb C^M\right\}.
\end{align*}
To this end, set $z \coloneqq (\widehat{f}_k)_{\lambda R- m_0 \leq k \leq \lambda R+m_0} \in \mathbb C^R$. 
Observe from \cref{eq:hat_Ul1f_1,eq:hat_Ul1f_2} that, for each $k \in S_{\lambda,1}$, we can write $\left(F_M\left(U[\lambda,1]f\right)\right)_k = p_k(z,\overline{z})$, where $p_k\colon \mathbb C^R \times \mathbb C^R \to \mathbb C$ is a bihomogeneous polynomial of bidegree $(2,2)$; that is, $p_k$ is a polynomial such that $p(\mu z,\lambda\overline{z}) = \mu^2\lambda^2 p(z,\overline{z})$, for all $\mu,\lambda\in\mathbb C$, $z \in \mathbb C^R$. Specifically, we have
\begin{align*}
    \left(F_M\left(U[\lambda,1]f\right)\right)_k =p_k(z,\overline{z})\coloneqq \frac{1}{M^3} \sum_{\substack{\alpha, \beta \in \mathbb N_0^R \\\lvert \alpha \rvert =\lvert \beta \rvert =2}}a_{k,\alpha,\beta}z^\alpha \overline{z}^{\beta} = \frac{1}{M^3}\langle A_k \mathcal{Z}, \mathcal{Z}\rangle,
\end{align*}
where
\begin{align*}
    \mathcal{Z} \coloneqq \begin{pmatrix}
        z_0^2, z_0z_1, \dots, z_R^2
    \end{pmatrix} \in \mathbb C^{\binom{R+1}{2}}
\end{align*}
contains all monomials of degree $2$ in degree lexicographic order, 
and where $a_{k,\alpha,\beta} \in \{0,1\}$ are the entries of the matrix $A_k \in \mathbb R^{\binom{R+1}{2}\times \binom{R+1}{2}}$. Note that no two polynomials $p_k$ and $p_{k'}$ share the same monomial $z^\alpha \overline{z}^\beta$. Indeed, from \cref{eq:hat_Ul1f_1}, we can deduce that if the monomial $\widehat{f}_{i_1}\overline{\widehat{f}_{i_2}\widehat{f}_{i_3}}\widehat{f}_{i_4}$ appears in $(F_M(U[\lambda,1]f))_k$, then $(i_1+i_4)-(i_2+i_3)=k$, for $k \in \{0,\ldots,m_0-1\}$ and $i_1,i_2,i_3,i_4\in \{\lambda R-m_0,\ldots,\lambda R+m_0\}$. Furthermore, for $k\in\{M-m_0+1,\ldots,M-1\}$, $(F_M(U[\lambda,1]f))_k$ carries exactly the complex conjugate monomials of those in $(F_M(U[\lambda,1]f))_{M-k}$ by \cref{eq:hat_Ul1f_2}. Upon noting that monomials of the form $\lvert f_\ell \rvert^2$, $\ell \in \{\lambda R-m_0,\ldots,\lambda R+m_0\}$, can only appear for $k=0$, the claim that no two polynomials $p_k$ and $p_{k'}$ share the same monomial $z^\alpha \overline{z}^\beta$ follows.   

Let $h^{(1)},h^{(2)} \in \mathbb C^M$, and set $h \coloneqq \left(\left(h^{(1)}\right)^{\mathsf{T}},\left(h^{(2)}\right)^{\mathsf{T}}\right)^{\mathsf{T}} \in \mathbb C^{2M}$. For $f \in \mathbb C^M$, compute
\begin{align*}
    \left\langle h, \begin{pmatrix}F_M\left((U[\lambda]f) \ast \chi \right)\\ F_M\left((U[\lambda,1]f)\ast \chi \right)\end{pmatrix} \right\rangle &= \left\langle h^{(1)},F_M\left((U[\lambda]f) \ast \chi \right) \right\rangle + \left\langle h^{(2)},  F_M\left((U[\lambda,1]f) \ast \chi \right)\right\rangle \\
    &= \left\langle \widetilde{h}^{(1)},F_M(U[\lambda]f) \right\rangle + \left\langle h^{(2)},  F_M(U[\lambda,1]f)\right\rangle \\
    &= \frac{1}{M} \left\langle A^{\mathsf{H}}C_{\widetilde{h}^{(1)}} Az, z\right\rangle + \frac{1}{M^3} \left\langle \sum_{k \in S_{\lambda,1}} h_k^{(2)}A_k \mathcal{Z}, \mathcal{Z} \right\rangle,   
\end{align*}
where we used that $S_{\lambda,1} \subset \mathrm{supp}(\widehat{\chi})$, and where $\widetilde{h}^{(1)} \coloneqq \left(h^{(1)}_k \widehat{\chi}_k \right)_{0 \leq k \leq M-1}$. The last equality holds by \cref{eq:inner_prod_h_hat_fg_l}, where $A \in \mathbb R^{M\times R}$ is such that $Az = \widehat{f \ast g_\lambda}$. Rewriting the last equality gives
\begin{align*}
    \left\langle h, \begin{pmatrix}F_M\left((U[\lambda]f) \ast \chi \right)\\ F_M\left((U[\lambda,1]f)\ast \chi \right)\end{pmatrix} \right\rangle 
    &= \left\langle \begin{pmatrix}
        \frac 1M A^{\mathsf{H}}C_{\widetilde{h}^{(1)}} A &0 \\0 & \frac{1}{M^3} \sum_{k \in S_{\lambda,1}} h_k^{(2)}A_k 
    \end{pmatrix}\begin{pmatrix}
        z \\ \mathcal{Z}
    \end{pmatrix}, \begin{pmatrix}
        z \\ \mathcal{Z}
    \end{pmatrix}\right\rangle.
\end{align*}
Now suppose $h \in \mathcal{V}^\perp$, then 
\begin{align*}
    \left\langle \begin{pmatrix}
        \frac 1M A^{\mathsf{H}}C_{\widetilde{h}^{(1)}} A &0 \\0 & \frac{1}{M^3} \sum_{k \in S_{\lambda,1}} h_k^{(2)}A_k 
    \end{pmatrix}\begin{pmatrix}
        z \\ \mathcal{Z}
    \end{pmatrix}, \begin{pmatrix}
        z \\ \mathcal{Z}
    \end{pmatrix}\right\rangle = 0, \quad \text{for all $z \in \mathbb C^R$}.
\end{align*}
By taking the partial derivatives $\partial^{\ell+\ell'}/\partial z^\alpha \partial \overline{z}^\beta$ of the LHS, where $\ell,\ell' \in \{1,2\}$ and $\alpha,\beta \in \mathbb N_0^R$ satisfy $\lvert \alpha \rvert = \ell$, $\lvert \beta \rvert = \ell'$, and then evaluating at $z=0$, we obtain
\begin{align*}
    \begin{pmatrix}
        \frac 1M A^{\mathsf{H}}C_{\widetilde{h}^{(1)}} A &0 \\0 & \frac{1}{M^3} \sum_{k \in S_{\lambda,1}} h_k^{(2)}A_k 
    \end{pmatrix} =0.
\end{align*}
Hence, $h^{(1)}_k = 0$ whenever $k \in S_{\lambda}$ and $h^{(2)}_k = 0$ whenever $k \in S_{\lambda,1}$. Consequently,
\begin{align*}
    \mathcal{V}^\perp \subseteq \left\{h = \begin{pmatrix}
        h^{(1)} \\ h^{(2)}
    \end{pmatrix} \in \mathbb C^{2M} \colon \begin{aligned} &h^{(1)}_k = 0,\,\, \forall k \in S_{\lambda} \\ &h^{(2)}_k = 0,\,\, \forall k \in S_{\lambda,1}\end{aligned}\right\},
\end{align*}
and $\dimCinline{\mathcal{V}^\perp} \leq 2M-(R+(R-2)_+)$. Combining this result with the previously established inequality $\dimCinline{\mathcal{V}} \leq R+(R-2)_+$, we conclude that $\dimCinline{\mathcal{V}} = R+(R-2)_+$, as $\dimCinline{\mathcal{V}} = 2M-\dimCinline{\mathcal{V}^\perp}$. The proof is now finalized.
\section{Proof of \texorpdfstring{\cref{L:subsampling}}{Lemma \ref{L:subsampling}}}\label[appendix]{app:subsampling}
Observe that 
    \allowdisplaybreaks
    \begin{align}
    (\widehat{h_{\mathrm{d}}})_k &= \sum_{\ell=0}^{M-1} f_{(\ell S) \bmod M} \,e^{-2\pi ik\ell/M} \nonumber \\
    &= \sum_{r=0}^{\gcd(M,S)-1} \sum_{\ell=0}^{\widetilde{M}-1} f_{(\ell S) \bmod M} \,e^{-2\pi ik(\ell+r \widetilde{M})/M} \nonumber\\
    &= \sum_{r=0}^{\gcd(M,S)-1} e^{2\pi i kr /\gcd(M,S)} \sum_{\ell=0}^{\widetilde{M}-1} f_{(\ell S) \bmod M} \,e^{-2\pi ik\ell/M} \nonumber\\
    & = \gcd(M,S) \mathbbm{1}_{\{k \equiv 0 \Mod{\gcd(M,S)}\}} \sum_{\ell=0}^{\widetilde{M}-1} f_{(\ell S) \bmod M} \,e^{-2\pi ik\ell/M}, \label{eq:hat_hd}
\end{align}
where we used in the last equation that
\begin{align}\label{eq:indicator_sum}
    \sum_{r=0}^{\gcd(M,S)-1} e^{2\pi i kr /\gcd(M,S)} = \begin{cases}
        \gcd(M,S), &\quad \text{if $k\equiv 0 \Mod{\gcd(M,S)}$,} \\
        0, &\quad \text{otherwise.}
    \end{cases}
\end{align}
Next, define $h = (h_\ell)_{0\leq \ell \leq M-1}\in \mathbb C^M$ according to
\begin{align*}
    h_\ell \coloneqq \begin{cases}
    f_{\gcd(M,S)\ell}, &\quad \text{if $\ell \in \{0,\dots,\widetilde{M}-1\},$} \\
    0, &\quad \text{otherwise,}
    \end{cases} \quad \ell \in \{0, \dots,M-1\}.
\end{align*}
Then, 
\begin{align*}
    \sum_{\ell=0}^{\widetilde{M}-1} f_{(\ell S) \bmod M} \,e^{-2\pi ik\ell/M} &= \sum_{\ell=0}^{\widetilde{M}-1} h_{(\ell\widetilde S) \bmod {\widetilde{M}}} \,e^{-2\pi ik\ell/M}.
\end{align*}
Denote by $\widetilde{S}^{-1} \in  \mathbb Z/{\widetilde M}\mathbb Z$ the multiplicative inverse of $\widetilde{S}$ in $ \mathbb Z/{\widetilde M}\mathbb Z$, which exists because $\widetilde{M}$ and $\widetilde{S}$ are coprime. Since $\widetilde{S} (\mathbb Z/{\widetilde{M}}\mathbb Z) = \mathbb Z/{\widetilde{M}}\mathbb Z$ and $\widetilde{S}^{-1} (\mathbb Z/{\widetilde{M}}\mathbb Z) = \mathbb Z/{\widetilde{M}}\mathbb Z$, we have for $k \in \{0,\dots,M\}$ with $k \equiv 0 \Mod{\gcd(M,S)}$, 
\allowdisplaybreaks
\begin{align}
    \sum_{\ell=0}^{\widetilde{M}-1} f_{(\ell S) \bmod M}\,e^{-2\pi ik\ell/M} &= \sum_{\ell=0}^{\widetilde{M}-1} h_{(\ell\widetilde S) \bmod {\widetilde{M}}} \,e^{-2\pi ik\ell/M} \nonumber\\
    &= \sum_{\ell=0}^{\widetilde{M}-1} h_{\ell} \,e^{-2\pi ik((\ell\widetilde{S}^{-1}) \bmod{\widetilde M})/M} \nonumber\\
    &= \sum_{\ell=0}^{\widetilde{M}-1} h_{\ell} \,e^{-2\pi ik((\ell\widetilde{S}^{-1}) \bmod{\widetilde M})/(\gcd(M,S)\widetilde{M})} \nonumber\\
    &= \sum_{\ell=0}^{\widetilde{M}-1} h_{\ell} \,e^{-2\pi i(k/\gcd(M,S))((\ell\widetilde{S}^{-1}) \bmod{\widetilde M})/\widetilde{M}} \nonumber\\
    &= \sum_{\ell=0}^{\widetilde{M}-1} h_{\ell} \,e^{-2\pi i(k/\gcd(M,S))\ell\widetilde{S}^{-1}/\widetilde{M}} \nonumber\\
    &= \sum_{\ell=0}^{\widetilde{M}-1} h_{\ell} \,e^{-2\pi ik\ell\widetilde{S}^{-1}/M} \nonumber\\
    &= \widehat{h}_{(\widetilde{S}^{-1} k) \bmod{M}}. \label{eq:hat_h_inv_S}
\end{align}
It remains to express $\widehat{h}$ in terms of $\widehat{f}$. For $k' \in \{0,\dots,M-1\}$ with $k' \equiv 0 \Mod{\gcd(M,S)}$, we compute
\begin{align}
    \widehat{h}_{k'} &= \sum_{\ell=0}^{\widetilde{M}-1} f_{\gcd(M,S)\ell} \,e^{-2\pi i k'\ell/M} \nonumber\\
    &= \sum_{\ell=0}^{M-1} f_{\ell}\,e^{-2\pi i k'(\ell/\gcd(M,S))/M} \mathbbm{1}_{\{\ell \equiv 0 \Mod{\gcd(M,S)}\}} \nonumber \\
    &= \sum_{\ell=0}^{M-1} f_{\ell} \,e^{-2\pi i (k'/\gcd(M,S))\ell/M} \frac{1}{\gcd(M,S)}\sum_{r=0}^{\gcd(M,S)-1} \,e^{-2\pi i \ell r/\gcd(M,S)} \label{eq:hat_h_k_sum_indicator}\\
    &= \frac{1}{\gcd(M,S)} \sum_{r=0}^{\gcd(M,S)-1} \sum_{\ell=0}^{M-1} f_{\ell} \,e^{-2\pi i (k'/\gcd(M,S)+rM/\gcd(M,S))\ell/M} \nonumber\\
    &= \frac{1}{\gcd(M,S)} \sum_{r=0}^{\gcd(M,S)-1} \widehat{f}_{\left(k'/\gcd(M,S)+rM/\gcd(M,S)\right) \bmod{M}}, \label{eq:hat_h_k}
\end{align}
where we used \cref{eq:indicator_sum} in \cref{eq:hat_h_k_sum_indicator}. Upon noting that $\gcd(M,S)$ divides $((\widetilde{S}^{-1}k) \bmod{M})$ for $k \in \{0,\dots,M-1\}$ with $k \equiv 0 \Mod{\gcd(M,S)}$, combining \cref{eq:hat_hd}, \cref{eq:hat_h_inv_S}, and \cref{eq:hat_h_k} yields the desired expression.

\bibliography{ref}

@article{mallat2012group,
  title={Group invariant scattering},
  author={Mallat, St{\'e}phane},
  journal={Communications on Pure and Applied Mathematics},
  volume={65},
  number={10},
  pages={1331--1398},
  year={2012},
  publisher={Wiley Online Library}
}

@article{wiatowski2017deep,
  title={A mathematical theory of deep convolutional neural networks for feature extraction},
  author={Wiatowski, Thomas and B{\"o}lcskei, Helmut},
  journal={IEEE Transactions on Information Theory},
  volume={64},
  number={3},
  pages={1845--1866},
  year={2018},
  publisher={IEEE}
}

@article{wiatowski2017energy,
  title={Energy propagation in deep convolutional neural networks},
  author={Wiatowski, Thomas and Grohs, Philipp and B{\"o}lcskei, Helmut},
  journal={IEEE Transactions on Information Theory},
  volume={64},
  number={7},
  pages={4819--4842},
  year={2018},
  publisher={IEEE}
}

@article{cover1965geometrical,
  title={Geometrical and statistical properties of systems of linear inequalities with applications in pattern recognition},
  author={Cover, Thomas M},
  journal={IEEE Transactions on Electronic Computers},
  number={3},
  pages={326--334},
  year={1965},
  publisher={IEEE}
}

@book{bishop2006pattern,
  title={Pattern recognition and machine learning},
  author={Bishop, Christopher M and Nasrabadi, Nasser M},
  volume={4},
  number={4},
  year={2006},
  publisher={Springer}
}

@book{christensen2003introduction,
  title={An introduction to frames and Riesz bases},
  author={Christensen, Ole},
  volume={7},
  year={2003},
  publisher={Springer}
}

@inproceedings{kowalczyk1994separating,
  title={Separating capacity of analytic neurons},
  author={Kowalczyk, Adam},
  booktitle={Proceedings of 1994 IEEE International Conference on Neural Networks (ICNN'94)},
  volume={5},
  pages={3038--3043},
  year={1994},
  organization={IEEE}
}

@article{mitchison1989bounds,
  title={Bounds on the learning capacity of some multi-layer networks},
  author={Mitchison, GJ and Durbin, RM},
  journal={Biological Cybernetics},
  volume={60},
  number={5},
  pages={345--365},
  year={1989},
  publisher={Springer}
}

@article{bruna2013invariant,
  title={Invariant scattering convolution networks},
  author={Bruna, Joan and Mallat, St{\'e}phane},
  journal={IEEE Transactions on Pattern Analysis and Machine Intelligence},
  volume={35},
  number={8},
  pages={1872--1886},
  year={2013},
  publisher={IEEE}
}

@article{anden2014deep,
  title={Deep scattering spectrum},
  author={And{\'e}n, Joakim and Mallat, St{\'e}phane},
  journal={IEEE Transactions on Signal Processing},
  volume={62},
  number={16},
  pages={4114--4128},
  year={2014},
  publisher={IEEE}
}

@article{ponomarev1987submersions,
  title={Submersions and preimages of sets of measure zero},
  author={Ponomarev, Stanislav P},
  journal={Siberian Mathematical Journal},
  volume={28},
  number={1},
  pages={153--163},
  year={1987},
  publisher={Springer}
}

@book{narasimhan1971scv,
  title={Several complex variables},
  author={Narasimhan, Raghavan},
  year={1971},
  publisher={University of Chicago Press}
}

@book{hall2018theory,
  title={The theory of groups},
  author={Hall, Marshall},
  year={2018},
  publisher={Courier Dover Publications}
}

@book{daubechies1992ten,
  title={Ten lectures on wavelets},
  author={Daubechies, Ingrid},
  year={1992},
  publisher={SIAM}
}

@book{kaiser1994friendly,
  title={A friendly guide to wavelets},
  author={Kaiser, Gerald and Hudgins, Lonnie H},
  volume={300},
  year={1994},
  publisher={Springer}
}

@article{ali1993continuous,
  title={Continuous frames in {H}ilbert space},
  author={Ali, S Twareque and Antoine, Jean-Pierre and Gazeau, Jean-Pierre},
  journal={Annals of physics},
  volume={222},
  number={1},
  pages={1--37},
  year={1993},
  publisher={Elsevier}
}

@article{cortes1995support,
  title={Support-vector networks},
  author={Cortes, Corinna and Vapnik, Vladimir},
  journal={Machine learning},
  volume={20},
  pages={273--297},
  year={1995},
  publisher={Springer}
}

@inproceedings{huang2006large,
  title={Large-scale learning with {SVM} and convolutional for generic object categorization},
  author={Huang, Fu Jie and LeCun, Yann},
  booktitle={2006 IEEE Computer Society Conference on Computer Vision and Pattern Recognition (CVPR'06)},
  volume={1},
  pages={284--291},
  year={2006},
  organization={IEEE}
}

@book{grochenig2001foundations,
  title={Foundations of time-frequency analysis},
  author={Gr{\"o}chenig, Karlheinz},
  year={2001},
  publisher={Birkh{\"a}user}
}

@book{zhang2011matrix,
  title={Matrix theory: {B}asic results and techniques},
  author={Zhang, Fuzhen},
  year={2011},
  publisher={Springer}
}

@article{vapnik1971uniform,
author = {Vapnik, V. N. and Chervonenkis, A. Ya.},
title = {On the Uniform Convergence of Relative Frequencies of Events to Their Probabilities},
journal = {Theory of Probability \& Its Applications},
volume = {16},
number = {2},
pages = {264-280},
year = {1971},
}

@article{eickenberg2017solid,
  title={Solid harmonic wavelet scattering: Predicting quantum molecular energy from invariant descriptors of 3D electronic densities},
  author={Eickenberg, Michael and Exarchakis, Georgios and Hirn, Matthew and Mallat, St{\'e}phane},
  journal={Advances in Neural Information Processing Systems},
  volume={30},
  year={2017}
}

@article{morel2022scale,
  title={Scale dependencies and self-similar models with wavelet scattering spectra},
  author={Morel, Rudy and Rochette, Gaspar and Leonarduzzi, Roberto and Bouchaud, Jean-Philippe and Mallat, St{\'e}phane},
  journal={arXiv preprint arXiv:2204.10177},
  year={2022}
}

@article{cheng2020new,
  title={A new approach to observational cosmology using the scattering transform},
  author={Cheng, Sihao and Ting, Yuan-Sen and M{\'e}nard, Brice and Bruna, Joan},
  journal={Monthly Notices of the Royal Astronomical Society},
  volume={499},
  number={4},
  pages={5902--5914},
  year={2020},
  publisher={Oxford University Press}
}

@article{valogiannis2022towards,
  title={Towards an optimal estimation of cosmological parameters with the wavelet scattering transform},
  author={Valogiannis, Georgios and Dvorkin, Cora},
  journal={Physical Review D},
  volume={105},
  number={10},
  pages={103534},
  year={2022},
  publisher={APS}
}

@inproceedings{warrick2020arrhythmia,
  title={Arrhythmia classification of 12-lead electrocardiograms by hybrid scattering-LSTM networks},
  author={Warrick, Philip A and Lostanlen, Vincent and Eickenberg, Michael and And{\'e}n, Joakim and Homsi, Masun Nabhan},
  booktitle={2020 Computing in Cardiology},
  pages={1--4},
  year={2020},
  organization={IEEE}
}

@inproceedings{leonarduzzi2019maximum,
  title={Maximum-entropy scattering models for financial time series},
  author={Leonarduzzi, Roberto and Rochette, Gaspar and Bouchaud, Jean-Phillipe and Mallat, St{\'e}phane},
  booktitle={ICASSP 2019-2019 IEEE International Conference on Acoustics, Speech and Signal Processing (ICASSP)},
  pages={5496--5500},
  year={2019},
  organization={IEEE}
}

@article{cheng2024scattering,
  title={Scattering spectra models for physics},
  author={Cheng, Sihao and Morel, Rudy and Allys, Erwan and M{\'e}nard, Brice and Mallat, St{\'e}phane},
  journal={PNAS nexus},
  volume={3},
  number={4},
  pages={1--13},
  year={2024},
  publisher={Oxford University Press US}
}

@article{mityagin2020zerorealanalytic,
	author = {Mityagin, Boris},
	journal = {Mathematical Notes},
	number = {3},
	pages = {529--530},
	title = {The Zero Set of a Real Analytic Function},
	volume = {107},
	year = {2020},
}

@inproceedings{wiatowski2017topology,
  title={Topology reduction in deep convolutional feature extraction networks},
  author={Wiatowski, Thomas and Grohs, Philipp and B{\"o}lcskei, Helmut},
  booktitle={Wavelets and Sparsity XVII},
  volume={10394},
  pages={269--280},
  year={2017},
  organization={SPIE}
}

@book{rudin1991functional,
  title={Functional analysis},
  author={Rudin, Walter},
  year={1991},
  publisher={McGraw-Hill}
}

@book{krantz2002primer,
  title={A primer of real analytic functions},
  author={Krantz, Steven G and Parks, Harold R},
  year={2002},
  publisher={Springer Science \& Business Media}
}

@article{vashisht2017necessary,
  title={Necessary and sufficient conditions for discrete wavelet frames in {$\mathbb{C}^N$}},
  author = {Deepshikha and Lalit K. Vashisht},
  journal={Journal of Geometry and Physics},
  volume={117},
  pages={134--143},
  year={2017},
  publisher={Elsevier}
}

@book{shalev2014understanding,
  title={Understanding machine learning: From theory to algorithms},
  author={Shalev-Shwartz, Shai and Ben-David, Shai},
  year={2014},
  publisher={Cambridge university press}
}

@article{hsu2003practical,
  title={A practical guide to support vector classification},
  author={Hsu, Chih-Wei and Chang, Chih-Chung and Lin, Chih-Jen},
  journal={Taipei, Taiwan},
  year={2003},
  publisher={National Taiwan University}
}

@article{haberle2026function,
    author = {H\"aberle, Konstantin and B\"olcskei, Helmut},
    title = {Function-counting theory for low-dimensional data structures},
    journal = {in preparation},
    year = 2026
}

\newpage

\vfill

\end{document}